\documentclass[11pt]{article}

\usepackage[utf8]{inputenc} 
\usepackage[T1]{fontenc}    
\usepackage{hyperref}       
\usepackage{url}            
\usepackage{booktabs}       
\usepackage{amsfonts}       
\usepackage{nicefrac}       
\usepackage{microtype}      
\usepackage{tcolorbox}
\usepackage{bookmark}
\usepackage{fullpage}
\usepackage{enumerate}
\usepackage{enumitem}
\usepackage{color}
\usepackage{array}
\usepackage{wrapfig}
\usepackage{bbm}
\usepackage{booktabs}
\usepackage[sort]{natbib}
\hypersetup{
	colorlinks = true,
	citecolor = blue,
	linkcolor = black
}

\usepackage{graphicx}
\usepackage{caption}
\usepackage{subcaption}
\usepackage{amsmath}
\usepackage{amsthm}
\usepackage{amssymb}
\usepackage{tikz}
\usepackage{mathtools}
\usepackage{tablefootnote}
\usepackage{multirow}
\usepackage{xcolor}
\usetikzlibrary{arrows}

\allowdisplaybreaks[4]

\usepackage{mathrsfs}

\usepackage{algorithm}
\usepackage{algorithmic}
\usepackage{bm,todonotes}

\def\robc{\text{rob,c}}
\def\robp{\text{rob,p}}

\def\diam{\mathrm{diam}}

\allowdisplaybreaks

\newtheorem{thm}{Theorem}[section]
\newtheorem{lem}{Lemma}[section]

\newtheorem{prop}{Proposition}[section]
\newtheorem{asmp}{Assumption}[section]

\usepackage{xcolor}

\def\1{\bm{1}}

\def\rmP{{\mathbf{P}}}

\DeclareMathAlphabet{\mathsfit}{\encodingdefault}{\sfdefault}{m}{sl}
\SetMathAlphabet{\mathsfit}{bold}{\encodingdefault}{\sfdefault}{bx}{n}

\def\0{{\bf 0}}
\def\1{{\bf 1}}

\def\AM{{\mathcal A}}

\def\GM{{\mathcal G}}
\def\FM{{\mathcal F}}

\def\NM{{\mathcal N}}
\def\OM{{\mathcal O}}
\def\PM{{\mathcal P}}

\def\SM{{\mathcal S}}
\def\TM{{\mathcal T}}

\def\VM{{\mathcal V}}

\def\XM{{\mathcal X}}

\def\RB{{\mathbb R}}
\def\EB{{\mathbb E}}
\def\ZB{{\mathbb Z}}
\def\PB{{\mathbb P}}

\def\argmax{\mathop{\rm argmax}}
\def\argmin{\mathop{\rm argmin}}

\begin{document}

\title{
Robust Markov Decision Processes without Model Estimation
}
\author{
Wenhao Yang\thanks{Academy for Advanced Interdisciplinary Studies, Peking University; email: \texttt{yangwenhaosms@pku.edu.cn}. } \\
\and\hspace{-10pt}
Han Wang\thanks{Computing Science, University of Alberta. } \\
\and\hspace{-10pt}
Tadashi Kozuno\thanks{OMRON SINIC X, Japan. } \\
\and\hspace{-10pt}
Scott M. Jordan\footnotemark[2] \\
\and\hspace{-10pt}
Zhihua Zhang\thanks{School of Mathematical Sciences, Peking University. } \\
}
\maketitle

\begin{abstract}%
    Robust Markov Decision Processes (MDPs) are receiving much attention in learning a robust policy which is less sensitive to environment changes. There are an increasing number of works analyzing sample-efficiency of robust MDPs. However, there are two major barriers to applying robust MDPs in practice. First, most works study robust MDPs in a model-based regime, where the transition probability needs to be estimated and requires a large amount of memories $\OM(|\SM|^2|\AM|)$. Second, prior work typically assumes a strong oracle to obtain the optimal solution as an intermediate step to solve robust MDPs. However, in practice, such an oracle does not exist usually. To remove the oracle, we transform the original robust MDPs into an alternative form, which allows us to use stochastic gradient methods to solve the robust MDPs. Moreover, we prove the alternative form still plays a similar role as the original form. With this new formulation, we devise a sample-efficient algorithm to solve the robust MDPs in a model-free regime, which does not require an oracle and trades off a lower storage requirement $\OM(|\SM||\AM|)$ with being able to generate samples from a generative model or Markovian chain. Finally, we validate our theoretical findings via numerical experiments,  showing the efficiency with the alternative form of robust MDPs. 

\end{abstract}

\section{Introduction}
\label{sec: intro}
Current popular reinforcement learning (RL) algorithms rarely consider the distribution shift from simulation environments to real-world environments, which might make an RL agent suffer from a performance drop. From a theoretical perspective, a small perturbation of reward and transition probability can cause an optimal policy to become sub-optimal and a significant change in the value function \citep{mannor2004bias}.
To alleviate sensitivity in environment shift,  one combines MDPs \citep{sutton2018reinforcement} with a DRO problem \citep{duchi2016variance, duchi2018learning, namkoong2016stochastic, shapiro2017distributionally} to optimize the policy over the worst distribution within a region of the possible transition functions. And this region is called ``uncertainty set''. The mathematical model of this problem is  called robust MDPs \citep{wiesemann2013robust,iyengar2005robust,satia1973markovian}  (see Section~\ref{sec: preli} for more details).

How to design a computationally efficient and sample-efficient algorithm for solving robust MDPs is a challenge. There exists some learning algorithms with polynomially computational complexity\citep{goyal2018robust, ho2020partial, ho2018fast}, but it is still large in practice in terms of space memory, and they require the knowledge of underlying transition probabilities and rewards. In a data-driven scenario, other works \citep{si2020distributionally, zhou2021finite, yang2021towards,panaganti2022sample} give the sample complexity of robust bandits and MDPs without the knowledge of underlying transition functions and rewards but only offline data. But these works ignore the computation complexity of solving DRO problems, which is expensive, and assume the optimal solution of a DRO problem can be obtained exactly from an oracle. Moreover, these two lines of works rely on either the true value or empirical estimation of transition functions and rewards, which requires a large space to store the model in memory. Therefore, a core question remains open:
\begin{quote}
    \textit{Can we design a practical algorithm with a low storage requirement to solve robust MDPs with sample-efficiency guarantees? }
\end{quote}
In this paper we would address this issue by design an efficient algorithm with only $\OM(|\SM||\AM|)$ storage, which is model-free \citep{chen2016stochastic}. And we offer the following main contributions.


\paragraph*{Contributions.}
Rather than solving original robust MDPs, we propose a surrogate of robust MDPs, where we remove the constraint on transition functions and instead treat it as a penalty in the value function. And we call the original one the ``constrained'' problem and the surrogate the ``penalized'' problem. The two different problems connect with each other via Lagrangian duality \citep{boyd2004convex}. The motivation from the transformation is two-staged. First, in order to design a model-free algorithm, we need to leverage the dual form \citep{shapiro2017distributionally} of the DRO problem, which could allow us to apply stochastic gradient methods. Second, solving a constrained DRO problem from its dual form will suffer from unbounded gradients \citep{namkoong2016stochastic}, which makes stochastic gradient method fail to converge. Thus, we introduce the penalized version which provide bounded gradient and finite-sample convergence guarantees.

In Section~\ref{sec: prmdp}, to validate whether the penalized robust MDPs is well-defined, we establish the same fundamental propositions used to develop constrained robust MDPs \citep{iyengar2005robust}. To be concrete, we show the Bellman equation still exists in the penalized setting and establish statistical results with a generative model \citep{azar2013minimax}. Comparing to constrained robust MDPs in \citet{yang2021towards}, we find the statistical results are similar, which guarantee the reasonability of the penalized version.


With the penalized form, the dual form of the DRO problem can be regarded as a risk minimization problem. Thus, it is natural to solve it by a stochastic gradient method, from which we do not require an oracle to the DRO problem solutions anymore. Leveraging on this, in Section~\ref{sec: generative}, we design a ``Q-learning'' type algorithm and prove the sample complexity of our algorithm is polynomially dependent on the robust MDPs' parameters (see the detail in Section~\ref{sec: generative}), including state-action space size, discount factor, size of uncertainty set, etc.

The previous approach required independent samples for each state-action pair, but in practice such a generating mechasim might not exist. This creates the algorithm in the generative model setting would be restricted in some scenarios. Instead, in Section~\ref{sec: markovian}, we consider a more realistic and difficult data generating mechanism, named Markovian data, where we can only observe one trajectory following a given behavior policy. Different from the generative model, only one sample could be generated for current visiting state-action pair in this setting. Again, under some regular assumptions, we design a ``Q-learning'' type algorithm and prove its sample complexity in this setting. However, the result relies heavily on some parameters than the generative model setting but is still polynomially dependent on the robust MDPs' parameters. 

Finally, in Section~\ref{sec: experiment}, we conduct numerical experiments to demonstrate the utility of the penalized robust MDP formulation as a practical and efficiently solvable alternative to the constrained robust MDP formulation. 


\paragraph*{Related Work.}
Robust MDPs were proposed by \citet{iyengar2005robust,nilim2005robust,satia1973markovian} to alleviate the sensitivity of optimal policies and value functions w.r.t. estimation errors of transition functions and rewards. Given the access to the true transition functions and rewards, many works  have developed computationally efficient algorithms to solve the robust MDPs~\citep{wiesemann2013robust,xu2006robustness, lim2013reinforcement, goyal2018robust,ho2018fast,ho2020partial}. If the true environment is unknown but samples can be  generated from the environment, there are various works proving sample complexity bounds that tell us how many samples are sufficient to guarantee an accurate solution. 
In terms of model-based methods, \citet{zhou2021finite, panaganti2022sample, yang2021towards, shi2022distributionally} constructed empirical estimation of the transition functions and rewards from the samples. And they applied a variant of value iteration \citep{sutton2018reinforcement} with the estimated model to solve robust MDPs. Although they gave the sample complexity of their algorithms, they did not consider the computation complexity of solving robust MDPs. For model-free methods, \citet{liu2022distributionally} proposed a robust Q-learning algorithm to learn the robust Q-value function by multilevel Monte-Carlo method. Subsequently, \cite{pmlr-v206-wang23b} showed the sample-complexity of this algorithm. And both of them require an oracle to solve the DRO problem.


Despite the accomplishments of previous works, it is still unknown how to design an algorithm requiring less memory space (model-free) and theoretically efficient. In the primal form of constraint robust MDPs, we need to solve the DRO problem with $|\SM|^2|\AM|$ variables which requires significant computational and memory resources  \citep{namkoong2016stochastic, duchi2018learning}, from which a model-free algorithm is unlikely to be designed. Instead, if we solve the DRO problem from its Lagrangian dual form \citep{shapiro2017distributionally}, it is possible to design a model-free algorithm . We provide the details in Sections~\ref{sec: preli} and~\ref{sec: prmdp}. Because of the unbounded issue in the constraint problem,  \citet{sinha2017certifying} changes the constraint problem to the penalty term in objective function. Using the penalty form, \citet{qi2021online, Jin2021non} provide a theoretically efficient gradient method for DRO problem. Inspired by this transformation, we apply it to robust MDPs and design a sample-efficient and model-free algorithm.

Moreover, to deal with the Markovian data setting, the algorithm we propose in Section~\ref{sec: markovian} can also be regarded as a two-time-scale stochastic optimization problem. For linear case, several works \citep{konda2004convergence,kaledin2020finite,gupta2019finite,doan2020finite} has studied the finite-sample results. For non-linear case, there are also some works \citep{zeng2021two,doan2021nonlinear, doan2021finite,mokkadem2006convergence} study the finite-sample results. However, due to the non-smoothness of Q-learning, we can not apply the results of prior works directly. Moreover, to control the noise induced by Markovian data, we adapt a Poisson equation method \cite{benveniste2012adaptive,li2023online,metpri87} in this paper. 

The remainder of this paper is organized as follows. In Section~\ref{sec: preli} we review distributionally robust optimization and robust Markov decision processes. In Section~\ref{sec: prmdp} we present rnative formulation for robust Markov decision processes. nt aWe then present our main results with a generative model and a Markovian data mechanism in Sections~\ref{sec: generative}  and~\ref{sec: markovian}, respectively. We conduct experimental analysis in Section~\ref{sec: experiment}, and conclude our work in Section~\ref{sec: conc}. We leave the proof details to the appendix. 


\section{Preliminaries}
\label{sec: preli}
For any finite set $\XM$, we denote the set of  probability distributions on $\XM$ as $\Delta(\XM)$. For any two probability distributions $P,Q$ with a finite support $\XM$, $Q\ll P$ signifies $Q$ is absolutely continuous w.r.t.~$P$, which means for any $x\in\XM$, $P(x)=0$ implies $Q(x)=0$. For a convex function $f$ satisfying $f(1)=0$, we define the $f$-divergence by $D_f(Q\|P):=\sum_{x\in\XM}f\left(\frac{Q(x)}{P(x)}\right)P(x)$ for $Q\ll P$ and $D_f(Q\|P):=+\infty$ for $Q$ is not absolutely continuous w.r.t.~$P$. 
For a function $f:\Omega\rightarrow\RB\cup\{-\infty,+\infty\}$, its convex conjugate is defined by $f^*(t):=\sup_{s\in\Omega}\{st-f(s)\}$. For a random variable $X$, we denote the sigma-algebra generated by $X$ as $\sigma(X)$. For a sequence of random variables $\{X_t\}_{t=1}^T$, we denote the sigma-algebra generated by $\{X_t\}_{t=1}^T$ as $\sigma(\{X_t\}_{t=1}^T):=\sigma\left(\bigcup_{t=1}^T\sigma(X_t)\right)$.

\paragraph*{Distributionally Robust Optimization} Let $P^*(\cdot)$ be a probability distribution on a set $\XM$ and $V$ be a real-valued function on $\XM$. The constrained DRO problem \citep{shapiro2017distributionally} is formulated as:
\begin{align}
    \label{eq: drop}
    \mathcal{R}_{c}(P^*,V):=\inf_{D_{f}(P\|P^*)\le\rho,\atop P\in\Delta(\XM)} \sum_{x\in \XM} P(x)V(x),
\end{align}
and its dual form is:
\begin{align}
    \label{eq: drod}
    \mathcal{R}_{c}(P^*,V)=\sup_{\lambda\ge0,\eta\in\RB}&\Bigg[-\lambda\sum_{x\in \XM} P^*(x)f^\dagger_{\eta, \lambda, V}(x)-\lambda\rho+\eta\Bigg],
\end{align}
where $f^\dagger_{\eta, \lambda, V}(x):=f^*\left(\frac{\eta-V(x)}{\lambda}\right)$, and $\lambda$ ($\geq 0$) and $\eta$ are the dual variables w.r.t.\ constraints $D_f(P\|P^*)\le\rho$ and  $\sum_{x\in\XM}P(x)=1$, respectively. Usually, we make some assumptions on the function $f$.
\begin{asmp}
    \label{asmp: f}
    $f(t)$ is a convex function on $\RB$. It satisfies $f(1)=0$ and $f(t):=+\infty$ when $t<0$, and differentiable on $\RB_+$.
\end{asmp}

Due to the unbounded gradient issue in \ref{eq: drod} \citep{namkoong2016stochastic}, some works replace the constraint $D_f(P\|P^*)\le\rho$ with penalty \citep{Jin2021non, qi2021online, sinha2017certifying}:
\begin{align}
    \label{eq: pdrop}
    \mathcal{R}_{p}(P^*,V):=\inf_{P\in\Delta(\XM)} \sum_{x\in \XM} P(x)V(x)+\lambda D_{f}(P\|P^*).
\end{align}
Similar to \eqref{eq: drod}, the dual problem of~\eqref{eq: pdrop} is:
\begin{align}
    \label{eq: pdrod}
    \mathcal{R}_{p}(P^*,V)=\sup_{\eta\in\RB}&\left[-\lambda\sum_{x\in \XM} P^*(x)f^\dagger_{\eta, \lambda, V}(x)+\eta\right],
\end{align}
where $\eta$ is the dual variable w.r.t.~constraint $\sum_{x\in\XM}P(x)=1$. The robustness parameter for the constrained DRO problem is $\rho$, while it is $\lambda$ for the penalized DRO problem.

\paragraph*{Robust Markov Decision Processes}

An MDP is defined by the tuple $\langle\SM,\AM,P^*,R,\gamma\rangle$, where $\SM$ is a finite state space, $\AM$ is a finite action space, $P^*\colon \SM\times\AM\rightarrow\Delta(\SM)$ is the transition function, $R\colon  \SM\times\AM \rightarrow [0,1]$ is the reward function, and $\gamma\in[0,1)$ is the discount factor. A stationary policy is a function $\pi\colon \SM\rightarrow\Delta(\AM)$. A trajectory induced by a policy $\pi$ and $P$ is $(s_0, a_0, s_1, a_1, \cdots)$, where $s_{t+1}\sim P(\cdot|s_t,a_t)$, $a_t\sim\pi(\cdot|s_t)$ and $s_0$ is given or generated from an initial distribution. A robust MDP considers a set $\PM$ of transition functions within a small region around $P^*$. In the literature \citep{iyengar2005robust,wiesemann2013robust}, a $(s,a)$-rectangular uncertainty set w.r.t.\ a $f$-divergence is considered. Formally, the uncertainty set is defined by $\PM:=\otimes_{s,a\in\SM\times\AM}\PM_{s,a}(\rho)$, where 
\begin{align*}
    \PM_{s,a}(\rho):=\Bigg\{&P(\cdot|s,a)\in\Delta(\SM)\Bigg{|}D_f(P(\cdot|s,a)\|P^*(\cdot|s,a))\le\rho\Bigg\}.
\end{align*}
The value function under a policy $\pi$ on an MDP is defined by $V^\pi_{P^*}(s):=\EB_{\pi,P^*}\left[\sum_{t=0}^{\infty}\gamma^t R(s_t,a_t)|s_0=s\right]$. In a robust MDP, there is a robust value function, which considers the worst case evaluation of value for all transition functions $P \in \PM$, i.e., $V_{\robc}^\pi(s):=\inf_{P\in\PM}V_{P}^{\pi}(s)$, where ``$\text{c}$'' stands for word ``constraint.''
In this setting, it is shown the optimal robust value function $V_{\robc}^*:=\max_{\pi}V_{\robc}^\pi$ satisfies a Bellman equation $V_{\robc}^*=\TM_{\robc}V_{\robc}^*$ \citep{iyengar2005robust, zhou2021finite}, where the robust Bellman operator $\TM_{\robc}$ is defined by:
\begin{align}
    \label{eq: robcP}
    \TM_{\robc}V (s)&:=\max_{a\in\AM}\Bigg(R(s,a)+\gamma\inf_{P(\cdot|s,a)\in\PM_{s,a}(\rho)}\sum_{s'\in\SM}P(s'|s,a)V(s')\Bigg)
\end{align}
for any $V\in\VM:=\left[0,1/(1-\gamma)\right]^{|\SM|}$. Indeed, the inner problem $\inf_{P(\cdot|s,a)\in\PM_{s,a}(\rho)}\sum_{s'\in\SM}P(s'|s,a)V(s')$ is a DRO problem. We leverage the dual form of the DRO problem in Eqn.~\eqref{eq: drod} and can rewrite the robust Bellman operator $\TM_{\robc}$ by:
\begin{align}
    \label{eq: cons}
    \TM_{\robc}V (s)&:=\max_{a\in\AM}\Bigg{(}R(s,a)+\gamma\sup_{\lambda\ge0,\eta\in\RB}\Bigg[-\lambda\rho+\eta-\lambda\sum_{s'\in \SM} P^*(s'|s,a)f^\dagger_{\eta, \lambda, V}(s')\Bigg]\Bigg{)},
\end{align}
where $\lambda$ is the dual variable w.r.t.\ constraint $P(\cdot|s,a)\in\PM_{s,a}(\rho)$, and $\eta$ is the dual variable w.r.t~constraint $\sum_{s'\in\SM}P(s'|s,a)=1$.


When transition function $P^*$ is unknown, we can estimate it via offline dataset and substitute the empirical estimator $\widehat{P}$ for $P^*$. Then the empirical uncertainty set $\widehat{\PM}_{s,a}(\rho)$ is defined by:
\begin{align*}
    \widehat{\PM}_{s,a}(\rho):=\Bigg\{&P(\cdot|s,a)\in\Delta(\SM)\Bigg{|}D_f(P(\cdot|s,a)\|\widehat{P}(\cdot|s,a))\le\rho\Bigg\},
\end{align*}
and the corresponding empirical robust Bellman operator is defined by:
\begin{align*}
    \widehat{\TM}_{\robc}V (s)&:=\max_{a\in\AM}\Bigg(R(s,a)+\gamma\inf_{P\in\widehat{\PM}_{s,a}(\rho)}\sum_{s'\in \SM}P(s'|s,a)V(s')\Bigg).
\end{align*}

By the dual form~\eqref{eq: cons}, for each $(s,a)$ pair, we can sample $s'\sim P^*(\cdot|s,a)$ to get a stochastic unbiased gradient update the dual variable $\eta$. Once ~\eqref{eq: cons} is solved approximately, then we can obtain the near-optimal robust value function by the Q-learning algorithm. In this way, we can avoid estimating the transition functions and obtain a model-free method.






\section{Alternative Form of Robust MDPs}
\label{sec: prmdp}
However, solving \eqref{eq: cons} by stochastic gradient descent will suffer from unbounded gradient issue. Thus, it is impossible to derive theoretical guarantee for the convergence of stochastic gradient method from the dual form \eqref{eq: cons} \citep{bubeck2015convex}. To overcome this limitation, we propose a novel penalty version of robust value function with robustness parameter $\rho$  replaced by $\lambda$:
\begin{align}
    \label{eq: robpV}
    V^{\pi}_{\robp}(s):=&\inf_{P\in\Delta(\SM)^{|\SM||\AM|}}\EB_{P,\pi}\Big{[}\sum_{t=0}^{\infty}\gamma^t (R(s_t,a_t)+\lambda\gamma D_f(P(\cdot|s_t,a_t)\|P^*(\cdot|s_t,a_t)))\Big{|}s_0=s\Big{]}.
\end{align}
Similarly, we can also define a robust Bellman operator: 
\begin{align}
    \label{eq: pen}
    &\TM_{\robp}V (s):=\max_{a\in\AM}\Bigg{(}R(s,a)+\gamma\inf_{P(\cdot|s,a)\in\Delta(\SM)}
    \Bigg[\sum_{s'\in \SM}P(s'|s,a)V(s')+\lambda D_f(P(\cdot|s,a)\|P^*(\cdot|s,a))\Bigg]\Bigg{)}.
\end{align}
Similar to $\TM_{\robc}$, the dual form of $\TM_{\robp}$ is:
\begin{align*}
    &\TM_{\robp}V (s):=\max_{a\in\AM}\Bigg{(}R(s,a)+\gamma\sup_{\eta\in\RB}\left[-\lambda\sum_{s'\in \SM} P^*(s'|s,a)f^\dagger_{\eta, \lambda, V}(x)+\eta\right]\Bigg{)}.
\end{align*}
For Q-value function, the robust Bellman operator is defined by:
    \begin{align}
        &\TM_{\robp}Q (s,a):=R(s,a)+\gamma\sup_{\eta\in\RB}\left[-\lambda\EB_{s'\sim P^*_{s,a}} f^*\left(\frac{\eta-\max_{a'}Q(s',a')}{\lambda}\right)+\eta\right].
        \label{eq: robpQ}
    \end{align}

In a high-level idea, \eqref{eq: pen} and \eqref{eq: cons} are connected with each other via Lagrange duality. The next proposition shows that the optimal robust value function $\max_{\pi} V_{\robp}^\pi$ is exactly the fixed point of $\TM_{\robp}$, which illustrates the reasonability of the penalized form. We defer the proof to Appendix~\ref{apd: preli}.
\begin{prop}
    \label{prop: pen_def}
    $\TM_{\robp}$ is a $\gamma$-contraction operator on $\VM$. Thus, a fixed point $V^*_{\robp}$ exists, and $V^*_{\robp}=\max_{\pi}V^\pi_{\robp}$.
\end{prop}

Proposition~\ref{prop: pen_def} shows the penalized robust MDPs share the similar basic properties as constraint MDPs do. Subsequently, we provide a stronger connection between these two forms. In Theorem~\ref{thm: rholambda}, we show for each given constraint robust MDP, there exists a penalized robust MDP, whose value functions are exactly the same.
\begin{thm}
    \label{thm: rholambda}
    For a given robust MDP with parameters $\langle\SM,\AM,R,P^*,\gamma\rangle$ and $f(\cdot)$-divergence, for a given constraint parameter $\rho>0$ there exists a penalty parameter $\lambda>0$, such that $V_{\robc}^*(\mu)=V_{\robp}^*(\mu)$, where $\mu\in\Delta(\SM)$ is a given initial distribution. 
    Similarly, for a given penalty parameter $\lambda>0$, there exists a constraint parameter $\rho>0$, such that $V_{\robp}^*(\mu)=V_{\robc}^*(\mu)$.
\end{thm}

Besides, in a data-driven scenario, we provide a result showing that robustness parameter $\lambda$ plays a similar role in penalized robust MDPs with robustness parameter $1/\rho$ in constrained robust MDPs in a finite-sample regime in the following theorem.
\begin{thm}[Statistical Equivalence]
    \label{thm: prmdp_stat}
    Suppose we access a generative model and estimate $\widehat{P}(s'|s,a)=\frac{1}{n}\sum_{i=1}^n\1(X_{i}^{(s,a)}=s')$, where $X_{i}^{(s,a)}\sim P^*(\cdot|s,a)$ are independent random variables.  Choosing $f(s)=(s-1)^2$ where $s\ge0$, with probability $1-\delta$, we have:
    \begin{align*}
        \left\|\widehat{V}_{\robp}^*-V_{\robp}^*\right\|_{\infty}\le\widetilde{\OM}\left(\frac{\max\left\{\frac{1}{\lambda(1-\gamma)^2},\lambda\right\}}{(1-\gamma)\sqrt{n}}\right).
    \end{align*}
    Furthermore, there exists a class of penalized robust MDPs with $f(s)=(s-1)^2$, such that for every $(\varepsilon,\delta)$-correct robust RL algorithm, when $\lambda=\OM(1-\gamma)$, the total number of samples needed is at least:
    \[
        \widetilde{\Omega}\left(\frac{|\SM||\AM|\lambda^2}{\varepsilon^2(1-\gamma)^3}\right).
    \]
    Additionally, when $\lambda=\Omega(1-\gamma)$, the total number of samples needed is at least:
    \[
        \widetilde{\Omega}\left(\frac{|\SM||\AM|}{\varepsilon^2(1-\gamma)^3}\min\left\{\frac{1}{16},\frac{\lambda\gamma(1-\gamma)}{2\gamma-1}\right\}\right).
    \]
\end{thm}

In \citet{yang2021towards}, the upper bound of constrained robust MDPs is $\widetilde{\OM}\left(\frac{|\SM||\AM|(1+\rho)^3}{\rho^2\varepsilon^2(1-\gamma)^4}\right)$\footnote{Here we reduce the $|\SM|^2$ to $|\SM|$ because we consider the deviation of value functions instead of $\varepsilon$-optimal policy, where we do not need a uniform bound over policy class and value function class as \cite{yang2021towards} did.} with $f(s)=(s-1)^2$, and the lower bound is $\widetilde{\Omega}\left(\frac{|\SM||\AM|}{\varepsilon^2(1-\gamma)^2}\min\left\{\frac{1}{1-\gamma},\frac{1}{\rho}\right\}\right)$. According to results of Theorem~\ref{thm: prmdp_stat}, the coefficient $\lambda$ plays a similar role as $1/\rho$ does in constrained robust MDPs. When $\lambda$ is small, we expect a robust solution, which leads to small sample complexity but conservative policy. When $\lambda$ is large, we expect a non-robust solution, which means the sample complexity should be approximately equal with sample complexity of non-robust MDPs \citep{azar2013minimax}.

With all the background presented, we are ready to design a model-free algorithm by combing stochastic gradient method and Q-learning algorithm with sample efficiency guarantees. Prior to introducing our results, we simplify the notation and denote:
\[
    J^{(s,a)}(\eta, V):=-\lambda\sum_{s'\in \SM}P^*(s'|s,a)f^\dagger_{\eta, \lambda, V}(s')+\eta.
\]
Additionally, the data is obtained in an online approach with a generative model, which means at each time step $t$, we have an observation $s'_t(s,a)\sim P^*(\cdot|s,a)$ and $r_t(s,a)$ for each $(s,a)$ pair satisfying $\EB[ r_t(s_t,a_t)|s_t,a_t]=R(s_t,a_t)$. We denote
\[
    J^{(s,a)}_t(\eta,V;s_t'(s,a)):=-\lambda f^\dagger_{\eta, \lambda, V}(s_t'(s,a))+\eta,
\]
where $\EB[J^{(s,a)}_t(\eta,V;s_t'(s,a))]=J^{(s,a)}(\eta,V)$.


\section{Results with a Generative Model}
\label{sec: generative}
In the traditional Q-learning algorithm with a generative model oracle, at the timestep $t$, for each $(s,a)\in\SM\times\AM$, the update rule is:
\begin{align*}
    &Q_{t+1}(s,a)=(1-\beta_t)Q_t(s,a)+\beta_t\widehat{\TM} Q_t(s,a),\\
    &\widehat{\TM}Q_t(s,a):=r_t(s,a)+\gamma\max_{a'\in\AM}Q_t(s_t'(s,a),a'),
\end{align*}
where $s_t'(s,a)\sim P^*(\cdot|s,a)$ and $\EB r_t(s,a)=R(s,a)$. \cite{wainwright2019stochastic} provided a $\OM\left(T^{-\frac{1}{2}}\right)$ convergence rate when $\beta_t=\frac{1}{1+(1-\gamma)t}$. In their analysis, a key point is that $\widehat{\TM}Q_t$ is unbiased condition on $Q_t$. Analogously, in robust MDPs scenario, we can also learn optimal $Q^*_{\robp}$ by:
\[
    Q_{t+1}(s,a)=(1-\beta_t)Q_t(s,a)+\beta_t\widehat{\TM}_{\robp} Q_t(s,a),
\]
as long as we can obtain a ``good'' estimator $\widehat{\TM}_{\robp} Q_t(s,a)$, which is approximately unbiased ($\EB\widehat{\TM}_{\robp} Q_t(s,a)\approx\TM_{\robp}Q_t(s,a)$). Given the expression of $\widehat{\TM}_{\robp}Q$ in~\eqref{eq: robpQ}, we notice that stochastic gradient method can be applied to achieve this goal. In the following part, we investigate the error between $\EB\widehat{\TM}_{\robp} Q_t(s,a)$ and $\TM_{\robp}Q_t(s,a)$.

\subsection{Estimating $\TM_{\robp}Q$}
\label{sec: bandit}

As $\TM_{\robp}Q(s,a)=R(s,a)+\gamma \sup_{\eta}J^{(s,a)}(\eta;V)$, where $V(s):=\max_{a\in\AM}Q(s,a)$, we only need to study how to estimate $\sup_{\eta}J^{(s,a)}(\eta;V)$. The objective can be written by:
\begin{align*}
    &J^{(s,a)}(\eta, V)=-\lambda \sum_{s'\in\SM} P^*(s'|s,a)f^\dagger_{\eta, \lambda, V}(s')+\eta=\sum_{s'\in\SM}P^*(s'|s,a)J(\eta,V;s'), \\
    &J(\eta,V;s')=-\lambda f^\dagger_{\eta, \lambda, V}(s')+\eta.
\end{align*}
Next, we consider an online i.i.d.\ data stream $\{s'_t(s,a)\}_{t=0}^{T-1}$, where $s'_t(s,a)\sim P^*(\cdot|s,a)$. Then, we can apply  Stochastic Gradient Ascent (SGA) algorithm to approximate $\sup_{\eta}J^{(s,a)}(\eta,V)$:
\begin{align}
    \label{eq: sga}
    \eta_{t+1}(s,a)&=\eta_t(s,a)+\alpha_t\cdot\frac{\partial J(\eta_t(s,a),V;s_t'(s,a))}{\partial\eta},
\end{align}
where $\alpha_t$ is the learning rate, and $\frac{\partial J(\eta_t(s,a),V;s_t'(s,a))}{\partial\eta}$ is an unbiased estimator of $\frac{\partial J^{(s,a)}(\eta_t(s,a),V)}{\partial\eta}$. Noting that $J^{(s,a)}(\eta,V)$ must be concave w.r.t. $\eta$ as it is the dual form of problem~\eqref{eq: pdrop} \citep{boyd2004convex}, the convergence of SGA algorithm can be guaranteed. To specify the convergence rate, 
we make two basic assumptions for the objective $J^{(s,a)}(\eta, V)$. 
\begin{asmp}
    \label{asmp: diam}
    For any $V\in\left[0,(1-\gamma)^{-1}\right]^{|\SM|}$ and $(s,a)\in\SM\times\AM$, the optimal point $\eta^*(s,a)=\argmax_{\eta\in\RB}J^{(s,a)}(\eta, V)$ is finite. We can restrict the range of $\eta$ in $\Theta\subseteq\RB$, whose diameter is finite  (denoted  $\diam(\Theta)$) and is independent of $P^*$.
\end{asmp}
\begin{asmp}
    \label{asmp: smooth}
    $J^{(s,a)}(\eta,V)$ is $\frac{1}{\lambda\sigma}$-smooth w.r.t. $\eta\in\Theta$.
\end{asmp}
In Assumption~\ref{asmp: diam}, we assume a finite region of dual variables to exclude some extreme cases. In Assumption~\ref{asmp: smooth}, we assume the smoothness of $J^{(s,a)}(\eta,V)$. Indeed, by \cite{zhou2018fenchel}, if $f(\cdot)$ is a $\sigma$-strongly convex function, it comes $f^*(\cdot)$ is $1/\sigma$-smooth and $J^{(s,a)}(\eta,V)$ is $1/\sigma\lambda$-smooth. However, $\sigma$-strongly convexity of $f(\cdot)$ on $\RB$ may fail for some function $f(\cdot)$, such as Cressie-Read family of $f$-divergences \citep{cressie1984multinomial}. But with a given closed set $\Theta$, the smoothness of $J^{(s,a)}(\eta,V)$ can be guaranteed while the smoothness parameter may be dependent with the diameter of $\Theta$. In this scenario, on a finite region $\Theta$, the stochastic gradient can also be bounded (Lemma~\ref{lem: bound_grad}). Therefore, we can finally specify the convergence rate in Theorem~\ref{thm: bandit_converge}. The proofs are deferred to the Appendix~\ref{apd: bandit}.


\begin{lem}
    \label{lem: bound_grad}
    If Assumptions~\ref{asmp: diam} and \ref{asmp: smooth} hold, for any $\eta\in\Theta$ and $(s,a)\in\SM\times\AM$, $s'(s,a)\sim P^*(\cdot|s,a)$, then we have:
    \[
        \left|\frac{\partial J(\eta,V;s'(s,a))}{\partial \eta}\right|\le\frac{\diam(\Theta)+(1-\gamma)^{-1}}{\lambda \sigma}:=C_{g}.
    \]
\end{lem}

\begin{thm}[Convergence guarantee]
    \label{thm: bandit_converge}
    If Assumptions~\ref{asmp: diam} and \ref{asmp: smooth} hold,  the i.i.d.\ data stream $\{s'_t(s,a)\}_{t=0}^{T-1}$ is generated from $P^*$, and the learning rate satisfies $\alpha_t=\frac{\diam(\Theta)}{C_g\sqrt{t}}$, then for any given $V\in\left[0,(1-\gamma)^{-1}\right]^{|\SM|}$, the convergence rate of the SGA algorithm in Eqn. \eqref{eq: sga} is:
    \[
        \EB\left[\max_{(s,a)\in\SM\times\AM}\left(\sup_{\eta}J^{(s,a)}(\eta,V)-J^{(s,a)}(\eta_T(s,a),V)\right)\right]\le\frac{\diam(\Theta)C_g(2+\ln T)(4\sqrt{2\ln|\SM||\AM|}+1)}{\sqrt{T}}.
    \]
\end{thm}

For example, we apply Theorem~\ref{thm: bandit_converge} to a specific case where $f(t)=(t-1)^2$. Similar with Lemma~\ref{lem: chi2_range} in Appendix~\ref{apd: preli}, we can show $\Theta=[-2\lambda,2(1-\gamma)^{-1}+2\lambda]$, which satisfies Assumption~\ref{asmp: diam}. Then, the convergence rate becomes $\sup_{\eta}J(\eta,V)-\EB [J(\eta_T,V)]\le\frac{(3(1-\gamma)^{-1}+4\lambda)^2(2+\log T)}{\lambda\sigma\sqrt{T}}$.

\subsection{Learning $Q^*_{\robp}$}
\label{sec: onlinermdp}
In this section, we combine the gradient method in Section~\ref{sec: bandit} with Q-learning algorithm to learn the optimal robust Q-value function $Q_{\robp}^*$, where we run multiple gradient steps for dual variables $\eta$ between each Q-learning step in Algorithm~\ref{alg: onlinermdp_modification}. The high-level idea is if the number of multiple gradient steps are enough, then $\eta_{t,T'}(s,a)\approx\argmax_{\eta}J^{(s,a)}(\eta,V_t)$, which leads to $\EB_{s_t'} [J(\eta_{t,T'}(s,a), V_t;s_t'(s,a))]\approx \sup_{\eta}J^{(s,a)}(\eta,V_t)$. In Algorithm~\ref{alg: onlinermdp_modification}, $\Pi_{\Theta}$ is the projection onto $\Theta$ in the Euclidean norm. Moreover, we also need to make sure the range of $Q_{t}$ remains unchanged during the training process or it will blow up. Thus, we assume the range of $|J(\eta,V;s')|$ is bounded by a constant $C_M$ in Assumption~\ref{asmp: bound_f}. Then $|Q_{t}(s,a)|$ is also bounded by $C_M$ as $C_M\ge(1-\gamma)^{-1}$.

\begin{asmp}
    \label{asmp: bound_f}
    For any $\eta\in\Theta$, $V\in[0,(1-\gamma)^{-1}]$, and $s'\sim P^*(\cdot|s,a)$, we have $|J(\eta, V;s')|\le C_M$, where $C_M\ge(1-\gamma)^{-1}$.
\end{asmp}

\begin{algorithm}[htbp!]
    \caption{Model-free approach to robust MDPs }
    \label{alg: onlinermdp_modification}
    \begin{algorithmic}
        \STATE{\bfseries Input:} $Q_0(s,a)=(1-\gamma)^{-1}$ for all $(s,a)\in\SM\times\AM$.
        \FOR{iteration $t=0$ {\bfseries to} $T-1$}
            \STATE $V_t=\Pi_{[0,(1-\gamma)^{-1}]}\left(\max_a Q_t(\cdot,a)\right)$;
            \FOR{each state-action pair $(s,a)\in\SM\times\AM$}
                \STATE Set $\eta_{t,0}(s,a)=0$.
                \FOR{iteration $t'=0$ {\bfseries to} $T'-1$}
                \STATE Receive next state $s'_{t,t'}(s,a)\sim P^*(\cdot|s,a)$.
                \STATE $\eta_{t,t'+1}(s,a)=\Pi_{\Theta}\left(\eta_{t,t'}(s,a)+\alpha_{t'}\frac{\partial J(\eta_{t,t'}(s,a), V_{t};s'_{t,t'}(s,a))}{\partial\eta}\right)$;
                \ENDFOR
            \STATE Receive reward $r_t(s,a)$ and next state $s'_t(s,a)$;
            \STATE $Q_{t+1}(s,a)=(1-\beta_t)Q_t(s,a)+\beta_t(r_t+\gamma J(\eta_{t,T'}(s,a), V_t;s_t'(s,a)))$;
            \ENDFOR
        \ENDFOR
    \end{algorithmic}
\end{algorithm}

Below we give a proof sketch to the convergence guarantee for Algorithm~\ref{alg: onlinermdp_modification}. The detailed proofs of the lemmas and theorems in this section are deferred to Appendix~\ref{apd: onlinermdp}. To ease the notation, in Algorithm~\ref{alg: onlinermdp_modification}, 
we recursively define two sequences $(\FM_t)_{t=-1}^T$ and $(\GM_t)_{t=-1}^T$  by $(t\ge0)$:
\begin{align*}
    &\GM_{t-1} = \sigma\left(\FM_{t-1}\cup\sigma\left(\{s_{t,t'}'\}_{t'=0}^{T'-1}\right)\right), \notag \\
    &\FM_{t} = \sigma\left(\GM_{t-1}\cup \sigma\left(\{r_t, s_t'\}\right)\right),
\end{align*}
where $\FM_{-1}:=\sigma(\{\varnothing\})$.

\paragraph{Error Decomposition.} For each $(s,a)\in\SM\times\AM$, at iteration $t+1$, we have
\begin{align*}
    Q_{t+1}(s,a)-Q^*(s,a)  = & (1-\beta_t)(Q_{t}(s,a)-Q^*(s,a)) \notag\\
    &+\beta_t(r_t(s,a)-\EB [r_t(s,a)]+\gamma(\widehat{J}_t(s,a)-J^*(s,a))),
\end{align*}
where $\widehat{J}_t(s,a):=J(\eta_{t,T'}(s,a),V_t;s_t'(s,a))$ and $J^*(s,a)=\sup_{\eta}J^{(s,a)}(\eta;V^*_{\robp})$. We also construct auxiliary terms $\tilde{J}_t(s,a):=J^{(s,a)}(\eta_{t,T'}(s,a), V_t)$ and $\bar{J}_t(s,a):=\max_{\eta}J^{(s,a)}(\eta, V_t)$. We can decompose $\widehat{J}_t(s,a)-J^*(s,a)$ into three terms:
\[
    \widehat{J}_t(s,a)-J^*(s,a) := I_{t,1}(s,a)+I_{t,2}(s,a)+I_{t,3}(s,a),
\]
where
\begin{align*}
    I_{t,1}&:=\widehat{J}_t(s,a)-\tilde{J}_t(s,a),\\
    I_{t,2}&:=\tilde{J}_t(s,a)-\bar{J}_t(s,a),\\
    I_{t,3}&:=\bar{J}_t(s,a)-J^*(s,a).
\end{align*}
For $I_{t,1}(s,a)$, we observe that its mean is zero under event $\GM_{t-1}$, which means $\EB [I_{t,1}|\GM_{t-1}]=0$. For $I_{t,2}(s,a)$, it is controlled by optimization error in Theorem~\ref{thm: bandit_converge}, where we can determine $T'$ such that $\EB \|I_{t,2}\|_{\infty}\le\varepsilon_{\text{opt}}$. For $I_{t,3}(s,a)$, we find $|I_{t,3}(s,a)|\le\|V_t-V^*\|_\infty$ by primal objective~\eqref{eq: pen}. 
Denoting $\Delta_{t}(s,a)=Q_{t}(s,a)-Q^*(s,a)$ and $\varepsilon_{r,t}(s,a)=r_t(s,a)-\EB r_t(s,a)$, we have:
\begin{align*}
    \Delta_{t+1}(s,a)\le &(1-\beta_t)\Delta_t(s,a)+\beta_t (\varepsilon_{r,t}(s,a)+\gamma I_{t,1}(s,a)+\gamma I_{t,2}(s,a)+\gamma I_{t,3}(s,a))\notag \\
    \le&(1-\beta_t)\Delta_t(s,a)+\beta_t (\varepsilon_{r,t}(s,a)+\gamma I_{t,1}(s,a)+\gamma\|I_{t,2}\|_{\infty}\1+\gamma\|\Delta_t\|_{\infty}\1).
\end{align*}
Reversely, we also have:
\begin{align*}
    \Delta_{t+1}(s,a)\ge &(1-\beta_t)\Delta_t(s,a)+\beta_t(\varepsilon_{r,t}(s,a)+\gamma I_{t,1}(s,a)-\gamma \|I_{t,2}\|_{\infty}\1-\gamma\|\Delta_t\|_{\infty}\1).
\end{align*}
Then we construct auxiliary sequences:
\begin{align*}
    a_{t+1} &= (1-\beta_t(1-\gamma))a_t,\\
    b_{t+1} &= (1-\beta_t(1-\gamma))b_t+\gamma\beta_t\|N_t\|_{\infty},\\
    c_{t+1} &= (1-\beta_t(1-\gamma))c_t+\gamma\beta_t\|I_{t,2}\|_{\infty},\\
    N_{t+1}(s,a) &= (1-\beta_t)N_{t}(s,a)+\beta_t(\varepsilon_{r,t}(s,a)+\gamma I_{t,1}(s,a)),
\end{align*}
where $a_0=\|\Delta_0\|_{\infty}$, $b_0=c_0=0$ and $N_0=\0$. It can be verified:
\[
    -(a_t+b_t+c_t)\1+N_t\le\Delta_{t}\le (a_t+b_t+c_t)\1 +N_t.
\]

\paragraph{Concentration on $N_t$.} Noting that $\varepsilon_{r,t}(s,a)$ and $I_{t,1}(s,a)$ are bounded mean zero random variables, we can construct a Hoeffding bound for $N_t$. 

\begin{lem}
    \label{lem: expect_N}
    If $(1-\beta_t)\beta_{t-1}\le\beta_t$ and Assumption~\ref{asmp: bound_f} holds, then the expectation of $\|N_t\|_{\infty}$ satisfies:
    \[
        \EB[\|N_t\|_{\infty}]\le2\sqrt{2\beta_{t-1}(1+\gamma C_M)^2\ln(2|\SM||\AM|)}.
    \]
\end{lem}

\paragraph{Concentration on $I_{t,2}$.} Directly applying Theorem~\ref{thm: bandit_converge} with $V=V_t$, we deduce the following convergence rate.

\begin{lem}
    At any time step t, if $\alpha_{t'}=\frac{\diam(\Theta)}{C_g\sqrt{t'}}$,  then
    \[
        \EB\|I_{t,2}\|_{\infty}\le\frac{\diam(\Theta)C_g(2+\ln T')(4\sqrt{2\ln|\SM||\AM|}+1)}{\sqrt{T'}}.
    \]
\end{lem}

\paragraph{Convergence of $a_t, b_t$ and $c_t$.} We can write the explicit expressions of $a_t,b_t$ and $c_t$ with $\beta_t$, $\|N_t\|_{\infty}$ and $\|I_{t,2}\|_{\infty}$.

\begin{lem}
    \label{lem: abc} It is true that
    $a_T, b_t$, and $c_T$ satisfies
    \begin{align*}
        a_T &= a_0\cdot\prod_{t=0}^{T-1}(1-\beta_t(1-\gamma)),\\
        b_T &= \sum_{t=0}^{T-1}\gamma\beta_{t}\|N_t\|_{\infty}\cdot\prod_{i=t+1}^{T-1}(1-\beta_i(1-\gamma)),\\
        c_T &= \sum_{t=0}^{T-1}\gamma\beta_{t}\|I_{t,2}\|_{\infty}\cdot\prod_{i=t+1}^{T-1}(1-\beta_i(1-\gamma)),
    \end{align*}
    where $\prod_{t=i}^j x_t:=1$ if $i>j$ for any sequence $\{x_t\}$. 
\end{lem}

By Lemma~\ref{lem: abc}, we can write an explicit expression of the deviation $\Delta_t$ in the following lemma. The first term is the upper bound of $\EB[\|N_T\|_{\infty}]$ in Lemma~\ref{lem: expect_N}. The rate of this term is determined by the learing rate $\beta_T$. The second term arises from the expression of $a_T$. In the braces, the third and forth term arise from the expression of $b_T$ and $c_T$ respectively.
\begin{lem}
    \label{lem: bound_delta}
    We have:
    \begin{align*}
        \EB\|\Delta_T\|_{\infty}&\le\sqrt{\beta_T} C_N+\|\Delta_0\|_{\infty}\cdot\left(1-\sum_{t=0}^{T-1}\beta_{t,T-1}\right)+\frac{\gamma}{1-\gamma}\sum_{t=0}^{T-1}\beta_{t,T-1}\left(\sqrt{\beta_{t-1}}C_{N}+\EB\|I_{t,2}\|_{\infty}\right),
    \end{align*}
    where $C_N:=2 \sqrt{2(1 + \gamma C_M)^2\log(2|\SM||\AM|)}$, and $\beta_{i,j}:=(1-\gamma)\beta_i\prod_{t=i+1}^j(1-\beta_t(1-\gamma))$.
\end{lem}

Finally, we specify the learning rate to be $\beta_t=\frac{1}{1+(1-\gamma)(1+t)}$. With all we have ahead, we have the convergence result of Algorithm~\ref{alg: onlinermdp_modification} in Theorem~\ref{thm: Delta_converge}.

\begin{thm}
    \label{thm: Delta_converge}
    If Assumptions~\ref{asmp: diam} and~\ref{asmp: bound_f}  hold, then for $\alpha_{t'}=\frac{\diam(\Theta)}{C_g\sqrt{t'}}$ and $\beta_t=\frac{1}{1+(1-\gamma)(t+1)}$, we have:
    \begin{align}
        \label{eq: delta}
        \EB[\|\Delta_T\|_{\infty}]\le&\frac{\|\Delta_0\|_{\infty}}{1+(1-\gamma)T}+\frac{2C_N}{\sqrt{(1-\gamma)^3T}}
        +\frac{\diam(\Theta)C_g(2+\log T')}{(1-\gamma)\sqrt{T'}}+\frac{C_N}{\sqrt{1+(1-\gamma)T}},
    \end{align}
    where $C_M$, $\diam(\Theta)$ and $C_g$ are dependent with choice of $f$ and $C_N:=2\sqrt{2(1+\gamma C_M)^2\log(2|\SM||\AM|)}$. Thus, to obtain an $\varepsilon$-optimal Q-value function, the total number of sample complexity is: 
    \begin{align*}
        |\SM||\AM|\cdot&\OM\left(\frac{C_N^2}{\varepsilon^2(1-\gamma)^3}\right)\cdot\widetilde{\OM}\left(\frac{\diam(\Theta)^2 C_g^2}{\varepsilon^2(1-\gamma)^2}\right)
        =\widetilde{\OM}\left(\frac{|\SM||\AM|\diam(\Theta)^2 C_N^2 C_g^2}{\varepsilon^4(1-\gamma)^5}\right).
    \end{align*}
\end{thm}

To be more concrete, we still apply $f(t)=(t-1)^2$ to this theorem. In this case, we can verify $\diam(\Theta)=\frac{2}{1-\gamma}+4\lambda$, $C_M=\frac{1}{1-\gamma}$, $C_N=\frac{2\sqrt{2\log(2|\SM||\AM|)}}{1-\gamma}$ and $C_g=\frac{3}{4\lambda(1-\gamma)}+2$. Then, the total sample complexity for $f(t)=(t-1)^2$ is $\widetilde{\OM}\left(\frac{|\SM||\AM|}{\varepsilon^4(1-\gamma)^7}\max\left\{\frac{1}{\lambda^2(1-\gamma)^4},16\lambda^2\right\}\right)$.

\paragraph*{Discussion.}  So far, Theorem~\ref{thm: Delta_converge} answers how many samples are sufficient to guarantee an $\varepsilon$-optimal Q-value function for robust MDPs without an oracle to DRO solutions. We will discuss some points on the setting of parameters in Algorithm~\ref{alg: onlinermdp_modification}.

\begin{itemize}
    \item Choice of $\alpha_{t'}$: The inner optimization problem is indeed a convex stochastic optimization problem. \citet{fontaine2021convergence} proved that the convergence rate would be $\widetilde{\OM}(T^{-k\wedge(1-k)})$ if $\alpha_{t'}=t'^{-k}$ where $k\in(0,1)$. Thus, the choice of $\alpha_{t'}$ is the best we hope for in Theorem~\ref{thm: bandit_converge}.

    \item Choice of $\beta_{t}$: By Lemma~\ref{lem: bound_delta}, the convergence will still be guaranteed if we choose another learning rate scheme such as $\beta_t=t^{-k}$, where $k\in(0,1)$. As pointed out in \cite{wainwright2019stochastic}, the convergence rate would be slower than linear scale learning rate.

    \item Choice of $\eta_{t,0}$: In Theorem~\ref{thm: bandit_converge}, we find the convergence rate is not related with the initial point, which is due to a loose inequality with $\diam(\Theta)$. Therefore, in Algorithm~\ref{alg: onlinermdp_modification}, we force the initial point of inner optimization problem to be fixed at zero. In practice, we can set $\eta_{t,0}=\eta_{t-1,T'}$ to save iteration complexity. However, how to theoretically prove it is challenging because the value function $V_t$ is changing w.r.t. $t$.

    \item Choice of $T'$. By Lemma~\ref{lem: bound_delta}, the coefficient of $\EB[\|I_{t,2}\|_{\infty}]$ is $\beta_{t,T-1}$. Thus, we don't need to require $\EB[\|I_{t,2}\|_{\infty}]\le\varepsilon_{opt}$ for a fixed optimization error $\varepsilon_{opt}$ at any time step $t$. Instead, as long as $\sum_{t=0}^{T-1}\beta_{t,T-1}\EB[\|I_{t,2}\|_{\infty}]$ converges finitely, the optimization error $\EB[\|I_{t,2}\|_{\infty}]$ can vary w.r.t. $t$.
\end{itemize}


\section{Results with Markovian Data}
\label{sec: markovian}
In Section~\ref{sec: generative}, we introduce Algorithm~\ref{alg: onlinermdp_modification} to learn the optimal robust Q-value function. The data generating mechanism is the generative model. However, the generative model is far away from the realistic scenario. In this section, we consider a more practical data generating mechanism, namely Markovian data \citep{li2020sample}. Under the mechanism, we can only observe the samples from a single trajectory $(s_0, a_0, s_1, a_1,\cdots)$, where $s_0\sim\mu(\cdot)$, $a_t\sim\pi_b(\cdot|s_t)$, $s_{t+1}\sim P^*(\cdot|s_t,a_t)$. Unlike the generative model, we can not query next states for each $(s,a)$ pairs. Then, Algorithm~\ref{alg: onlinermdp_modification} does not fit in this setting. A straightforward modification to the algorithm is that we only update the $Q$-value and dual variable for the current visited $(s,a)$ pair as shown in Algorithm~\ref{alg: markoviandata_new}.

\begin{algorithm}[htbp!]
    \caption{Model-free approach to robust MDPs (Markovian Data) }
    \label{alg: markoviandata_new}
    \begin{algorithmic}
        \STATE{\bfseries Input:} $Q_0(s,a)=1/(1-\gamma)$ for all $(s,a)\in\SM\times\AM$.
        \FOR{iteration $t=0$ {\bfseries to} $T-1$}
            \STATE Perform $a_t\sim\pi(\cdot|s_t)$, and receive reward $r(s_t,a_t)$ and next state $s_{t+1}\sim P^*(\cdot|s_t,a_t)$;
            \STATE $Q_{t+1}(s_t,a_t)=(1-\beta_{t})Q_t(s_t,a_t)+\beta_{t}(r(s_t,a_t)+\gamma \widehat{J}^{(s_t,a_t)}(\eta(s_t,a_t); V_t))$;
            \STATE $\eta_{t+1}(s_t,a_t)=\Pi_{\Theta}\left(\eta_t(s_t,a_t)+\alpha_{t}\frac{\partial \widehat{J}^{(s_t,a_t)}(\eta(s_t,a_t); V_{t})}{\partial\eta}\right)$;
            \STATE $V_{t+1}=\max_a Q_{t+1}(\cdot,a)$;
        \ENDFOR
    \end{algorithmic}
\end{algorithm}

However, in order to guarantee the convergence of Algorithm~\ref{alg: markoviandata_new}, several additional assumptions are needed. The first one is the induced Markovian chain by policy $\pi_b$ converges to its stationary distribution geometrically fast. In Assumption~\ref{asmp: fastmixing}, the convergence rate $\rho$ is related with the mixing time $\tau_{\text{mix}}$ by $\tau_{\text{mix}}=\frac{\ln 4M}{\log\rho^{-1}}\approx\frac{\ln4M}{1-\rho}$ when $\rho$ approaches 1. A fast mixing Markovian chain implies that the pairs $(s_t,a_t)$ we observe are almost i.i.d.\ generated from the stationary distribution $d_{\pi_b}(\cdot)$ as long as $t$ is sufficiently large.

\begin{asmp}
    \label{asmp: fastmixing}
    For the given policy $\pi$ in ~\ref{alg: markoviandata_new}, the Markovian chain $(s_0, a_0, s_1, a_1,\cdots)$ is fast mixing, that is, 
    \begin{align*}
        \sup_{(s,a)\in\SM\times\AM}d_{TV}\left(P_t^\pi(\cdot|(s,a)),d_{\pi}(\cdot)\right)\le M\rho^t,
    \end{align*}
    where $M>0$ and $\rho\in(0,1)$. And $P^\pi_t:=(P^\pi)^t$, and $d_\pi(\cdot)$ is the stationary distribution of the Markovian chain, which satisfies:
    \begin{align*}
        d_\pi^\top=d_\pi^\top P^\pi.
    \end{align*}
    Moreover, we denote $d_{\min}:=\min_{(s,a)\in\SM\times\AM}d_{\pi}(s,a)$.
\end{asmp}

In addition, we also require additional assumptions for $J^{(s,a)}(\eta;V)$. Assumption~\ref{asmp: monotone} implies the objective is strongly-convex at its optimal point. In this case, the convergence rate for solving $\sup_{\eta}J^{(s,a)}(\eta;V)$ can be faster than the convex case, which enables we alternatively update variables $\eta_t$ and $Q_t$ in Algorithm~\ref{alg: markoviandata_new}. Moreover, by the dual objective \eqref{eq: pdrod}, we have Lemma~\ref{asmp: eta_lip_v},  guaranteeing that  the optimal solutions w.r.t.\ different values do not differ too much if the values are close. 
\begin{asmp}
    \label{asmp: monotone}
    For any given $V\in[0,(1-\gamma)^{-1}]^{|\SM|}$, there exists $\kappa>0$ such for each $(s,a)\in\SM\times\AM$ that $J^{(s,a)}(\eta;V)$ satisfies:
    \[
        \nabla J^{(s,a)}(\eta; V)\cdot(\eta-\eta_{V}^*(s,a))\le-\kappa\cdot(\eta-\eta_V^*(s,a))^2.
    \]
\end{asmp}
\begin{lem}
    \label{asmp: eta_lip_v}
    For any $V_1,V_2\in[0,(1-\gamma)^{-1}]^{|\SM|}$, we have:
    \[
        \left|\eta^*_{V_1}(s,a)-\eta^*_{V_2}(s,a)\right|\le \left\|V_1-V_2\right\|_{\infty}.
    \]
\end{lem}
Taking $f(t)=(t-1)^2$ as an example, for any given $(s,a)\in\SM\times\AM$, we have:
\begin{align*}
    J^{(s,a)}(\eta; V)&=-\lambda \EB_{s'\sim P^*(\cdot|s,a)}f^{*}\left(\frac{\eta-V(s')}{\lambda}\right)+\eta \notag \\
    &=-\frac{\EB_{s'\sim P^*(\cdot|s,a)}[\left(\eta+2\lambda-V(s')\right)_{+}^2]}{4\lambda}+\lambda+\eta.
\end{align*}
Thus, the gradient of $J^{(s,a)}(\eta;V)$ w.r.t. $\eta$ is:
\begin{align}
    \nabla J^{(s,a)}(\eta;V)=-\frac{\EB_{s'\sim P^*(\cdot|s,a)}[\left(\eta+2\lambda-V(s')\right)_{+}]}{2\lambda}+1.
\end{align}
Then we observe that $\eta_V^*(s,a)$ satisfies $ \nabla J^{(s,a)}(\eta_V^*(s,a);V)=0$. Moreover, the subgradient of $\nabla J^{(s,a)}(\eta;V)$ at $\eta_V^*(s,a)$ satisfies:
\[
    \nabla^2 J^{(s,a)}(\eta_V^*(s,a);V)\le -\frac{\EB_{s'\sim P^*(\cdot|s,a)}[\1(\eta^*_V(s,a)+2\lambda-V(s')>0)]}{2\lambda}.
\]
By the Cauchy–Schwarz inequality, we have:
\begin{align*}
    4\lambda^2&=\left(\EB_{s'\sim P^*(\cdot|s,a)}[\left(\eta^*_V(s,a)+2\lambda-V(s')\right)_{+}]\right)^2\notag \\
    &\le\EB_{s'\sim P^*(\cdot|s,a)}[\left(\eta^*_V(s,a)+2\lambda-V(s')\right)_{+}^2]\cdot\EB_{s'\sim P^*(\cdot|s,a)}[\1(\eta^*_V(s,a)+2\lambda-V(s')>0)].
\end{align*}
Moreover, we notice $J^{(s,a)}(\eta^*_V(s,a);V)\ge0$ by primal objective. Thus,
\[
    4\lambda^2\le4\lambda(\lambda+\eta^*_V(s,a))\EB_{s'\sim P^*(\cdot|s,a)}[\1(\eta^*_V(s,a)+2\lambda-V(s')>0)].
\]
In addition, we have $\eta^*_V(s,a)\in[-\lambda,2(1-\gamma)^{-1}+2\lambda]$, which leads to
\begin{align}
    \nabla^2 J^{(s,a)}(\eta_V^*(s,a);V)\le-\frac{1}{2(\lambda+\eta_V^*(s,a))}\le-\frac{1}{6(\lambda+(1-\gamma)^{-1})}.
    \label{eq: chi2_kappa}
\end{align}
Then,  Assumption~\ref{asmp: monotone} holds for $f(t)=(t-1)^2$ with $\kappa^{-1}=6(\lambda+(1-\gamma)^{-1})$.

We now give a proof sketch to the convergence for Algorithm~\ref{alg: markoviandata_new}. The detailed proofs in this section are deferred to Appendix~\ref{apd: markovian}.

\paragraph*{Error Decomposition.} Here we denote the error $\Delta_t(s,a)=Q_{t+1}(s,a)-Q^*_{\robp}(s,a)$ and $\xi_t$ is a random variable on $\SM\times\AM$ satisfying $P(\xi_{t+1}=(s',a')|\xi_t=(s,a))=P^*(s'|s,a)\pi_b(a'|s')$. Thus, by Algorithm~\ref{alg: markoviandata_new}, we have:
\begin{align}
    \Delta_{t+1}(s,a)=&(1-\beta_t\1(\xi_t=(s,a)))\Delta_t(s,a)+\beta_t\1(\xi_t=(s,a))(\varepsilon_{r,t}(s,a)+\gamma\varepsilon_{J,t}(s,a)) \nonumber \\
        &+\gamma\beta_t\1(\xi_t=(s,a))\left(J^{(s,a)}(\eta_t(s,a);V_t)-J^{(s,a)}(\eta_t^*;V_t)\right)\nonumber \\
        &+\gamma\beta_t\1(\xi_t=(s,a))\left(J^{(s,a)}(\eta_t^*;V_t)-J^{(s,a)}(\eta^*;V^*_{\robp})\right).\label{eq: markovian_decomp}
\end{align}
Different from Section~\ref{sec: generative}, here we introduce a new random variable sequence $\{\xi_t\}_{t\ge0}$ to represent only one pair $(s,a)$ occurs at each time step. Moreover, we denote the filtration $\{\FM_t\}_{t\ge0}$ by ($\FM_0:=\sigma(\{\varnothing\})$):
\[
    \FM_{t}=\sigma\left(\{\xi_i\}_{i=0}^t\cup\{r_i\}_{i=0}^{t-1}\right).
\]
In the decomposition, we denote:
\begin{align*}
    &Z_{t,1}(s,a):=\1(\xi_t=(s,a))(\varepsilon_{r,t}(s,a)+\gamma\varepsilon_{J,t}(s,a)),\\
    &Z_{t,2}(s,a):=\gamma\1(\xi_t=(s,a))\left(J^{(s,a)}(\eta_t(s,a);V_t)-J^{(s,a)}(\eta_t^*;V_t)\right),\\
    &Z_{t,3}(s,a):=\gamma\1(\xi_t=(s,a))\left(J^{(s,a)}(\eta_t^*;V_t)-J^{(s,a)}(\eta^*;V^*_{\robp})\right).
\end{align*}
It is worth noticing that $\EB [Z_{t,1}(s,a)|\FM_t]=0$, $|Z_{t,2}(s,a)|\le\frac{\gamma}{2\sigma\lambda}|\eta_t(s,a)-\eta_{t}^*(s,a)|^2$ by smoothness property in Assumption~\ref{asmp: smooth}, and $|Z_{t,3}(s,a)|\le\gamma\1(\xi_t=(s,a))\|\Delta_t\|_{\infty}$. Then, two things left to be done: (a) dealing with the $\xi_t$ in the recursion; (b) controlling the error $\|\eta_t-\eta_t^*\|_{\infty}$.

\paragraph*{Dealing $\xi_t$.} In a high-level idea, when $t$ is sufficiently large,  we have $\EB\1(\xi_t=(s,a))\approx d_\pi(s,a)$ by fast mixing assumption~\ref{asmp: fastmixing}. Thus, we can write the decomposition \eqref{eq: markovian_decomp} into an abstract form:
\[
    \Delta_{t+1}(s,a)=(1-\beta_t d_{\pi}(s,a))\Delta_t(s,a)+\beta_t f_t(s,a)(1(\xi_t=(s,a))-d_\pi(s,a)),
\]
where $\|f_t\|$ is almost surely bounded and adaptive to $\FM_t$. Then we decompose $\1(\xi_t=(s,a))-d_\pi(s,a)$ into:
\begin{align*}
    \1(\xi_t=(s,a))-d_\pi(s,a)&=\sum_{k=0}^\infty \left(P_k(s,a|\xi_t)-d_\pi(s,a)\right)-\sum_{k=1}^\infty \left(P_k(s,a|\xi_t)-d_\pi(s,a)\right)\notag\\
    &:=\psi(s,a;\xi_t)-\PM\psi(s,a;\xi_t),
\end{align*}
where $\PM\psi(s,a;\xi_t):=\sum_{s',a'}\psi(s,a;s',a')P^{\pi}(s',a'|\xi_t)$ is one-step transition on $\psi$. This decomposition is also called Poisson equation method \citep{metpri87, benveniste2012adaptive, li2023online}. Next, we plug in $\PM \psi(s,a;\xi_{t-1})$, which happens to be the conditional expectation $\EB[\psi(s,a;\xi_t)|\FM_t]$. Thus, the error can also be written by:
\begin{align}
    \Delta_{t+1}(s,a)=&(1-\beta_t d_{\pi}(s,a))\Delta_t(s,a)+\beta_t f_t(s,a)\left(\psi(s,a;\xi_t)-\PM(s,a;\xi_{t-1})\right)\notag\\
    &+\beta_t f_t(s,a)\left(\PM\psi(s,a;\xi_{t-1})-\PM\psi(s,a;\xi_{t})\right).
    \label{eq: delta_xi}
\end{align}
By Assumption~\ref{asmp: fastmixing}, we can show $|\psi(s,a;\xi_t)|\le\frac{M}{1-\rho}$. Thus, Azuma-Hoeffding can be applied to deal with error induced by $\psi(s,a;\xi_t)-\PM(s,a;\xi_{t-1})$. For the error induced by $\PM\psi(s,a;\xi_{t-1})-\PM\psi(s,a;\xi_{t})$, we can replace it with error induced by $f_{t+1}(s,a)-f_t(s,a)$ according to change of summation. We leave the details to Appendix~\ref{apd: aux}.

\paragraph*{Controlling $\|\eta_t-\eta_t^*\|_{\infty}$.} A key obstacle for controlling $\|\eta_t-\eta_t^*\|_{\infty}$ is that $V_t$ keeps varying at each time step. By update rule in Algorithm~\ref{alg: markoviandata_new}, the error can be decomposed to two major terms (we ignore all the parameters independent with $t$ here):
\begin{align*}
    (\eta_{t+1}(s,a)-\eta_{t+1}^*(s,a))^2\approx &(1-\OM(\alpha_t))(\eta_{t}(s,a)-\eta_{t}^*(s,a))^2+\OM\left(\frac{(\eta_{t}^*(s,a)-\eta_{t+1}^*(s,a))^2}{\alpha_t}\right)\notag \\
    &+H_t(s,a),
\end{align*}
where $H_t(s,a)$ are some rest random terms induced by $\xi_t$, which is handled similarly with Eqn~\eqref{eq: delta_xi}. With Lemma~\ref{asmp: eta_lip_v}, we have $(\eta_{t}^*(s,a)-\eta_{t+1}^*(s,a))\le\|V_{t}-V_{t+1}\|_{\infty}$. Besides, by update rule in Algorithm~\ref{alg: markoviandata_new}, we have $\|V_{t}-V_{t+1}\|_{\infty}=\OM(\beta_t)$. Thus, we replace the error decomposition with:
\begin{align*}
    (\eta_{t+1}(s,a)-\eta_{t+1}^*(s,a))^2\approx &(1-\OM(\alpha_t))(\eta_{t}(s,a)-\eta_{t}^*(s,a))^2+\OM\left(\frac{\beta_t^2}{\alpha_t}\right)+H_t(s,a).
\end{align*}
Then, with some proper chosen $\alpha_t$ and $\beta_t$, we have the following lemma to determine the convergence rate of $\|\eta_t-\eta_t^*\|_{\infty}$.

\begin{lem}
    \label{lem: markov_delta}
  Let $\eta_t^*(s,a):=\argmax_{\eta}J^{(s,a)}(\eta;V_t)$ and $\delta_t(s,a):=\eta_t(s,a)-\eta_t^*(s,a)$ in Algorithm~\ref{alg: markoviandata_new}. Then the following inequality holds:
    \begin{align*}
        \EB\|\delta_{t+1}\|^2_{\infty}\le\frac{\Phi_1}{(t+1)^{\frac{1}{3}}}+\frac{\Phi_2\ln(t+1+p^\dagger)}{(t+1)^{\frac{2}{3}}}+\frac{\Phi_3}{(t+1)^{\frac{2}{3}}},
    \end{align*}
    where $\Phi_1$, $\Phi_2$ and $\Phi_3$ are dependent with instance parameters, which are deferred to Appendix~\ref{apd: markovian}.
\end{lem}

Finally, combing Lemma~\ref{lem: markov_delta} and error decomposition for $\Delta_t$ in Eqn~\eqref{eq: delta_xi}, the following Theorem specify the final convergence rate.
\begin{thm}
    \label{thm: markov_q}
    If Assumptions~\ref{asmp: fastmixing} and \ref{asmp: monotone} hold, by taking $\alpha_t = \frac{1}{\kappa d_{\min}(t+p_{\alpha})^{\frac{2}{3}}}$ and $\beta_t = \frac{1}{(1-\gamma)d_{\min}(t+p^\dagger)}$, where $p_{\alpha}=\left\lceil\left(\frac{d_{\max}}{d_{\min}}\right)^{\frac{3}{2}}\right\rceil$ and $p^\dagger=\lceil\frac{d_{\max}}{(1-\gamma)d_{\min}}\rceil$, then the convergence rate for Algorithm~\ref{alg: markoviandata_new} satisfies:
    \begin{align*}
        \EB\|Q_t-Q_{\robp}^*\|_{\infty}\le\widetilde{\OM}\left(\frac{\Phi_1}{\sigma\lambda d_{\min}(1-\gamma)^2(t+1)^{\frac{1}{3}}}\right),
    \end{align*}
    where the expression of $\Phi_1$ is deferred to Appendix~\ref{apd: markovian}.
\end{thm}

Here we also take $f(t)=(t-1)^2$ as an example to calculate the specific convergence rate in Theorem~\ref{thm: markov_q}. By Eqn~\eqref{eq: chi2_kappa}, we have $\kappa = \frac{1}{6(\lambda+(1-\gamma)^{-1})}$. In addition, we also take $\lambda=\frac{1}{1-\gamma}$, as the expression of $\Phi_1$ is complicated. Thus, the convergence rate is:
\begin{align*}
    \EB\|Q_t-Q_{\robp}^*\|_{\infty}\le\widetilde{\OM}\left(\frac{\max\left\{\frac{M\sqrt{\ln2|\SM||\AM|}}{1-\rho},\frac{1}{(1-\gamma)^2}\right\}}{d_{\min}^3(1-\gamma)^3(t+1)^{\frac{1}{3}}}\right).
\end{align*}

\paragraph*{Discussion.} Theorem~\ref{thm: markov_q} presents the sample complexity of Algorithm~\ref{alg: markoviandata_new} is $\OM\left(T^{-\frac{1}{3}}\right)$. Compared with Theorem~\ref{thm: Delta_converge}, the dependence on $T$ is improved. The major contribution belongs to Assumption~\ref{asmp: monotone}, where we assume the objective is local strongly-convex. Indeed, Theorem~\ref{thm: Delta_converge} can also be improved to the same order $\OM\left(T^{-\frac{1}{3}}\right)$ if Assumption~\ref{asmp: monotone} holds in Section~\ref{sec: generative}. However, in the case of Markovian data, the convergence can not be guaranteed if Assumption~\ref{asmp: monotone} is blocked. This is due to we require the convergence of $\eta_t$ is faster than $Q_t$ to control the overall error. In addition, we stay positive on improving the convergence rate from $\OM\left(T^{-\frac{1}{3}}\right)$ to $\OM\left(T^{-\frac{1}{2}}\right)$ if Polyak-averaging technique \citep{polyak1992acceleration} is applied, which we leave for subsequent works.

\section{Experiments}
\label{sec: experiment}
In this section, we verify our theory from the following aspects: (a) The connection between robust value function and non-robust value function, (b) The relationship between the statistical error and robustness parameter $\lambda$, (c) The convergence result of Algorithm~\ref{alg: onlinermdp_modification}, and (d) The convergence result of Algorithm~\ref{alg: markoviandata_new}. For (d), we need to make some minor changes to the setting and we provide the experimental details to Section~\ref{sec: exp_mark}.

\subsection{Experimental Details}
\label{apd: exp}
\begin{figure}[htbp!]
    \centering
    \includegraphics[scale=0.8]{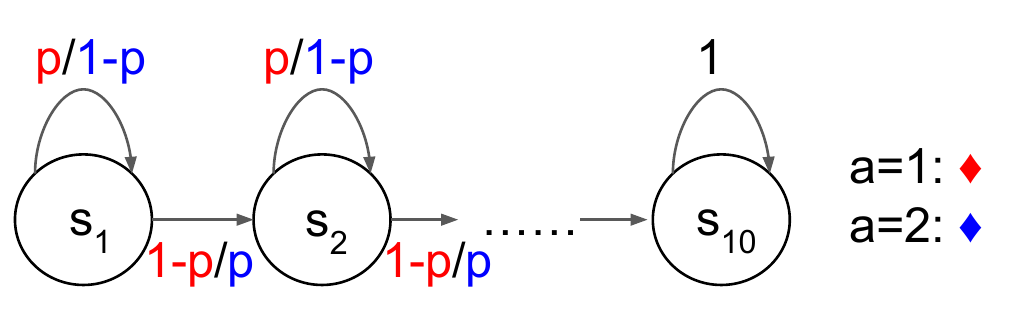}
    \caption{MDP with 10 states 2 actions. The transition probabilities are marked with red for taking action $a=1$, while the transition probabilities are marked blue taking action $a=2$.}
    \label{fig: 10state}
\end{figure}
We use a 10-state MDP environment (Figure \ref{fig: 10state}) at first. At state $s_i$ where $1\leq i \leq 9$, the transition probability is given by the following rules: When taking action $1$, $P^*(s_{i} | s_{i}, 1) = p$ and $P^*(s_{i+1} | s_{i}, 1) = 1-p$; When taking action $2$, the probability is opposite $P^*(s_{i} | s_{i}, 2) = 1 - p$ and $P^*(s_{i+1} | s_{i}, 2) = p$). At state $s_{10}$, the agent is always transited back to the same state. The reward is always $1$ except that transitions at $s_{10}$ always gives $0$. The discount rate is $\gamma=0.9$.


To obtain the true value functions, we run value iteration algorithms to achieve them. For non-robust optimal value function $V^*$, we run standard value iteration algorithm and set the iteration step being $T=10000$. For robust optimal value function $V^*_{\robp}$, we run a robust value iteration algorithm (Algorithm~\ref{alg: model_based}) with the transition probability $P^*$ and set $T=10000$ (outer loop steps) and $T'=1000$ (inner loop steps) to make sure the the dual variables and robust Q-values converging.

In Section~\ref{sec: connection} and \ref{sec: stat_error}, we apply a model-based method to learn $\widehat{V}_{\robp}$ and $\widehat{Q}_{\robp}$. First, we estimate $\widehat{P}$ with $1000$ transitions collected with model $P^*$ for each $(s,a)$. Then, we run Algorithm~\ref{alg: model_based} with $\widehat{P}$ to obtain $\widehat{V}_{\robp}$. We use $T=T'=100$ and set $\eta_{t,0}=\eta_{t-1, T'}$ to save steps. In these sections, we test several settings with different choice of $\lambda$: $\{0.5, 1.0, 2.0, 3.0, 4.0, \allowbreak 5.0, 10.0\}$. 
Moreover, the learning rate in the inner loop is set to be a constant $\alpha_t'=\lambda$ by the fact the smoothness of the dual objective is $1/\lambda$.
\begin{algorithm}[htbp!]
    \caption{Model-based approach to robust MDPs }
    \label{alg: model_based}
    \begin{algorithmic}
        \STATE{\bfseries Input:} $Q_0(s,a)=(1-\gamma)^{-1}$ for all $(s,a)\in\SM\times\AM$, and transition probability $P$.
        \FOR{iteration $t=0$ {\bfseries to} $T-1$}
            \STATE $V_t=\Pi_{[0,(1-\gamma)^{-1}]}\left(\max_a Q_t(\cdot,a)\right)$;
            \FOR{each state-action pair $(s,a)\in\SM\times\AM$}
                \STATE Set $\eta_{t,0}(s,a)=0$.
                \FOR{iteration $t'=0$ {\bfseries to} $T'-1$}
                \STATE $\eta_{t,t'+1}(s,a)=\Pi_{\Theta}\left(\eta_{t,t'}(s,a)+\alpha_{t'}\sum_{s'}P(s'|s,a)\frac{\partial J(\eta_{t,t'}(s,a), V_{t};s')}{\partial\eta}\right)$;
                \ENDFOR
            \STATE $Q_{t+1}(s,a)=R(s,a)+\gamma \sum_{s'}P(s'|s,a)J(\eta_{t,T'}(s,a), V_t;s')$;
            \ENDFOR
        \ENDFOR
    \end{algorithmic}
\end{algorithm}

In Section~\ref{sec:convergence},  we run the model-free algorithm with a generative model (Algorithm~\ref{alg: onlinermdp_modification}) to learn $\widehat{Q}_{\robp}$. In this section, we set $T=1000$, and $\eta_{t,0}$ is set to $0$ at the beginning of each inner loop. 
Moreover, we sweep $\lambda$ with the same values listing as above, and test different $T'$ settings in $[10, 50, 100]$. 
The learning rate in the outer loop is $\frac{1}{1 + (1 - \gamma) t}$ and the inner loop has learning rate $\frac{\lambda}{\sqrt{t'}}$, where $t$ refers to the iteration at the outer loop and $t'$ is the iteration at the inner loop. We repeat each experiment 100 times using different random seeds to account for noise.

\subsection{Connection to Non-robust Value Functions}
\label{sec: connection}
\begin{figure}[h!]
    \centering
    \includegraphics[scale=1.0]{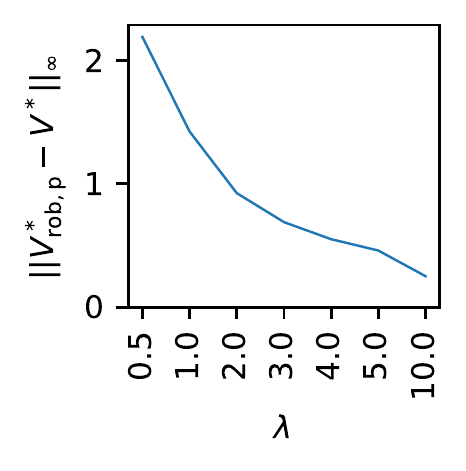}
    \caption{Deviation $\|V_{\robp}^*-V^*\|_{\infty}$ v.s. $\lambda$.}
    \label{fig: connection}
\end{figure}

In this section, we show how the robust value function $V^*_{\robp}$ varies with different $\lambda$. In a high-level idea, by definition of ~\eqref{eq: robpV}, we observe that the robust value function is less dependent with $P^*$ when $\lambda$ approaches $0$. Similarly, the robust value function would be approaching the non-robust value function since the infimum of ~\eqref{eq: robpV} would be obtained at $P=P^*$. We simply run Algorithm~\ref{alg: model_based} (taking $P=P^*$) and Value Iteration algorithms to obtain $V_{\robp}^*$ and $V^*$ respectively. In Figure~\ref{fig: connection}, we find the error $\|V_{\robp}^*-V^*\|_{\infty}$ decays as $\lambda$ increases, which suggests this idea is correct.

\subsection{Statistical Errors}
\label{sec: stat_error}
\begin{figure}[htbp!]
    \centering
    \includegraphics[scale=.7]{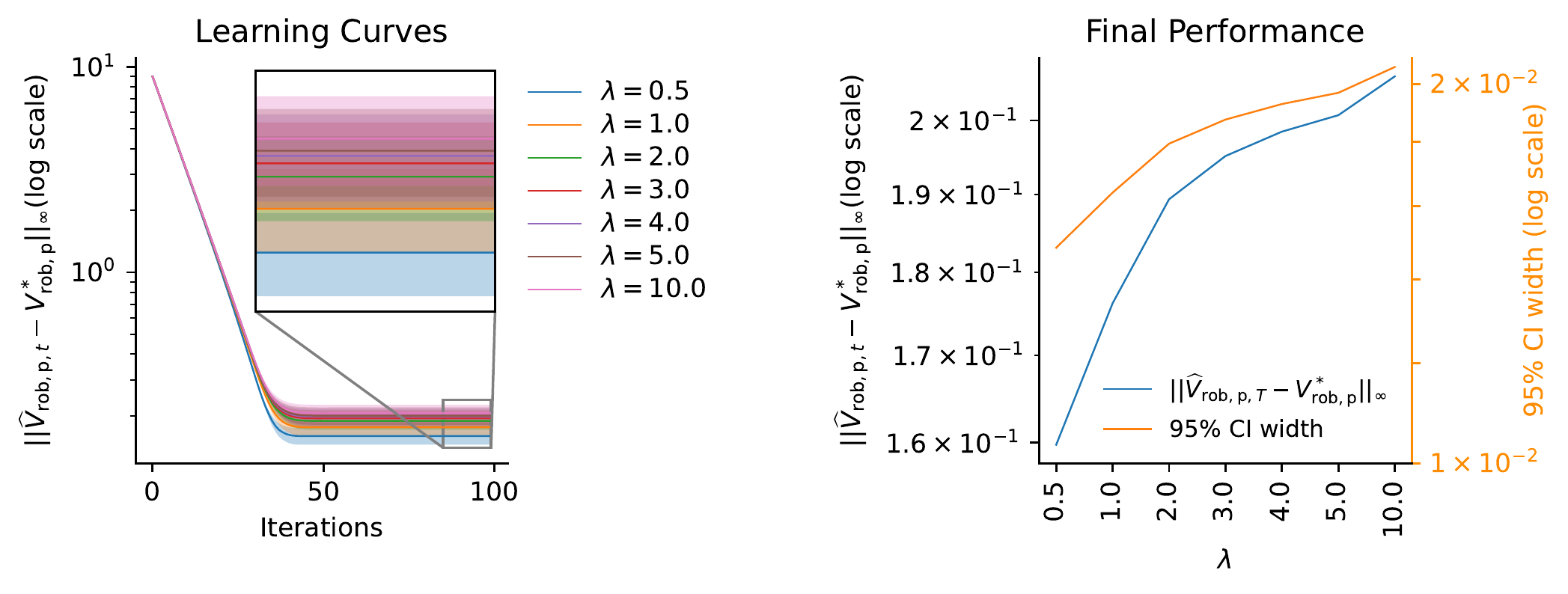}
    \caption{Left: deviation $\|\widehat{V}_{\robp,t}-V_{\robp}^*\|_{\infty}$ v.s. number of iterations. The shaded region represents the $95\%$ CI. The subplot on the top right corner zooms in the performance before cutting off learning. Right: deviation $\|\widehat{V}_{\robp,t}-V_{\robp}^*\|_{\infty}$ and $95\%$ confidence interval v.s. $\lambda$. }
    \label{fig: stat_error}
\end{figure}
In this section, we investigate the relationship between the statistical error and $\lambda$. In Theorem~\ref{thm: prmdp_stat}, we prove both upper and lower bounds for penalized robust MDPs. It is worth noticing the upper bound is conservative. And the example (Figure~\ref{fig: 10state}) we use is an extension of lower bound in Theorem~\ref{thm: prmdp_stat}. Thus, we compare our experiment results with lower bounds. In Figure~\ref{fig: stat_error}, the left learning curve composes of two stages: the curve drops at a linear rate in the first stage, and then becomes flat in the second stage. In fact, the first stage is due to $\widehat{\TM}_{\robp}$ is a $\gamma$-contraction, and the second stage is due to statistical error between $\widehat{P}$ and $P^*$. On the right side of Figure~\ref{fig: stat_error}, we find the deviation $\|\widehat{V}_{\robp,t}-V_{\robp}^*\|_{\infty}$ and confidence interval both increases as $\lambda$ increases, which matches the lower bound in Theorem~\ref{thm: prmdp_stat}. Also, on the left side of Figure~\ref{fig: stat_error}, it is notable that the convergence rate is also slightly related with the choice of $\lambda$. It is that the convergence rate would be fast at the first stage when $\lambda$ is small. This phenomenon is due to the robust Bellman gap of different values $V_1$ and $V_2$ becomes:
\begin{align}
    \left\|\TM_{\robp}V_1-\TM_{\robp}V_2\right\|_{\infty}\approx\gamma\left|V_{1,\text{min}}-V_{2,\text{min}}\right|
\end{align}
when $\lambda$ is small. On the contrary, if $\lambda$ is large, the error is determined by:
\begin{align}
    \left\|\TM_{\robp}V_1-\TM_{\robp}V_2\right\|_{\infty}\approx\gamma\left\|\EB_{P^*_{s,a}}(V_{1}-V_{2})\right\|_{\infty},
\end{align}
which is usually larger than the prior case.

\subsection{Convergence} \label{sec:convergence}
\begin{figure}[htbp!]
    \centering
    \includegraphics[width=0.9\textwidth]    
    {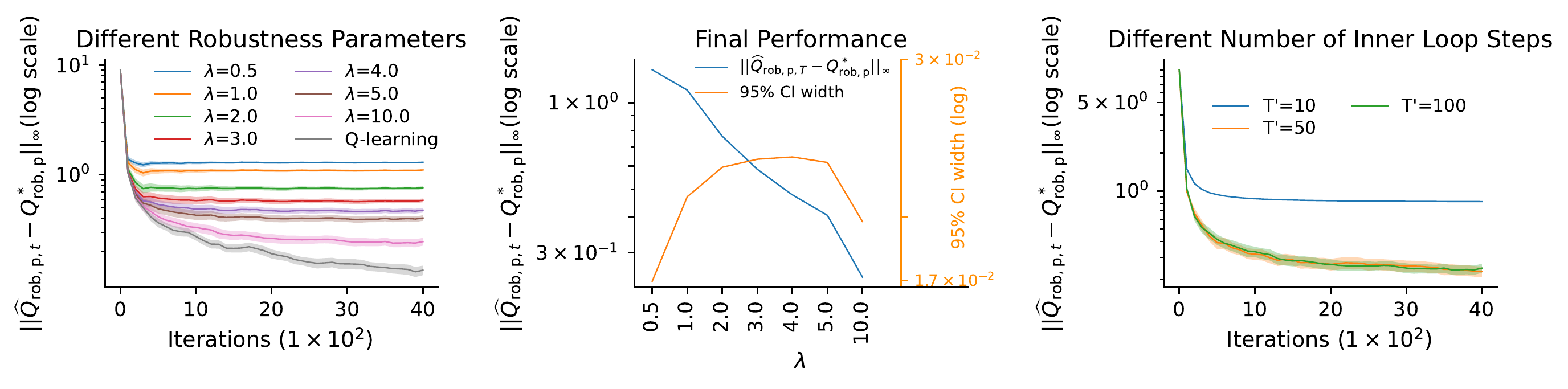}
    \caption{Deviation $\|\widehat{Q}_{\robp,t}-Q_{\robp}^*\|_{\infty}$ v.s. number of iterations. Left and Middle: $T'=100$; Right: $\lambda=10$. In the left and right subplots, the shaded region represents the $95\%$ CI. }
    \label{fig: qlearning}
\end{figure}
In this section, we test the convergence performance of Algorithm~\ref{alg: onlinermdp_modification}. In Figure~\ref{fig: qlearning}, we plot learning curves run by Algorithm~\ref{alg: onlinermdp_modification}. 
On the left side of Figure~\ref{fig: qlearning}, we find when $\lambda$ is smaller, the convergence rate is slightly faster.
This phenomenon coincides with the 1st stage performance in Fig~\ref{fig: stat_error} (left and middle). However, the final performance is strange: the error decreases as $\lambda$ increases. One reason is due to the optimization errors, where the error would amplify when $\lambda$ is small by the $1/\lambda$ factor in dual variable updating.
Except for the optimization errors, the other reason is that the bound of Theorem~\ref{thm: Delta_converge} is a worst case result, which is conservative when $\lambda$ is small or large. Thus we couldn't observe a matching performance with Theorem~\ref{thm: Delta_converge}. Moreover, we observe the confidence interval of the final run is increasing as $\lambda$ increases, which means robustness indeed works though there is a drop when $\lambda=10$. 
In the rightmost subplot of Fig~\ref{fig: qlearning}, we find the deviation $\|\widehat{Q}_{\robp, T}-Q^*_{\robp}\|_{\infty}$ also matters with the choice of $T'$. With a small $T'$, the solution of dual variable is not accurate, which leads to a bad performance on $\|\widehat{Q}_{\robp, T}-Q^*_{\robp}\|_{\infty}$. With $T'$ increases, the performance becomes better, which supports the third term in~\eqref{eq: delta}. 

\subsection{On a Markovian Chain Convergence} \label{sec:mkv_convergence}
\label{sec: exp_mark}

In this section, we test the convergence performance of Algorithm~\ref{alg: markoviandata_new}. To make Algorithm~\ref{alg: markoviandata_new} work, we make a slight change to the environment, where we allow the state $s_{10}$ can transit to $s_1$ with a positive probability. In this scenario, the stationary distribution satisfies $\min_{s,a}d_\pi(s,a)>0$. And we also set the behavior policies as $\pi(1|s)=\pi$ and $\pi(2|s)=1-\pi$. The minimal probability of the stationary distribution changes w.r.t. $\pi$, which is shown in Figure~\ref{fig: dmin}. It is notable that the stationary distribution is approximately uniform when $\pi\approx0.5$ and there exists a state-action pair becomes inaccessible when $\pi\approx 0$ or $\pi\approx 1$. To learn the robust Q-value function $\widehat{Q}_{\robp}$, we set $T=4\times10^6$ to make sure the overall sample complexity is the same as experiments in Section~\ref{sec:convergence}. Moreover, the learning rate for Q-value update and dual variable update are all set to be the same in Theorem~\ref{thm: markov_q}, where $\kappa=1/6(\lambda+(1-\gamma)^{-1})$. Besides, we sweep the same $\lambda\in\{0.5,1.0, 2.0, 3.0, 4.0, 5.0,10.0\}$ in the experiments. In the meantime, we also sweep behavior policy $\pi\in\{0.001,0.005,0.05,0.1,0.2,0.5\}$ for some chosen $\lambda$. In this setting, we run Algorithm~\ref{alg: markoviandata_new} repetitively with 100 different random seeds.

\begin{figure}[htbp!]
    \centering
    \includegraphics[scale=0.9]{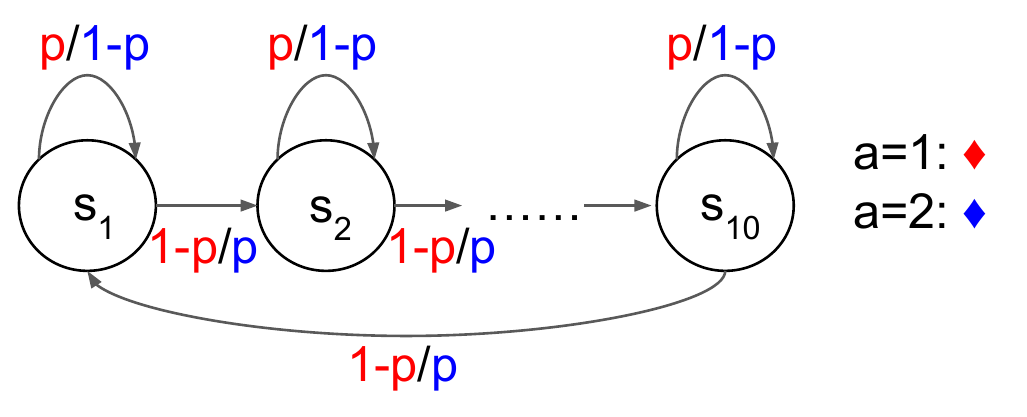}
    \caption{MDP with 10 states 2 actions. The transition probabilities are marked red taking action $a=1$, while the transition probabilities are marked blue taking action $a=2$.}
    \label{fig: 10state_mc}
\end{figure}
\begin{figure}[htbp!]
    \centering
    \includegraphics[scale=0.8]{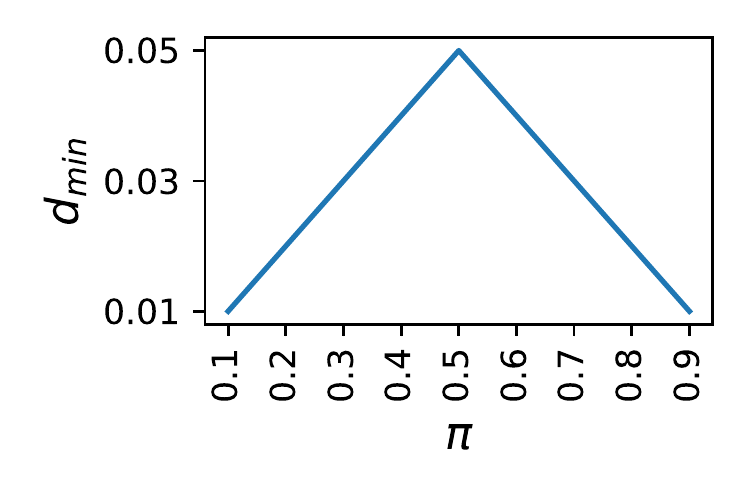}
    \caption{The minimal probability of the stationary distribution in Fig.~\ref{fig: 10state_mc} with changing of behavior policy $\pi$.}
    \label{fig: dmin}
\end{figure}

In Figure~\ref{fig: mc_lambd}, we show the performances with behavior policy $\pi=0.5$. 
It is notable that the training performances are undesirable when $\lambda=0.5$. We argue the main reason is due to the numerical problem in the learning rate for the smaller $\lambda$ as we explained in Section~\ref{sec:convergence}.
When $\lambda$ is large, we are delight to find the convergence can be guaranteed. Compared with Section~\ref{sec:convergence}, we also find the final performance of Algorithm~\ref{alg: markoviandata_new} is better in terms of sample complexity (error is better when number of samples are $10^6$). In Figure~\ref{fig: mc_behavior}, we also show the relationship between learning performances and behavior policy. It can be inferred that the performance would be better if the behavior policy approaches $0.5$ or $d_{\min}$ is large, which is due to each $(s,a)$ pair will be frequently and equally visited and also coincides with Theorem~\ref{thm: markov_q}. If the behavior policy approaches 0, the training performances will drop as some specific $(s,a)$ will be barely visited and the corresponding Q-values are inaccurate.

\begin{figure}[htbp!]
    \centering
    \includegraphics[scale=0.7]{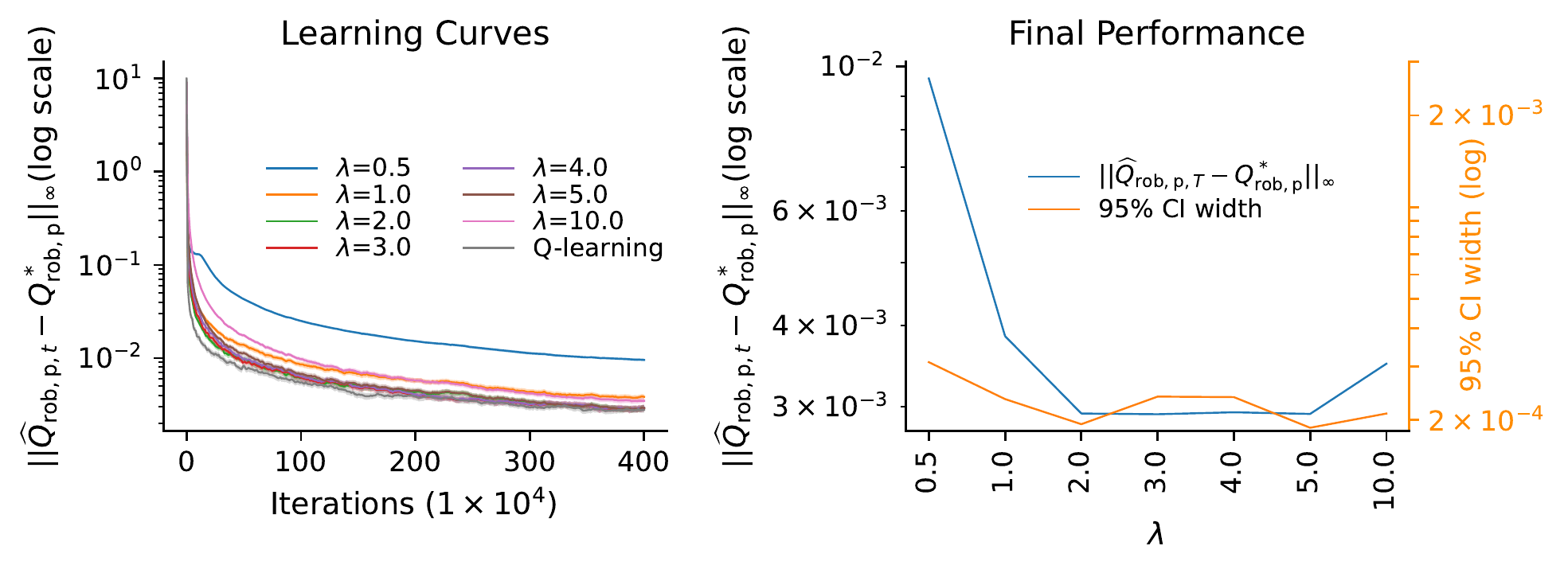}
    \caption{Deviation $\|\widehat{Q}_{\robp,t}-Q^*_{\robp}\|_{\infty}$ v.s. number of iterations (Algorithm~\ref{alg: markoviandata_new}, $\pi=0.5$). In the left subplot, the shaded region represents the $95\%$ CI. }
    \label{fig: mc_lambd}
\end{figure}
\begin{figure}[htbp!]
    \centering
    \begin{tabular}{ccc}
        \includegraphics[width=.35\columnwidth]{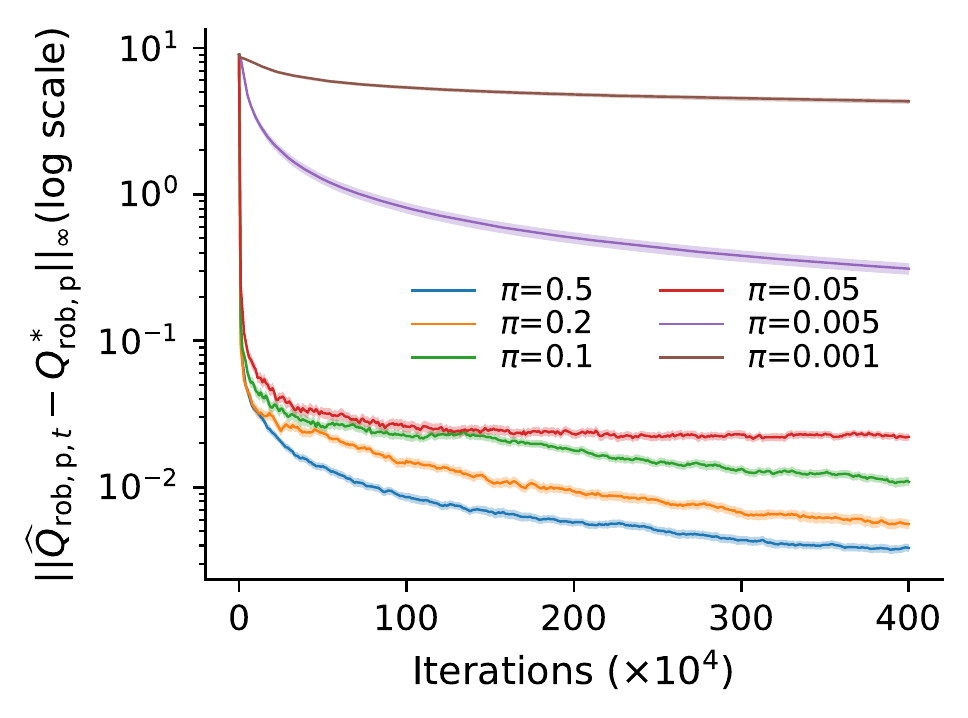}&
        \includegraphics[width=.35\columnwidth]{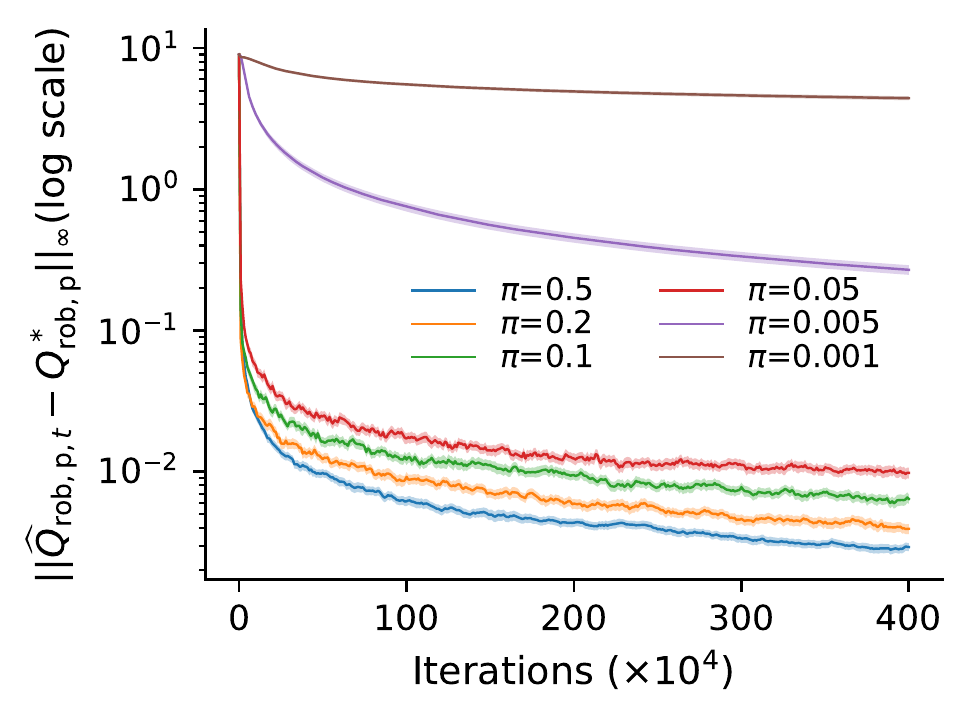}\\
        (a) $\lambda=1.0$ & (b) $\lambda=2.0$\\
        \includegraphics[width=.35\columnwidth]{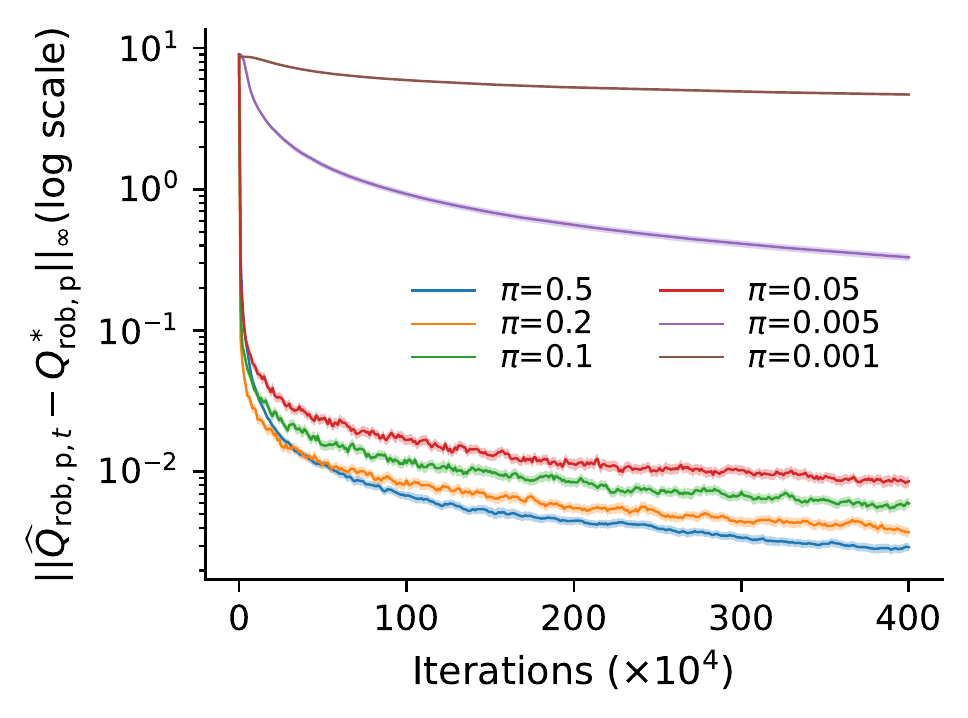}&
        \includegraphics[width=.35\columnwidth]{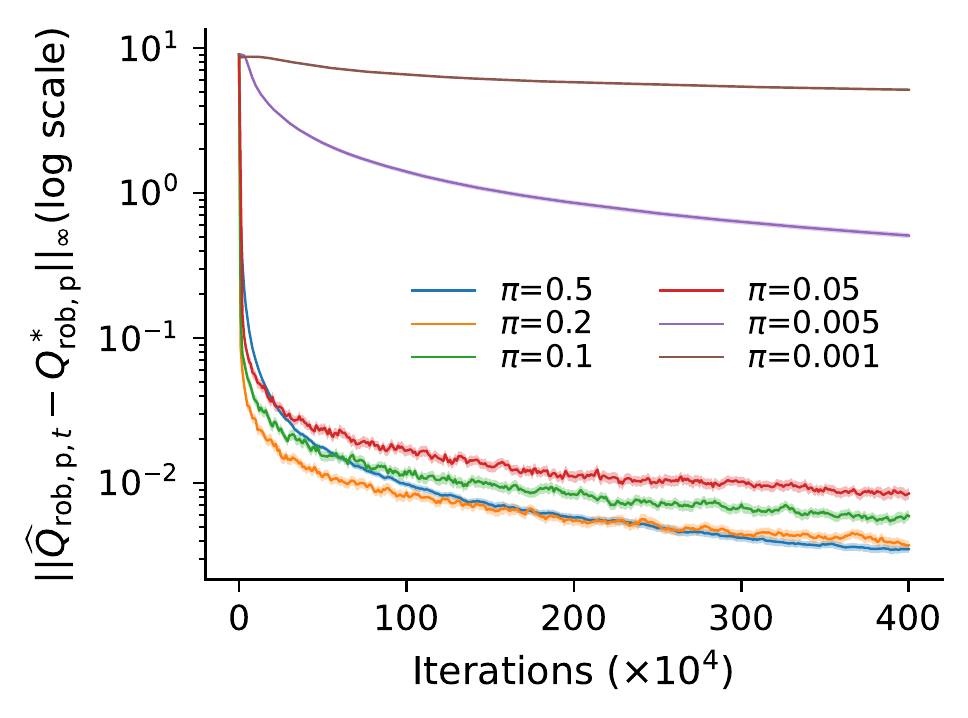}\\
        (c) $\lambda=5.0$ & (d)$\lambda=10.0$\\
    \end{tabular}
    \caption{Deviation $\|\widehat{Q}_{\robp,t}-Q^*_{\robp}\|_{\infty}$ v.s. number of iterations (Algorithm~\ref{alg: markoviandata_new}). }
    \label{fig: mc_behavior}
\end{figure}


\section{Concluding Remarks}
\label{sec: conc}
In this paper we have made two primary contributions towards solving robust MDPs efficiently. First, we have proposed an alternative formulation for distributionally robust MDPs and proved the statistical equivalence with the original forms. Second, we have devised a model-free algorithm to solve the robust MDPs without requiring an oracle to obtain solutions for DRO problems. We have also proved the polynomial convergence rate of our algorithm, in generative model setting and Markovian data setting. Here are some directions for further improvements. One direction is whether the convergence rate can be improved by some another technique such as Polyak-averaging \cite{polyak1992acceleration}. Furthermore, it would be challenging to move our theoretical results from a worst-case analysis to an instance-dependent analysis. Such instance-dependent results exist for MDPs \citep{khamaru2020temporal, yin2021towards, khamaru2021instance, li2020breaking,li2021polyak}, but it still remains open for robust MDPs. As the data generating mechanism is limited to generative model and Markovian in this paper, it is open whether the robust MDPs 
 could be solved efficiently if the behavior policy is changing with current Q-values. Moreover, from an empirical perspective, it also would be interesting to deploy our algorithms to some large-scale realistic applications.

\section*{Acknowledgements}
The authors thank Professor Martha White for valuable discussions with this project.

\bibliographystyle{plainnat}
\bibliography{bib/refer.bib}

\newpage
\appendix
\onecolumn
\begin{appendix}
	\onecolumn
	\begin{center}
		{\huge \textbf{Appendix}}
	\end{center}

\section{Proofs of Section~\ref{sec: prmdp}}
\label{apd: preli}

\begin{proof}[Proof of Proposition~\ref{prop: pen_def}]
    For any two $V_1, V_2\in\VM$, we have:
    \begin{align*}
        \TM_{\robp}V_1(s)-\TM_{\robp}V_2(s)\le\gamma\sup_{P\in\Delta(\SM)}\EB_{s'\sim P}|V_1(s')-V_2(s')|\le\gamma\|V_1-V_2\|_{\infty}.
    \end{align*}
    Thus, $\TM_{\robp}$ is a $\gamma$-contraction on $\VM$. Next, we prove the fixed point $V_{\robp}^*=\max_{\pi}V^{\pi}_{\robp}$. Firstly, for any fixed policy $\pi$, we define an operator:
    \begin{align*}
        \TM_{\robp}^\pi V(s):=\sum_{a\in\AM}\pi(a|s)R(s,a)+\gamma\sum_{a\in\AM}\pi(a|s)\inf_{P_a\in\Delta(\SM)}\left(\EB_{s'\sim P_a}V(s')+\lambda D_f(P_a\|P^*(\cdot|s,a))\right),
    \end{align*}
    where it is also a $\gamma$-contraction on $\VM$, and we denote the fixed point of $\TM_{\robp}^\pi$ by $V^*_{\pi}$ and $R^\pi(s):=\sum_{a\in\AM}\pi(a|s)R(s,a)$. By definition of $V_{\robp}^\pi$, we have:
    \begin{align*}
        V_{\robp}^\pi\ge\TM_{\robp}^{\pi} V_{\robp}^\pi\ge(\TM_{\robp}^{\pi})^2 V_{\robp}^\pi\ge\dots\ge(\TM_{\robp}^{\pi})^{\infty} V_{\robp}^\pi=V^*_{\pi}.
    \end{align*}
    On the contrary, denote $P^*_{\pi}$ as the solution to $\TM_{\robp}^\pi V^*_{\pi}=V^*_{\pi}$ and $P^*_{sa,\pi}:=P^*_{\pi}(\cdot|s,a)$, we have:
    \begin{align*}
        V^*_{\pi}(s)&=R^{\pi}(s)+\gamma\sum_{a\in\AM}\pi(a|s)\left(\EB_{s'\sim P_{sa,\pi}^*}V^*_{\pi}(s')+\lambda D_{f}(P_{sa,\pi}^*\|P^*(\cdot|s,a))\right)\\
        &=\EB_{\pi,P_{\pi}^*}\left[\sum_{t=0}^{\infty}\gamma^t(R(s_t,a_t)+\lambda\gamma D_f(P_{\pi}^*(\cdot|s_t,a_t)\|P^*(\cdot|s_t,a_t)))\Bigg{|}s_0=s\right]\\
        &\ge V_{\robp}^{\pi}(s).
    \end{align*}
    Thus, the fixed point $V_{\pi}^*=V^{\pi}_{\robp}$ for any fixed policy $\pi$. Similarly, by definition of $\TM_{\robp}$, for any policy $\pi$, we also have:
    \begin{align*}
        V^*_{\robp}\ge\TM^{\pi}_{\robp}V^*_{\robp}\ge(\TM^{\pi}_{\robp})^2 V^*_{\robp}\ge\dots\ge(\TM^{\pi}_{\robp})^{\infty}V^*_{\robp}=V^{\pi}_{\robp}.
    \end{align*}
    Taking maximum over $\pi$ on the RHS, we have $V_{\robp}^*\ge\max_{\pi}V_{\robp}^\pi$. Furthermore, denote $\pi^*$ as the solution to $\TM_{\robp}^* V^*_{\robp}=V^*_{\robp}$, note that $V^*_{\robp}$ is also the fixed point of the operator $\TM_{\robp}^{\pi^*}$, which means $V^*_{\robp}=V^{\pi^*}_{\robp}\le\max_{\pi} V^{\pi}_{\robp}$. Thus, we conclude that $V^*_{\robp}=\max_{\pi}V_{\robp}^\pi$.
\end{proof}

\begin{proof}[Proof of Theorem~\ref{thm: rholambda}]
    Without loss of generality, we assume $P^*(s'|s,a)>0$ for any $(s,a,s')\in\SM\times\AM\times\SM$. For notation simplicity, we denote $V_\lambda:=V_{\robp}^*$ and $Q_\lambda:=Q_{\robp}^*$, and $V_\rho:=V_{\robc}^*$ and $Q_\rho:=Q_{\robc}^*$. Firstly, for the penalized value function, by the fact $P\ll P^*$, we observe:
    \begin{align}
        \sup_{\pi, P\ll P^*}\EB_{P,\pi}\left[\left.\sum_{t=0}^{\infty}\gamma^t D_f(P(\cdot|s_t,a_t)\|P^*(\cdot|s_t,a_t))\right|s_0=s\right]<+\infty.
    \end{align}
    Thus, $V_\lambda(s)$ is continuous and non-decreasing w.r.t. $\lambda$ for any $s\in\SM$. Similarly, $Q_\lambda(s,a)$ is also continuous and non-decreasing w.r.t. $\lambda$ for any $(s,a)\in\SM\times\AM$. By the facts that $V_\lambda(s)\le V_{P^*}^*(s)$ and $Q_{\lambda}(s,a)\le Q_{P^*}(s,a)$, we know $\lim_{\lambda\rightarrow+\infty}V_\lambda:=V_{\infty}$ and $\lim_{\lambda\rightarrow+\infty}Q_\lambda:=Q_{\infty}$ exist. Next, we study the range of $V_\lambda(\cdot)$.
    \paragraph*{Case: $\lambda=0$.} In this case, we find $V_\lambda$ satisfies the following equation:
    \begin{align}
        V_\lambda(s) = \max_{a\in\AM}\left(R(s,a)+\gamma\inf_{s'}V_\lambda(s')\right).
    \end{align}
    \paragraph*{Case: $\lambda\rightarrow+\infty$.} By Bellman equation, for any $(s,a)\in\SM\times\AM$, we have:
    \begin{align}
        Q_\lambda(s,a)&=R(s,a)+\gamma\inf_{P\ll P^*_{s,a}}P^\top V_\lambda+\lambda D_f(P\|P^*_{s,a})\notag\\
        &=R(s,a)+\gamma P_{\lambda}^\top V_\lambda+\lambda D_f(P_{\lambda}\|P^*_{s,a}).
    \end{align}
    As $\lim_{\lambda\rightarrow+\infty}Q_{\lambda}$ exists, we have $\lim_{\lambda\rightarrow+\infty}\frac{Q_{\lambda}}{\lambda}=0$. By the fact $P_\lambda^\top V_\lambda$ is bounded, we have $\lim_{\lambda\to+\infty}D_f(P_\lambda\|P^*_{s,a})=0$. Then by Theorem 3.1 in \cite{csiszar1972class}, we have $\lim_{\lambda\to+\infty}\|P_\lambda-P^*_{s,a}\|_{1}=0$. Thus, we have:
    \begin{align}
        Q_{\infty}(s,a)&=R(s,a)+\gamma P^{*,\top}_{s,a}V_{\infty}+\lim_{\lambda\to+\infty}\lambda D_f(P_\lambda\|P^*_{s,a})\notag\\
        &:=R(s,a)+\gamma P^{*,\top}_{s,a}V_{\infty}+c(s,a),
    \end{align}
    where $c(s,a)\ge0$. Then, by definition of value function, we have $V_{P^*}^*\le V_{\infty}$. Thus, $V_{\infty}=V_{P^*}^*$. 
    
    Hence, for a given initial distribution $\mu$, $V_{\lambda}(\mu)$ is  continuous and non-decreasing w.r.t. $\lambda$ and $V_{\lambda}(\mu)\in[V_0(\mu), V_{P^*}^*(\mu))$. Then, we study the constrained value function. It is easy to obtain $V_{\rho}$ is non-increasing w.r.t. $\rho$. For any $\rho<\rho'$, we have:
    \begin{align}
        V_\rho(s)- V_{\rho'}(s)&\le\gamma\max_{a\in\AM}\left(\inf_{D_f(P\|P^*_{s,a})\le\rho}P^\top V_{\rho}-\inf_{D_f(P\|P^*_{s,a})\le\rho'}P^\top V_{\rho'}\right)\notag\\
        &\le\gamma\max_a\left(\inf_{D_f(P\|P^*_{s,a})\le\rho}P^\top V_{\rho}-\inf_{D_f(P\|P^*_{s,a})\le\rho'}P^\top V_{\rho}\right)+\gamma\|V_\rho-V_{\rho'}\|_{\infty}.
    \end{align}
    Thus, for any $\rho\not=\rho'$, we have:
    \begin{align}
        \label{eq: vrho}
        \|V_\rho-V_{\rho'}\|_{\infty}\le\frac{\gamma}{1-\gamma}\max_a\left|\inf_{D_f(P\|P^*_{s,a})\le\rho}P^\top V_{\rho}-\inf_{D_f(P\|P^*_{s,a})\le\rho'}P^\top V_{\rho}\right|.
    \end{align}
    From problem~\eqref{eq: drod}, we observe it is convex w.r.t. $\rho$. Thus, problem~\eqref{eq: drod} is continuous w.r.t. $\rho$. Combing with above inequality~\eqref{eq: vrho}, we obtain $V_\rho$ is continuous w.r.t. $\rho$. Next we study the range of $V_\rho$.
    \paragraph*{Case: $\rho=0$.} In this case, we find $V_\rho = V_{P^*}^*$.
    \paragraph*{Case: $\rho\to\infty$.} As $P\ll P^*$, we have:
    \begin{align}
        \sup_{(s,a)\in\SM\times\AM,P\ll P^*} D_f(P\|P_{s,a}^*)<+\infty,
    \end{align}
    and we denote $\rho^*(s,a):=\sup_{P\ll P^*_{s,a}}D_f(P\|P^*_{s,a})$. By Bellman equation, for any $s\in\SM$, we have:
    \begin{align}
        V_\rho(s)=\max_{a\in\AM}\left(R(s,a)+\gamma\inf_{D_f(P\|P^*_{s,a})\le\rho^*(s,a)}P^\top V_{\rho}\right)=\max_{a\in\AM}\left(R(s,a)+\gamma\inf_{s'\in\SM}V_\rho(s')\right),
    \end{align}
    which coincides with case $\lambda=0$ in penalized value function.

    Thus, for a given initial distribution $\mu$, $V_{\rho}(\mu)$ is non-increasing w.r.t. $\rho$ and $V_{\rho}(\mu)\in[V_0(\mu),\allowbreak V_{P^*}^*(\mu)]$. Finally, our result is obtained by intermediate value theorem.
\end{proof}

\begin{lem}
    \label{lem: chi2_range}
    Let $P$ be a probability measure on $(\SM,\FM)$, $V\in[0,\frac{1}{1-\gamma}]$, and $f(s)=(s-1)^2$, the optimal dual variable $\eta^*$ lies in $\Theta=[-\lambda,\frac{2}{1-\gamma}+2\lambda]$.
\end{lem}

\begin{proof}[Proof of Lemma~\ref{lem: chi2_range}]
    By definition of $f^*(\cdot)$, we have $f^*(s)=(s/2+1)^2_{+}-1$ when $f(s)=(s-1)^2$. Thus, the dual problem can be written by:
    \begin{align*}
        \mathcal{R}(P,V)&=\sup_{\eta\in\RB}-\lambda\EB_{s\sim P}\left(\frac{\eta-V(s)}{2\lambda}+1\right)_{+}^2+\lambda+\eta\\
        &=\sup_{\tilde{\eta}\in\RB}-\frac{1}{4\lambda}\EB_{s\sim P}\left(\tilde{\eta}-V(s)\right)_{+}^2-\lambda+\tilde{\eta}.
    \end{align*}
    The last equality holds by replacing $\eta$ with $\tilde{\eta}-2\lambda$. We denote $g(\tilde{\eta})=-\frac{1}{4\lambda}\EB_{s\sim P}\left(\tilde{\eta}-V(s)\right)_{+}^2-\lambda+\tilde{\eta}$, which is concave in $\tilde{\eta}$. For $\tilde{\eta}\le0$, we have $g(\tilde{\eta})=-\lambda+\tilde{\eta}$. For $\tilde{\eta}\ge\frac{2}{1-\gamma}+4\lambda$, we have:
    \begin{align*}
        g(\tilde{\eta})&\overset{(a)}{=}-\frac{1}{4\lambda}\EB_{s\sim P}\left(\tilde{\eta}-V(s)\right)^2-\lambda+\tilde{\eta}\\
        &=-\frac{\tilde{\eta}^2-2(\EB_{s\sim P}V(s)+2\lambda)\tilde{\eta}+\EB_{s\sim P} V(s)^2+4\lambda^2}{4\lambda}\\
        &\overset{(b)}{\le}-\frac{\tilde{\eta}^2-2(\frac{1}{1-\gamma}+2\lambda)\tilde{\eta}+4\lambda^2}{4\lambda}\\
        &\le-\lambda,
    \end{align*}
    where $(a)$ holds by $\tilde{\eta}\ge\max_s V(s)$ here and $(b)$ holds by $V\in[0,\frac{1}{1-\gamma}]$. By $g(\tilde{\eta})$ being concave, the optimal solution $\tilde{\eta}^*\in\argmax g(\tilde{\eta})$ lies in $[0,\frac{2}{1-\gamma}+4\lambda]$. Moreover, we notice $\sup_{\tilde{\eta}\in\RB}g(\tilde{\eta})\ge0$ by the primal objective. Thus, $-\lambda+\tilde{\eta}^*\ge0$. Thus, $\tilde{\eta}^*\in[\lambda,\frac{2}{1-\gamma}+4\lambda]$, which concludes $\eta^*\in[-\lambda,\frac{2}{1-\gamma}+2\lambda]$.
\end{proof}

\begin{proof}[Proof of Theorem~\ref{thm: prmdp_stat}]
    We prove upper bound at first. We note that:
    \begin{align*}
        \left\|\widehat{V}_{\robp}^*-V_{\robp}^*\right\|_{\infty}&=\left\|\widehat{\TM}_{\robp} \widehat{V}_{\robp}^*-\TM_{\robp}V_{\robp}^*\right\|_{\infty}\\
        &\le\left\|\widehat{\TM}_{\robp} \widehat{V}_{\robp}^*-\widehat{\TM}_{\robp} V_{\robp}^*\right\|_{\infty}+\left\|\widehat{\TM}_{\robp}V_{\robp}^*-\TM_{\robp} V_{\robp}^*\right\|_{\infty}\\
        &\le\gamma \left\|\widehat{V}_{\robp}^*-V_{\robp}^*\right\|_{\infty}+\left\|\widehat{\TM}_{\robp}V_{\robp}^*-\TM_{\robp} V_{\robp}^*\right\|_{\infty}.
    \end{align*}
    Arranging terms, we have:
    \begin{align*}
        &\left\|\widehat{V}_{\robp}^*-V_{\robp}^*\right\|_{\infty}\le\frac{1}{1-\gamma}\left\|\widehat{\TM}_{\robp}V_{\robp}^*-\TM_{\robp} V_{\robp}^*\right\|_{\infty}\\
        &\le\frac{\gamma}{1-\gamma}\max_{s,a}\left|\inf_{Q\in\Delta(\SM)}\EB_{s'\sim Q} V_{\robp}^*(s')+\lambda D_f(Q\|\widehat{P}_{s,a})-\inf_{Q\in\Delta(\SM)}\EB_{s'\sim Q} V_{\robp}^*(s')+\lambda D_f(Q\|P^*_{s,a})\right|,
    \end{align*}
    where the last inequality holds by definition of $\TM_{\robp}$ and $\widehat{\TM}_{\robp}$. By Eqn.~\eqref{eq: pdrod} and applying $f(s)=(s-1)^2$, we also have:
    \begin{align*}
        \inf_{Q\in\Delta(\SM)}\EB_{s'\sim Q} V_{\robp}^*(s')+\lambda D_f(Q\|P^*_{s,a})=\sup_{\eta\in\Theta}-\frac{1}{4\lambda}\EB_{s\sim P^*_{s,a}}(\eta-V_{\robp}^*(s))_{+}^2-\lambda+\eta.
    \end{align*}
    Moreover, we denote $g(\eta,P)=-\frac{1}{4\lambda}\EB_{s\sim P}(\eta-V_{\robp}^*(s))_{+}^2-\lambda+\eta$, where we omit $(s,a)$ dependence for simplification. Next, we study the deviation $|g(\eta,P)-g(\eta,\widehat{P})|$, where $\widehat{P}(s)=\frac{1}{n}\sum_{k=1}^n \1(X_k=s)$ and $X_{k}\sim P(\cdot)$ are i.i.d. random variables. Denote $Y_k := -\frac{1}{4\lambda}\sum_{s\in\SM}\1(X_k=s)(\eta-V_{\robp}^*(s))_{+}^2$, we have $|g(\eta,P)-g(\eta,\widehat{P})|=|\frac{1}{n}\sum_{k=1}^n Y_k-\EB Y_1|$. By Lemma~\ref{lem: chi2_range}, when $\eta\in\Theta$, we have $0\le-Y_k\le\frac{(\frac{2}{1-\gamma}+4\lambda)^2}{4\lambda}\le16\max\left\{\frac{1}{\lambda(1-\gamma)^2},\lambda\right\}$. By Hoeffding's inequality, we have:
    \begin{align*}
        \PB\left(|g(\eta,P)-g(\eta,\widehat{P})|\ge16\max\left\{\frac{1}{\lambda(1-\gamma)^2},\lambda\right\}\sqrt{\frac{\ln\frac{2}{\delta}}{2n}}\right)\le\delta.
    \end{align*}
    With $\eta\in\Theta$, we can prove that $\frac{1}{4\lambda}\sum_{s\in\SM}P(X_k=s)(\eta-V_{\robp}^*(s))_{+}^2$ is $8\max\left\{\frac{1}{\lambda^2(1-\gamma)^2},1\right\}$-Lipschitz w.r.t. $\eta$. Then we take the $\varepsilon$-net of $\Theta$ as $\NM_{\varepsilon}$ w.r.t. metric $|\cdot|$, whose size is bounded by:
    \begin{align*}
        |\NM_{\varepsilon}|\le1+\frac{\frac{2}{1-\gamma}+4\lambda}{\varepsilon}\le 1+\frac{8\max\{\frac{1}{1-\gamma},\lambda\}}{\varepsilon}.
    \end{align*}
    Then we have:
    \begin{align*}
        \sup_{\eta\in\Theta}|g(\eta,\widehat{P})-g(\eta,P)|\le16\max\left\{\frac{1}{\lambda^2(1-\gamma)^2},1\right\}\varepsilon+\sup_{\eta\in\NM_{\varepsilon}}|g(\eta,\widehat{P})-g(\eta,P)|.
    \end{align*}
    Taking $\varepsilon=\frac{\lambda}{\sqrt{2n}}$, we have:
    \begin{align*}
        &\PB\left(\sup_{\eta\in\Theta}|g(\eta,\widehat{P})-g(\eta,P)|\ge16\max\left\{\frac{1}{\lambda(1-\gamma)^2},\lambda\right\}\frac{\left(1+\sqrt{\ln\frac{2|\NM_{\varepsilon}|}{\delta}}\right)}{\sqrt{2n}}\right)\\
        \le&\PB\left(\sup_{\eta\in\NM_{\varepsilon}}|g(\eta,\widehat{P})-g(\eta,P)|\ge16\max\left\{\frac{1}{\lambda(1-\gamma)^2},\lambda\right\}\sqrt{\frac{\ln\frac{2|\NM_{\varepsilon}|}{\delta}}{2n}}\right)\le\delta.
    \end{align*}
    Finally, with probability $1-\delta$, we have:
    \begin{align*}
        \left\|\widehat{V}_{\robp}^*-V_{\robp}^*\right\|_{\infty}&\le\frac{16\gamma}{1-\gamma}\max\left\{\frac{1}{\lambda(1-\gamma)^2},\lambda\right\}\sqrt{\frac{\ln\frac{2|\SM||\AM||\NM_{\varepsilon}|}{\delta}}{2n}}\\
        &\le\frac{16\gamma}{1-\gamma}\max\left\{\frac{1}{\lambda(1-\gamma)^2},\lambda\right\}\sqrt{\frac{\ln\left(\frac{2|\SM||\AM|}{\delta}(1+8\max\{\frac{1}{\lambda(1-\gamma)},1\}\sqrt{2n})\right)}{2n}}\\
        &=\widetilde{\OM}\left(\frac{1}{(1-\gamma)\sqrt{2n}}\max\left\{\frac{1}{\lambda(1-\gamma)^2},\lambda\right\}\right).
    \end{align*}
    Next, we turn to calculate the lower bound. We consider a 2-state and 1-action MDP, where the states are denoted by $s_0$ and $s_1$. The reward is designed by $r(s_0)=1$ and $r(s_1)=0$. The transition probability is $P(s_0|s_0)=p$, $P(s_1|s_0)=1-p$, and $P(s_1|s_1)=1$. By robust Bellman equation, we have:
    \begin{align*}
        V(s_0)=1+\gamma\inf_{0\le q\le1}qV(s_0)+\lambda D_f(q\|p),
    \end{align*}
    where $D_f(q\|p)=p f(\frac{q}{p})+(1-p)f(\frac{1-q}{1-p})$. And $V(s_1)=0$. Setting $f(s)=(s-1)^2$, we have:
    \begin{align*}
        V(s_0)=1+\frac{\lambda\gamma}{p(1-p)}\inf_{0\le q\le 1}\left[q-\left(p-\frac{V(s_0)p(1-p)}{2\lambda}\right)\right]^2 + \frac{\gamma V(s_0)}{2}\left(2p-\frac{V(s_0)p(1-p)}{2\lambda}\right).
    \end{align*}
    \paragraph*{Case 1: $p-\frac{V(s_0)p(1-p)}{2\lambda}<0$.} In this case, the optimal $q^*=0$, and we have:
    \begin{align*}
        V(s_0)&=1+\frac{\lambda\gamma}{p(1-p)}\left(p-\frac{V(s_0)p(1-p)}{2\lambda}\right)^2 + \frac{\gamma V(s_0)}{2}\left(2p-\frac{V(s_0)p(1-p)}{2\lambda}\right)\\
        &=1+\frac{\lambda\gamma p}{1-p}.
    \end{align*}
    Denote $f(p)=1+\frac{\lambda\gamma p}{1-p}$, it is easy to verify that $f(p)$ is monotonically increasing and convex on $(0,1)$. Thus, we have:
    \begin{align*}
        f(p+\delta)-f(p)\ge f'(p)\delta=\frac{\lambda\gamma\delta}{(1-p)^2}.
    \end{align*}
    Thus, by choosing $\delta=\frac{2\varepsilon(1-p)^2}{\lambda\gamma}$ and Lemma 16 in \cite{azar2013minimax}, with a constant probability, to distinguish model $p$ and $p+\delta$, the number of samples we need at least is:
    \begin{align*}
        \Omega\left(\frac{\lambda^2p}{\varepsilon^2(1-p)^3}\right).
    \end{align*}
    Finally, by choosing $p=2-\frac{1}{\gamma}$ for $\gamma>3/4$, the lower bound for this 2-state MDP is:
    \begin{align*}
        \Omega\left(\frac{\lambda^2}{\varepsilon^2(1-\gamma)^3}\right),
    \end{align*}
    where $\lambda<\frac{1-\gamma}{\gamma(3-2\gamma)}$.

    \paragraph*{Case 2: $p-\frac{V(s_0)p(1-p)}{2\lambda}\ge0$.} In this case, the optimal $q^*=p-\frac{V(s_0)p(1-p)}{2\lambda}$, and we have:
    \begin{align*}
        V(s_0)=1+\frac{\gamma V(s_0)}{2}\left(2p-\frac{V(s_0)p(1-p)}{2\lambda}\right).
    \end{align*}
    By calculation, we have:
    \begin{align*}
        V(s_0)&=\frac{-2\lambda(1-\gamma p)+2\sqrt{\lambda^2(1-\gamma p)^2+\lambda\gamma p(1-p)}}{\gamma p(1-p)}\\
        &=\frac{2}{(1-\gamma p)+\sqrt{(1-\gamma p)^2+\frac{\gamma p(1-p)}{\lambda}}}\\
        &:=\frac{2}{g(p)}.
    \end{align*}
    Thus, to satisfy the condition $p-\frac{V(s_0)p(1-p)}{2\lambda}\ge0$, we need to restrict the range of $\lambda$ to $\lambda\ge\frac{1-p}{2-\gamma p}$. Then we wish to distinguish two value functions at $s_0$ under different transition probabilities $p+\delta$ and $\delta$. We denote them by $V_p$ and $V_{p+\delta}$ respectively and we have the following fact about $g(p)$:

    \paragraph*{Fact: $g(p)$ is concave and monotonically decreasing in $p\in(1/2,1)$.} The first order derivative of $g(p)$ is:
    \begin{align*}
        g'(p)=-\gamma+\frac{-2\gamma(1-\gamma p)+\frac{\gamma(1-2p)}{\lambda}}{2\sqrt{(1-\gamma p)^2+\frac{\gamma p(1-p)}{\lambda}}},
    \end{align*}
    where we find $g'(p)\le 0$ for $p\in(1/2,1)$ and conclude that $g(p)$ is monotonically decreasing in $p$. Furthermore, the second order derivative of $g(p)$ is:
    \begin{align*}
        g''(p)&=\frac{4\gamma(\gamma-\frac{1}{\lambda})\left((1-\gamma p)^2+\frac{\gamma p(1-p)}{\lambda}\right)-\left(-2\gamma(1-\gamma p)+\frac{\gamma(1-2p)}{\lambda}\right)^2}{4\left((1-\gamma p)^2+\frac{\gamma p(1-p)}{\lambda}\right)^{\frac{3}{2}}}\\
        &=\frac{\frac{4\gamma(\gamma-1)}{\lambda}-\frac{\gamma^2}{\lambda^2}}{4\left((1-\gamma p)^2+\frac{\gamma p(1-p)}{\lambda}\right)^{\frac{3}{2}}},
    \end{align*}
    from which we also find $g''(p)\le0$ and conclude that $g(p)$ is concave in $p$.

    Thus, the deviation $V_{p+\delta}-V_p$ satisfies:
    \begin{align*}
        V_{p+\delta}-V_p&=\frac{2}{g(p+\delta)}-\frac{2}{g(p)}\\
        &=\frac{2(g(p)-g(p+\delta))}{g(p+\delta)g(p)}\\
        &\overset{(a)}{\ge}\frac{-2 g'(p)\delta}{g(p+\delta)g(p)}\\
        &\overset{(b)}{\ge}\frac{-2 g'(p)\delta}{g(p)^2},
    \end{align*}
    where we apply the fact $g(p)$ is concave in $p$ to (a) and the fact $g(p)$ is monotonically decreasing in $p$ to (b). By choosing $\delta=\frac{\varepsilon g(p)^2}{-g'(p)}$ and Lemma 16 in \cite{azar2013minimax}, with a constant probability, to distinguish model $p$ and $p+\delta$, the number of samples we need at least is:
    \begin{align*}
        \widetilde{\Omega}\left(\frac{g'(p)^2 p(1-p)}{\varepsilon^2 g(p)^4}\right).
    \end{align*}
    Then by choosing $p=2-1/\gamma$ and $\gamma\in(3/4,1)$, we have:
    \begin{align*}
        g(2-\frac{1}{\gamma})&=2(1-\gamma)+\sqrt{4(1-\gamma)^2+\frac{(1-\gamma)(2\gamma-1)}{\lambda\gamma}}\\
        &\le 4(1-\gamma)+\sqrt{\frac{(1-\gamma)(2\gamma-1)}{\lambda\gamma}}\\
        &\le 2(1-\gamma)\max\left\{4,\sqrt{\frac{(2\gamma-1)}{\lambda\gamma(1-\gamma)}}\right\},\\
        \left|g'(2-\frac{1}{\gamma})\right|&=\gamma+\frac{4\gamma+\frac{3\gamma-2}{\lambda\gamma(1-\gamma)}}{2\sqrt{4+\frac{2\gamma-1}{\lambda\gamma(1-\gamma )}}}\\
        &\ge\gamma+\frac{3\gamma-2}{2(2\gamma-1)}\sqrt{4+\frac{2\gamma-1}{\lambda\gamma(1-\gamma)}}\\
        &\ge\frac{3}{4}+\frac{1}{4}\sqrt{4+\frac{2\gamma-1}{\lambda\gamma(1-\gamma)}}\\
        &\ge\max\left\{\frac{3}{4},\frac{1}{4}\sqrt{\frac{2\gamma-1}{\lambda\gamma(1-\gamma)}}\right\}\\
        &\ge\frac{3}{16}\max\left\{4,\sqrt{\frac{(2\gamma-1)}{\lambda\gamma(1-\gamma)}}\right\}.
    \end{align*}
    Thus, the number of samples we need at least is:
    \begin{align*}
        \widetilde{\Omega}\left(\frac{1}{\varepsilon^2(1-\gamma)^3}\min\left\{\frac{1}{16},\frac{\lambda\gamma(1-\gamma)}{2\gamma-1}\right\}\right),
    \end{align*}
    where $\lambda\ge\frac{1-\gamma}{\gamma(3-2\gamma)}$. Finally, for an MDP with $|\SM|$ states, $|\AM|$ actions, we can aggregate the 2-states-1-action MDPs together like Lemma 17 does in \cite{azar2013minimax}.

\end{proof}

\section{Proofs of Section~\ref{sec: bandit}}
\label{apd: bandit}

\begin{proof}[Proof of Lemma~\ref{lem: bound_grad}]
    By Assumption~\ref{asmp: diam} and first-order condition, we have:
    \begin{align*}
        \EB_P \nabla f^*\left(\frac{\eta^*-R(X)}{\lambda}\right)=1.
    \end{align*}
    By Assumption~\ref{asmp: smooth}, for any $\eta\in\Theta$ and $0\le R(X_i)\le M$, we have:
    \begin{align*}
        \left|\frac{\partial J(\eta;X_i)}{\partial \eta}\right|&=\left|\nabla f^*\left(\frac{\eta-R(X_i)}{\lambda}\right)-1\right|\\&=\left|\nabla f^*\left(\frac{\eta-R(X_i)}{\lambda}\right)-\EB_P\nabla f^*\left(\frac{\eta^*-R(X)}{\lambda}\right)\right|\\
        &\le\EB_P\left|\nabla f^*\left(\frac{\eta-R(X_i)}{\lambda}\right)-\nabla f^*\left(\frac{\eta^*-R(X)}{\lambda}\right)\right|\\
        &\le\frac{1}{\lambda \sigma}\left(|\eta-\eta^*|+\EB|R(X_i)-R(X)|\right)\\
        &\le\frac{\diam(\Theta)+M}{\lambda \sigma}.
    \end{align*}
\end{proof}

\begin{lem}
    \label{lem: error_extra}
    For any $\eta$, the following inequality holds:
    \begin{align*}
        \alpha_t\left(J(\eta)-J(\eta_t)\right)\le\frac{(\eta_t-\eta)^2}{2}-\frac{(\eta_{t+1}-\eta)^2}{2}+\frac{\alpha_t^2C_g^2}{2}+\alpha_t(\eta_t-\eta)\left(\frac{\partial J(\eta_t; X_t)}{\partial\eta}-\frac{\partial J(\eta_t)}{\partial\eta}\right).
    \end{align*}
\end{lem}

\begin{proof}
    By $\eta_{t+1}=\Pi_{\Theta}\left(\eta_t+\alpha_t\frac{\partial J(\eta_t; X_t)}{\partial\eta}\right)$, we have
    \begin{align*}
        \frac{1}{2}(\eta_{t+1}-\eta)^2-\frac{1}{2}(\eta_{t}-\eta)^2&\overset{(a)}{\le}\frac{1}{2}\left(\eta_t-\eta+\alpha_t\frac{\partial J(\eta_t; X_t)}{\partial\eta}\right)^2-\frac{1}{2}(\eta_{t}-\eta)^2\\
        &=\frac{\alpha_t^2}{2}\left|\frac{\partial J(\eta_t; X_t)}{\partial\eta}\right|^2+\alpha_t(\eta_t-\eta)\frac{\partial J(\eta_t; X_t)}{\partial\eta}\\
        &\overset{(b)}{\le}\frac{\alpha_t^2 C_g^2}{2}+\alpha_t(\eta_t-\eta)\left(\frac{\partial J(\eta_t; X_t)}{\partial\eta}-\frac{\partial J(\eta_t)}{\partial\eta}\right)+\alpha_t(\eta_t-\eta)\frac{\partial J(\eta_t)}{\partial \eta}\\
        &\overset{(c)}{\le}\frac{\alpha_t^2 C_g^2}{2}+\alpha_t(\eta_t-\eta)\left(\frac{\partial J(\eta_t; X_t)}{\partial\eta}-\frac{\partial J(\eta_t)}{\partial\eta}\right)+\alpha_t(J(\eta_t)-J(\eta)),
    \end{align*}
    where $(a)$ holds by the projection property, $(b)$ holds by Lemma~\ref{lem: bound_grad}, and $(c)$ holds by concavity of $J(\eta)$.
\end{proof}


\begin{lem}
    \label{lem: bound_final}
    If $\alpha_t$ is non-decreasing, then we have:
    \begin{align*}
        \alpha_T\left(J(\eta^*) - J(\eta_T)\right)\le\frac{1}{T}\sum_{t=1}^T\alpha_t\left(J(\eta^*)- J(\eta_t)\right)+\sum_{k=1}^{T-1}\frac{1}{k(k+1)}\sum_{t=T-k}^T\alpha_t\left(J(\eta_{T-k}) -  J(\eta_t)\right).
    \end{align*}
\end{lem}

\begin{proof}
    This proof technique was firstly derived in \cite{shamir2013stochastic}. Here we give a proof for a completeness consideration. We denote $S_k=\frac{1}{k+1}\sum_{t=T-k}^T\alpha_t\left(J(\eta^*) - J(\eta_t)\right)$, which satisfies:
    \begin{align*}
        (k+1)S_k &= \sum_{t=T-k}^T\alpha_t\left(J(\eta^*)- J(\eta_{T-k})\right) + \sum_{t=T-k}^T\alpha_t\left(J(\eta_{T-k}) - J(\eta_t)\right)\\
        &\overset{(a)}{\le}(k+1)\alpha_{T-k}\left(J(\eta^*)-J(\eta_{T-k})\right)+\sum_{t=T-k}^T\alpha_t\left(J(\eta_{T-k}) -  J(\eta_t)\right),
    \end{align*}
    where $(a)$ holds by $\alpha_t$ is non-decreasing. Then by definition of $S_k$, we have:
    \begin{align*}
        k S_{k-1} &= (k+1)S_k -\alpha_{T-k}\left(J(\eta^*)- J(\eta_{T-k})\right)\\
        &\le k S_k+\frac{1}{k+1}\sum_{t=T-k}^T\alpha_t\left(J(\eta_{T-k}) - J(\eta_t)\right).
    \end{align*}
    Summing $k$ from $k=1,...,T-1$, we have the final result:
    \begin{align*}
        S_0\le S_{T-1}+\sum_{k=1}^{T-1}\frac{1}{k(k+1)}\sum_{t=T-k}^T\alpha_t\left(J(\eta_{T-k}) -  J(\eta_t)\right).
    \end{align*}
\end{proof}

\begin{proof}[Proof of Theorem~\ref{thm: bandit_converge}]
    Firstly, we omit the $(s,a)$ dependence. From Lemma~\ref{lem: bound_final}, it is clear that we need upper bounds for the following two terms:
    \begin{align*}
        \Delta_1 &= \frac{1}{T}\sum_{t=1}^T\alpha_t(J(\eta^*)- J(\eta_t)),\\
        \Delta_2 &= \sum_{k=1}^{T-1}\frac{1}{k(k+1)}\sum_{t=T-k}^T\alpha_t\left(J(\eta_{T-k}) -  J(\eta_t)\right).
    \end{align*}
    For $\Delta_1$, applying Lemma~\ref{lem: error_extra}, we have:
    \begin{align}
        \Delta_1&\le\frac{1}{2T}(\eta_1-\eta^*)^2+\frac{C_g^2}{2T}\sum_{t=1}^T\alpha_t^2+\frac{1}{T}\sum_{t=1}^T\alpha_t(\eta_t-\eta^*)\left(\frac{\partial J(\eta_t; X_t)}{\partial\eta}-\frac{\partial J(\eta_t)}{\partial\eta}\right)\notag\\
        &\overset{(a)}{\le}\frac{\diam(\Theta)^2}{2T}+\frac{\diam(\Theta)^2(1+\ln T)}{2T}+\frac{1}{T}\sum_{t=1}^T\alpha_t(\eta_t-\eta^*)\left(\frac{\partial J(\eta_t; X_t)}{\partial\eta}-\frac{\partial J(\eta_t)}{\partial\eta}\right),
        \label{eq: delta_1}
    \end{align}
    where $(a)$ holds by $\alpha_t=\frac{\diam(\Theta)}{C_g\sqrt{ t}}$. 
    For $\Delta_2$, by Lemma~\ref{lem: error_extra}, we have:
    \begin{align*}
        \Delta_2&\le\sum_{k=1}^{T-1}\frac{1}{k(k+1)}\sum_{t=T-k}^T\frac{C_g^2\alpha_t^2}{2}+\alpha_t(\eta_t-\eta_{T-k})\left(\frac{\partial J(\eta_t; X_t)}{\partial\eta}-\frac{\partial J(\eta_t)}{\partial\eta}\right).
    \end{align*}
    Letting $\alpha_t=\frac{\diam(\Theta)}{C_g\sqrt{t}}$, we notice
    \begin{align*}
        \sum_{k=1}^{T-1}\frac{1}{k(k+1)}\sum_{t=T-k}^T\frac{C_g^2\alpha_t^2}{2}&=\frac{\diam(\Theta)^2}{2}\sum_{k=1}^{T-1}\frac{1}{k(k+1)}\sum_{t=T-k}^T\frac{1}{t}\\
        &=\frac{\diam(\Theta)^2}{2T}\sum_{k=1}^{T-1}\frac{1}{k(k+1)}+\frac{\diam(\Theta)^2}{2}\sum_{k=1}^{T-1}\frac{1}{k(k+1)}\sum_{t=T-k}^{T-1}\frac{1}{t}\\
        &\le\frac{\diam(\Theta)^2}{2T}+\frac{\diam(\Theta)^2}{2}\sum_{k=1}^{T-1}\sum_{t=T-k}^{T-1}\frac{1}{tk(k+1)}\\
        &\overset{(a)}{=}\frac{\diam(\Theta)^2}{2T}+\frac{\diam(\Theta)^2}{2}\sum_{k=1}^{T-1}\sum_{t=1}^{k}\frac{1}{(T-t)k(k+1)}\\
        &\overset{(b)}{=}\frac{\diam(\Theta)^2}{2T}+\frac{\diam(\Theta)^2}{2}\sum_{t=1}^{T-1}\sum_{k=t}^{T-1}\frac{1}{(T-t)k(k+1)}\\
        &=\frac{\diam(\Theta)^2}{2T}+\frac{\diam(\Theta)^2}{2}\sum_{t=1}^{T-1}\frac{1}{tT}\\
        &\le\frac{\diam(\Theta)^2(2+\ln T)}{2T},
    \end{align*}
    where $(a)$ holds by variable substitution of $t$, $(b)$ holds by interchanging the order of summation. Thus, we have:
    \begin{align*}
        \Delta_2\le\frac{\diam(\Theta)^2(2+\ln T)}{2T}+\sum_{k=1}^{T-1}\frac{1}{k(k+1)}\sum_{t=T-k}^T\alpha_t(\eta_t-\eta_{T-k})\left(\frac{\partial J(\eta_t; X_t)}{\partial\eta}-\frac{\partial J(\eta_t)}{\partial\eta}\right).
    \end{align*}
    Combining the upper bounds of $\Delta_1$ and $\Delta_2$, and denote $Z_t=\frac{\partial J(\eta_t; X_t)}{\partial\eta}-\frac{\partial J(\eta_t)}{\partial\eta}$, we have:
    \begin{align*}
        &\alpha_T(J(\eta^*)-J(\eta_T))\le\Delta_1+\Delta_2\notag\\
        &\le\frac{\diam(\Theta)^2(2+\ln T)}{T}+\frac{1}{T}\sum_{t=1}^T\alpha_t(\eta_t-\eta^*)Z_t+\sum_{k=1}^{T-1}\frac{1}{k(k+1)}\sum_{t=T-k}^T\alpha_t(\eta_t-\eta_{T-k})Z_t\notag\\
        &=\frac{\diam(\Theta)^2(2+\ln T)}{T}+\frac{1}{T}\sum_{t=1}^T\alpha_t(\eta_t-\eta^*)Z_t+\sum_{t=2}^{T}Z_t\left(\sum_{k=T-t+1}^{T-1}\frac{\alpha_t(\eta_t-\eta_{T-k})}{k(k+1)}\right).
    \end{align*}
    It is worth noticing that, for any $\lambda\in\RB$, $\EB[Z_t|\FM_{t-1}]=0$ and $\EB[\exp(\lambda Z_t)|\FM_{t-1}]\le\exp\left(\frac{\lambda^2C_g^2}{2}\right)$. As $\eta\in\Theta$, for any $\lambda\in\RB$, we have:
    \begin{align*}
        \EB\left[\exp\left(\frac{\lambda}{T}\sum_{t=1}^T\alpha_t(\eta_t-\eta^*)Z_t\right)\right]&\le\exp\left(\frac{2\lambda^2\diam(\Theta)^2C_g^2\sum_{t=1}^T\alpha_t^2}{T^2}\right)\notag\\
        &\le\exp\left(\frac{2\lambda^2\diam(\Theta)^4(1+\ln T)}{T^2}\right).
    \end{align*}
    Furthermore, we notice that:
    \begin{align*}
        \left(\sum_{k=T-t+1}^{T-1}\frac{\alpha_t(\eta_t-\eta_{T-k})}{k(k+1)}\right)^2&\le\alpha_t^2\diam(\Theta)^2\left(\sum_{k=T-t+1}^{T-1}\frac{1}{k(k+1)}\right)^2=\frac{\alpha_t^2\diam(\Theta)^2(t-1)}{T(T-t+1)}.
    \end{align*}
    Then, for any $\lambda\in\RB$, it implies:
    \begin{align*}
        \EB\left[\exp\left(\lambda\sum_{t=2}^{T}Z_t\left(\sum_{k=T-t+1}^{T-1}\frac{\alpha_t(\eta_t-\eta_{T-k})}{k(k+1)}\right)\right)\right]&\le\exp\left(2\lambda^2\sum_{t=2}^T\frac{\alpha_t^2\diam(\Theta)^2C_g^2(t-1)}{T(T-t+1)}\right)\notag\\
        &\le\exp\left(\frac{2\lambda^2\diam(\Theta)^4(1+\ln T)}{T}\right).
    \end{align*}
    Now we take $(s,a)$-dependence into consideration. By Lemma~\ref{lem: max_gaussian}, we have:
    \begin{align*}
        &\EB\left[\max_{(s,a)\in\SM\times\AM}\left(\sup_{\eta}J^{(s,a)}(\eta)-J^{(s,a)}(\eta_T(s,a))\right)\right]\notag\\
        &\le\frac{\diam(\Theta)C_g(2+\ln T)}{\sqrt{T}}+\frac{4\sqrt{2}\diam(\Theta)C_g\sqrt{(1+\ln T)\ln |\SM||\AM|}}{\sqrt{T}}\notag\\
        &\le\frac{\diam(\Theta)C_g(2+\ln T)(4\sqrt{2\ln|\SM||\AM|}+1)}{\sqrt{T}}.
    \end{align*}
\end{proof}

\section{Proofs of Section~\ref{sec: onlinermdp}}
\label{apd: onlinermdp}

\begin{lem}
    \label{lem: concen_error}
    For $N_t$ and $w\in\RB$, when $(1-\beta_t)\beta_{t-1}\le\beta_t$, we have: 
    \begin{align*}
        \EB\exp(w N_{t+1}(s,a))\le\exp\left(\frac{w^2\beta_t(1+\gamma C_M)^2}{2}\right).
    \end{align*}
\end{lem}

\begin{proof}[Proof of Lemma~\ref{lem: concen_error}]
    We prove the claim by induction on $t$. By Assumption~\ref{asmp: bound_f}, we have:
    \begin{align*}
        |\varepsilon_{t,1}(s,a)+\gamma I_{t,1}(s,a)|\le 1+\gamma C_M.
    \end{align*}
    By Hoeffding's Lemma, we have:
    \begin{align*}
        \EB[\exp(w(\varepsilon_{t,1}+\gamma I_{t,1}(s,a)))|\GM_{t-1}]\le\exp\left(\frac{w^2(1+\gamma C_M)^2}{2}\right).
    \end{align*}
    Therefor, the claim holds for $t=0$. Now we assume the claim holds for $t-1$:
    \begin{align*}
        \EB\exp(w N_{t}(s,a))\le\exp\left(\frac{w^2\beta_{t-1}(1+\gamma C_M)^2}{2}\right)\,.
        \label{eq: induction}
    \end{align*}
    Then, for $N_{t+1}(s,a)$, we have:
    \begin{align*}
        \EB\exp(w N_{t+1}(s,a))&=\EB\exp(w(1-\beta_t) N_{t}(s,a)+w\beta_t(\varepsilon_{r,t}(s,a)+\gamma I_{t,1}(s,a)))\\
        &=\EB\left(\exp(w(1-\beta_t)N_t(s,a))\cdot \EB[\exp(w\beta_t(\varepsilon_{r,t}(s,a)+\gamma I_{t,1}(s,a)))|\GM_{t-1}]\right)\\
        &\le\EB\exp(w(1-\beta_t)N_t(s,a))\cdot\exp\left(\frac{w^2\beta_t^2(1+\gamma C_M)^2}{2}\right)\\
        &\le\exp\left(\frac{w^2(1-\beta_t)^2\beta_{t-1}(1+\gamma C_M)^2}{2}\right)\cdot\exp\left(\frac{w^2\beta_t^2(1+\gamma C_M)^2}{2}\right),
    \end{align*}
    where the last inequality holds by assumption holding for $t-1$. Then we have:
    \begin{align*}
        \EB\exp(w N_{t+1}(s,a))\le\exp\left(\frac{w^2((1-\beta_t)^2\beta_{t-1}+\beta_t^2)(1+\gamma C_M)^2}{2}\right).
    \end{align*}
    By $(1-\beta_t)\beta_{t-1}\le\beta_t$, we finally have:
    \begin{align*}
        \EB\exp(w N_{t+1}(s,a))\le\exp\left(\frac{w^2\beta_t(1+\gamma C_M)^2}{2}\right).
    \end{align*}
\end{proof}

\begin{proof}[Proof of Lemma~\ref{lem: expect_N}]
    By Lemma~\ref{lem: concen_error}, we have:
    \begin{align*}
        \EB\exp(w\|N_t\|_{\infty})&\le\sum_{(s,a)\in\SM\times\AM}\EB\exp(w|N_t(s,a)|)\\
        &\le\sum_{(s,a)\in\SM\times\AM}\EB\exp(w N_t(s,a))+\EB\exp(-w N_t(s,a))\\
        &\le2|\SM||\AM|\exp\left(\frac{w^2\beta_{t-1}(1+\gamma C_M)^2}{2}\right).
    \end{align*}
    Thus, the tail bound of $\|N_t\|_{\infty}$ satisfies:
    \begin{align*}
        \PB\left(\|N_t\|_\infty\ge\tau\right)\le2|\SM||\AM|\exp\left(-\frac{\tau^2}{2\beta_{t-1}(1+\gamma C_M)^2}\right).
    \end{align*}
    By choosing $\tau_0=\sqrt{2\beta_{t-1}(1+\gamma C_M)^2\ln(2|\SM||\AM|)}$, the expectation of $\|N_t\|_{\infty}$ satisfies:
    \begin{align*}
        \EB\|N_t\|_{\infty}=\int_{0}^{+\infty}\PB(\|N_t\|_{\infty}\ge \tau)d\tau&=\int_{0}^{\tau_0}\PB(\|N_t\|_{\infty}\ge\tau)d\tau +\int_{\tau_0}^{+\infty}\PB(\|N_t\|_{\infty}\ge\tau)d\tau\\
        &\le\tau_0+\int_{\tau_0}^{+\infty}\PB(\|N_t\|_{\infty}\ge\tau)d\tau\\
        &\overset{(a)}{\le}\tau_0+\frac{2\beta_{t-1}(1+\gamma C_M)^2}{\tau_0}\\
        &\le2\sqrt{2\beta_{t-1}(1+\gamma C_M)^2\ln(2|\SM||\AM|)},
    \end{align*}
    where we use $\int_{c}^{+\infty}\exp(-t^2)dt\le\int_{c}^{+\infty}\exp(-ct)dt$ in (a).
\end{proof}

\begin{proof}[Proof of Lemma~\ref{lem: bound_delta}]
    By the fact $-(a_t+b_t+c_t)\1+N_t\le\Delta_t\le(a_t+b_t+c_t)\1+N_t$, we have:
    \begin{align*}
        \EB\|\Delta_T\|_{\infty}&\le a_T+\EB b_T+\EB c_T+\EB\|N_T\|_{\infty}\\
        &= \|\Delta_0\|_{\infty}\cdot\left(1-\sum_{t=0}^{T-1}\beta_{t,T-1}\right)+\frac{\gamma}{1-\gamma}\sum_{t=0}^{T-1}\beta_{t,T-1}\left(\EB\|N_t\|_{\infty}+\EB\|I_{t,2}\|_{\infty}\right)+\EB\|N_{T}\|_{\infty}\\
        &\le \|\Delta_0\|_{\infty}\cdot\left(1-\sum_{t=0}^{T-1}\beta_{t,T-1}\right)+\frac{\gamma}{1-\gamma}\sum_{t=0}^{T-1}\beta_{t,T-1}\left(\EB\|N_t\|_{\infty}+\varepsilon_{\text{opt}}\right)+\EB\|N_{T}\|_{\infty}.
    \end{align*}
    Combining with Lemma~\ref{lem: expect_N}, we have:
    \begin{align*}
        \EB\|\Delta_T\|_{\infty}\le\|\Delta_0\|_{\infty}\cdot\left(1-\sum_{t=0}^{T-1}\beta_{t,T-1}\right)+\frac{\gamma}{1-\gamma}\sum_{t=0}^{T-1}\beta_{t,T-1}\left(\sqrt{\beta_{t-1}}C_{N}+\varepsilon_{\text{opt}}\right)+\sqrt{\beta_T} C_N,
    \end{align*}
    where $C_N:=2\sqrt{2(1+\gamma C_M)^2\ln(2|\SM||\AM|)}$.
\end{proof}

\begin{lem}
    \label{lem: coeff_cal}
    For $\beta_t=\frac{1}{1+(1-\gamma)(t+1)}$, we have:
    \begin{align*}
        \prod_{t=0}^{T-1}(1-\beta_t(1-\gamma))=\frac{1}{1+(1-\gamma)T}.
    \end{align*}
\end{lem}
\begin{proof}[Proof of Lemma~\ref{lem: coeff_cal}]
    The result is obtained by calculation directly.
\end{proof}

\begin{proof}[Proof of Theorem~\ref{thm: Delta_converge}]
    By Lemma~\ref{lem: coeff_cal}, we have:
    \begin{align*}
        1-\sum_{t=0}^{T-1}\beta_{t,T-1}&=\frac{1}{1+(1-\gamma)T},\\
        \beta_{t,T-1}&=\frac{1-\gamma}{1+(1-\gamma)T}.
    \end{align*}
    Combining with Lemma~\ref{lem: bound_delta}, we have:
    \begin{align*}
        \EB\|\Delta_T\|_{\infty}&\le\frac{\|\Delta_0\|_{\infty}}{1+(1-\gamma)T}+\frac{\gamma C_N}{1+(1-\gamma)T}\sum_{t=0}^{T-1}\sqrt{\beta_{t-1}}+\frac{\gamma T\varepsilon_{\text{opt}}}{1+(1-\gamma)T}+\sqrt{\beta_T}C_N\\
        &\le\frac{\|\Delta_0\|_{\infty}}{1+(1-\gamma)T}+\frac{2C_N}{\sqrt{(1-\gamma)^3T}}+\frac{\varepsilon_{\text{opt}}}{1-\gamma}+\frac{C_N}{\sqrt{1+(1-\gamma)T}}.
    \end{align*}
    Combining with Theorem~\ref{thm: bandit_converge}, we finally have
    \begin{align*}\EB\|\Delta_T\|_{\infty}\le\frac{\|\Delta_0\|_{\infty}}{1+(1-\gamma)T}+\frac{2C_N}{\sqrt{(1-\gamma)^3T}}+\frac{\diam(\Theta)C_g(2+\ln T')}{(1-\gamma)\sqrt{T'}}+\frac{C_N}{\sqrt{1+(1-\gamma)T}}.
    \end{align*}
\end{proof}

\section{Proofs of Section~\ref{sec: markovian}}
\label{apd: markovian}
\begin{proof}[Proof of Lemma~\ref{asmp: eta_lip_v}]
    To ease the notations, we omit the $(s,a)$ dependence here. We notice $\eta_V^*$ satisfies the first order condition:
    \[
        \EB_{P^*}\nabla f^*\left(\frac{\eta_V^*-V(s')}{\lambda}\right)=1.
    \]
    Differential by $V$, we have:
    \begin{align*}
        \EB_{P^*}\nabla^2 f^*\left(\frac{\eta_V^*-V(s')}{\lambda}\right)\frac{\partial \eta_V^*}{\partial V}=\EB_{P^*}\nabla^2 f^*\left(\frac{\eta_V^*-V(s')}{\lambda}\right)\frac{\partial V(s')}{\partial V}.
    \end{align*}
    Taking $\|\cdot\|_{1}$ both sides, we have:
    \begin{align*}
        \left\|\frac{\partial \eta_V^*}{\partial V}\right\|_1\cdot\left|\EB_{P^*}\nabla^2 f^*\left(\frac{\eta_V^*-V(s')}{\lambda}\right)\right|=\EB_{P^*}\left|\nabla^2 f^*\left(\frac{\eta_V^*-V(s')}{\lambda}\right)\right|.
    \end{align*}
    Moreover, we notice $f^*(\cdot)$ is a convex function, thus $\nabla^2 f^*(\cdot)\ge0$. Thus, we have $ \left\|\frac{\partial \eta_V^*}{\partial V}\right\|_1=1$. Finally, for any $V_1,V_2\in[0,(1-\gamma)^{-1}]^{|\SM|}$, we have:
    \begin{align*}
        \left|\eta_{V_1}^*-\eta_{V_2}^*\right|=\left|\left\langle\frac{\partial \eta_{\tilde{V}}^*}{\partial V},V_1-V_2\right\rangle\right|\le\left\|V_1-V_2\right\|_{\infty},
    \end{align*}
    where $\tilde{V}$ lies in the convex combination of $V_1$ and $V_2$.
\end{proof}

\begin{proof}[Proof of Lemma~\ref{lem: markov_delta}]
    Following Algorithm~\ref{alg: markoviandata_new}, we have:
    \begin{align*}
        \delta_{t+1}^2(s,a)=&\left(\Pi_{\Theta}\left(\eta_{t}(s,a)+\alpha_t\1(\xi_t=(s,a))\frac{\partial J_{s_{t+1}}(\eta_t(s,a);V_t)}{\partial\eta}\right)-\eta_{t+1}^*(s,a)\right)^2\notag\\
        \overset{(a)}{\le}&\left(\eta_{t}(s,a)+\alpha_t\1(\xi_t=(s,a))\frac{\partial J_{s_{t+1}}(\eta_t(s,a);V_t)}{\partial\eta}-\eta_{t+1}^*(s,a)\right)^2\notag\\
        =&\delta_t^2(s,a)+2\alpha_t\delta_t(s,a)\1(\xi_t=(s,a))\frac{\partial J_{s_{t+1}}(\eta_t(s,a);V_t)}{\partial\eta}+2\delta_t(s,a)\left(\eta_{t}^*(s,a)-\eta_{t+1}^*(s,a)\right)\notag\\
        &+\left(\alpha_t\1(\xi_t=(s,a))\frac{\partial J_{s_{t+1}}(\eta_t(s,a);V_t)}{\partial\eta}+\eta_t^*(s,a)-\eta_{t+1}^*(s,a)\right)^2\notag\\
        \overset{(b)}{\le}&\delta_t^2(s,a)+2\alpha_t\delta_t(s,a)\1(\xi_t=(s,a))\frac{\partial J_{s_{t+1}}(\eta_t(s,a);V_t)}{\partial\eta}+2\delta_t(s,a)\left(\eta_{t}^*(s,a)-\eta_{t+1}^*(s,a)\right)\notag\\
        &+2\alpha_t^2\1(\xi_t=(s,a))\left(\frac{\partial J_{s_{t+1}}(\eta_t(s,a);V_t)}{\partial\eta}\right)^2+2\left(\eta_t^*(s,a)-\eta_{t+1}^*(s,a)\right)^2\notag\\
        \overset{(c)}{=}&\delta_t^2(s,a)+2\alpha_t\delta_t(s,a)\1(\xi_t=(s,a))\left(\frac{\partial J(\eta_t(s,a);V_t)}{\partial\eta}+g_t(s,a)\right)\notag\\
        &+2\delta_t(s,a)\left(\eta_{t}^*(s,a)-\eta_{t+1}^*(s,a)\right)+2\alpha_t^2\1(\xi_t=(s,a))\left(\frac{\partial J_{s_{t+1}}(\eta_t(s,a);V_t)}{\partial\eta}\right)^2\notag\\
        &+2\left(\eta_t^*(s,a)-\eta_{t+1}^*(s,a)\right)^2\notag\\
        \overset{(d)}{\le}&\delta_t^2(s,a)-2\kappa\alpha_t\delta_t^2(s,a)\1(\xi_t=(s,a))+2\alpha_t\delta_t(s,a)\1(\xi_t=(s,a))g_t(s,a)\notag\\
        &+H_1(s,a)+H_2(s,a)+H_3(s,a),
    \end{align*}
    where $(a)$ holds by projection property (Lemma~\ref{lem: proj}), $(b)$ holds by $(a+b)^2\le 2a^2+2b^2$, $(c)$ holds by setting $g_t(s,a):=\frac{\partial J_{s_{t+1}}(\eta_t(s,a);V_t)}{\partial\eta}-\EB\left[\left.\frac{\partial J_{s_{t+1}}(\eta_t(s,a);V_t)}{\partial\eta}\right|\FM_t\right]$, and $(d)$ holds by Assumption~\ref{asmp: monotone}.

    \paragraph{Term $H_1(s,a)$:}By $2ab\le \lambda a^2 + \lambda^{-1} b^2$ for any $\lambda>0$, we have:
    \begin{align*}
        H_1(s,a)&=2\delta_t(s,a)\left(\eta_{t}^*(s,a)-\eta_{t+1}^*(s,a)\right)\notag\\
        &\le\lambda\delta_t^2(s,a)+\frac{\left(\eta_{t}^*(s,a)-\eta_{t+1}^*(s,a)\right)^2}{\lambda}\notag\\
        &\overset{(a)}{=}\kappa\alpha_td_\pi(s,a)\delta_t^2(s,a)+\frac{\left(\eta_{t}^*(s,a)-\eta_{t+1}^*(s,a)\right)^2}{\kappa\alpha_td_\pi(s,a)}\notag\\
        &\overset{(b)}{\le}\kappa\alpha_td_\pi(s,a)\delta_t^2(s,a)+\frac{\left\|V_{t}-V_{t+1}\right\|_{\infty}^2}{\kappa\alpha_td_\pi(s,a)}\notag\\
        &\overset{(c)}{\le}\kappa\alpha_td_\pi(s,a)\delta_t^2(s,a)+\frac{4\beta_t^2C_M^2}{\kappa\alpha_td_\pi(s,a)},
    \end{align*}
    where $(a)$ holds by setting $\lambda=\kappa\alpha_t d_\pi(s,a)$, $(b)$ holds by Lemma~\ref{asmp: eta_lip_v}, and $(c)$ holds by updating rule in Algorithm~\ref{alg: markoviandata_new} and Assumption~\ref{asmp: bound_f}.

    \paragraph{Term $H_2(s,a)$:} By Lemma~\ref{lem: bound_grad}, we have:
    \begin{align*}
        H_2(s,a)\le2\alpha_t^2\1(\xi_t=(s,a))C_g^2.
    \end{align*}

    \paragraph{Term $H_3(s,a)$:} By Lemma~\ref{asmp: eta_lip_v}, we have:
    \[
        H_3(s,a) \le 2 \left\|V_t-V_{t+1}\right\|_{\infty}^2 \le 8\beta_t^2C_M^2.
    \]

    Combing all above together, we have:
    \begin{align*}
        \delta_{t+1}^2(s,a)\le&\left(1-\kappa\alpha_td_{\pi}(s,a)\right)\delta_{t}^2(s,a)-2\kappa\alpha_t\delta_t^2(s,a)\left(\1(\xi_t=(s,a))-d_{\pi}(s,a)\right)\notag\\
        &+2\alpha_t\delta_t(s,a)\1(\xi_t=(s,a))g_t(s,a)+2\alpha_t^2\1(\xi_t=(s,a))C_g^2\notag\\
        &+\frac{4\beta_t^2C_M^2}{\kappa\alpha_td_\pi(s,a)}+8\beta_t^2C_M^2.
    \end{align*}
    By induction, we have:
    \begin{align*}
        \delta_{t+1}^2(s,a)\le&\prod_{i=0}^t\left(1-\kappa\alpha_i d_{\pi}(s,a)\right)\cdot\delta_{0}^2(s,a)\notag\\
        &-2\kappa\sum_{i=0}^t\alpha_i\delta_i^2(s,a)(\1(\xi_i=(s,a))-d_\pi(s,a))\prod_{j=i+1}^t(1-\kappa\alpha_j d_\pi(s,a))\notag\\
        &+2\sum_{i=0}^t\alpha_i\delta_i(s,a)\1(\xi_i=(s,a))g_i(s,a)\prod_{j=i+1}^t(1-\kappa\alpha_j d_\pi(s,a))\notag\\
        &+2C_g^2\sum_{i=0}^t\alpha_i^2\1(\xi_i=(s,a))\prod_{j=i+1}^t(1-\kappa\alpha_j d_\pi(s,a))\notag\\
        &+\frac{4C_M^2}{\kappa d_\pi(s,a)}\sum_{i=0}^t\frac{\beta_i^2}{\alpha_i}\prod_{j=i+1}^t(1-\kappa\alpha_j d_\pi(s,a))\notag\\
        &+8C_M^2\sum_{i=0}^t\beta_i^2\prod_{j=i+1}^t(1-\kappa\alpha_j d_\pi(s,a))\\
        :=&I_1(s,a)+I_2(s,a)+I_3(s,a)+I_4(s,a)+I_5(s,a)+I_6(s,a).
    \end{align*}
    By Lemma~\ref{lem: alpha_ineq}, we know $(1-\kappa\alpha_{i+1}d_\pi(s,a))\alpha_i\le\alpha_{i+1}$ and we can bound $I_1(s,a)$, $I_4(s,a)$, $I_5(s,a)$ and $I_6(s,a)$ by:
    \begin{align*}
        &I_1(s,a)\le\frac{(1-\kappa\alpha_0 d_\pi(s,a))\alpha_t}{\alpha_0}\delta_0^2(s,a)\le\frac{\alpha_t\diam^2(\Theta)}{\alpha_0},\\
        &I_4(s,a)\le2C_g^2\alpha_t\sum_{i=0}^t\alpha_i\1(\xi_i=(s,a))\le2C_g^2\alpha_t\sum_{i=0}^t\alpha_i,\\
        &I_5(s,a)\le\frac{4C_M^2\alpha_t}{\kappa d_\pi(s,a)}\sum_{i=0}^t\frac{\beta_i^2}{\alpha_i^2},\\
        &I_6(s,a)\le8C_M^2\alpha_t\sum_{i=0}^t\frac{\beta_i^2}{\alpha_i}.
    \end{align*}
    For $I_2(s,a)$, by Lemma~\ref{lem: markov_ineq}, we let $f_t(s,a)=\delta_t(s,a)^2$ and find:
    \begin{align*}
        \left|\delta_{t+1}^2(s,a)-\delta_t^2(s,a)\right|&\le2\diam(\Theta)\left|\delta_{t+1}(s,a)-\delta_t(s,a)\right|\notag\\
        &\le2\diam(\Theta)\left(\left|\eta_{t+1}(s,a)-\eta_t(s,a)\right|+\left|\eta_{t+1}^*(s,a)-\eta_t^*(s,a)\right|\right)\notag\\
        &\le2\diam(\Theta)\left(\alpha_t C_g+2\beta_t  C_M\right).
    \end{align*}
    Thus, setting $\alpha_t = \frac{1}{\kappa d_{\min}(t+p_{\alpha})^{\alpha}}$, we have:
    \begin{align*}
        \EB\left\|I_2\right\|_{\infty}\le&\frac{\kappa\alpha_t \diam^2(\Theta) M}{1-\rho}\left(3+\ln(t+1+p_{\alpha})+\kappa\sum_{k=0}^t\alpha_k+\frac{2\sum_{k=0}^n(\alpha_t C_g+2\beta_t C_M)}{\diam(\Theta)}\right)\notag\\
        &+\frac{6\sqrt{2\kappa^{-1}d_{\min}^{-1}\alpha_t\diam^2(\Theta) M^2\ln 2|\SM||\AM|}}{1-\rho}.
    \end{align*}
    For $I_3(s,a)$, noting that $\delta_t(s,a)\1(\xi_t=(s,a))$ is measurable w.r.t. $\FM_t$, we can apply Lemma~\ref{lem: recursion_hoeffding} and obtain:
    \begin{align*}
        \EB\left\|I_3\right\|_{\infty}\le2\sqrt{32\kappa^{-1}d_{\min}^{-1}\alpha_{t}\diam^2(\Theta)C_g^2\ln(2|\SM||\AM|)}.
    \end{align*}
    Setting $\alpha=\frac{2}{3}$ in $\alpha_t$, $p^{\dagger}:=p_{\frac{2}{3}}=\left\lceil\left(\frac{d_{\max}}{d_{\min}}\right)^{\frac{3}{2}}\right\rceil$ and $\beta_t = \frac{1}{(1-\gamma)d_{\min}(t+p^\dagger)}$, we have:
    \begin{align*}
        \left\|I_1\right\|_{\infty}\le&\frac{2 \diam^2(\Theta)}{d_{\min}(t+1)^{\frac{2}{3}}},\\
        \EB\left\|I_2\right\|_{\infty}\le&\frac{3\diam(\Theta)M(\kappa\diam(\Theta)+2C_g+2d_{\min}\sqrt{\ln(2|\SM||\AM|)})}{\kappa d_{\min}^{2}(1-\rho)(t+1)^{\frac{1}{3}}}\notag\\
        &+\frac{\diam(\Theta)M(7d_{\min}\diam(\Theta)+4(1-\gamma)^{-1}L_V C_M)\ln(t+1+p^\dagger)}{ d_{\min}^{2}(1-\rho)(t+1)^{\frac{2}{3}}},\\
        \EB\left\|I_3\right\|_{\infty}\le&\frac{12\diam(\Theta)C_g\sqrt{\ln(2|\SM||\AM|)}}{\kappa d_{\min}(t+1)^{\frac{1}{3}}},\\
        \left\|I_4\right\|_{\infty}\le&\frac{6C_g^2}{\kappa^{2}d_{\min}^{2}(t+1)^{\frac{1}{3}}},\\
        \left\|I_5\right\|_{\infty}\le&\frac{12 C_M^2}{d_{\min}^{2}(1-\gamma)^{2}(t+1)^{\frac{1}{3}}},\\
        \left\|I_6\right\|_{\infty}\le&\frac{32  C_M^2}{d_{\min}^{2}(1-\gamma)^{2}(t+1)^{\frac{2}{3}}}.
    \end{align*}
    Combing all above, we have:
    \begin{align*}
        \EB\|\delta_{t+1}\|^2_{\infty}\le\frac{\Phi_1}{(t+1)^{\frac{1}{3}}}+\frac{\Phi_2\ln(t+1+p^\dagger)}{(t+1)^{\frac{2}{3}}}+\frac{\Phi_3}{(t+1)^{\frac{2}{3}}},
    \end{align*}
    where
    \begin{align*}
        \Phi_1=&\frac{3\diam(\Theta)M(\kappa\diam(\Theta)+2C_g+2d_{\min}\sqrt{\ln(2|\SM||\AM|)})}{\kappa d_{\min}^{2}(1-\rho)}\notag\\
        &+\frac{6C_g(2\kappa d_{\min}\diam(\Theta)\sqrt{\ln(2|\SM||\AM|)}+2C_g)}{\kappa^2d_{\min}^2}+\frac{12  C_M^2}{d_{\min}^2(1-\gamma)^2},\\
        \Phi_2=&\frac{\diam(\Theta)M(7d_{\min}\diam(\Theta)+4(1-\gamma)^{-1}L_V C_M)}{ d_{\min}^{2}(1-\rho)},\\
        \Phi_3=&\frac{2 \diam^2(\Theta)}{d_{\min}}+\frac{32  C_M^2}{d_{\min}^{2}(1-\gamma)^{2}}.
    \end{align*}
\end{proof}

\begin{proof}[Proof of Theorem~\ref{thm: markov_q}]
    We denote $\Delta_{t}(s,a):=Q_{t+1}(s,a)-Q_{\robp}^*(s,a)$. By Algorithm~\ref{alg: markoviandata_new}, we have:
    \begin{align*}
        \Delta_{t+1}(s,a)=&(1-\beta_t\1(\xi_t=(s,a)))\Delta_t(s,a)+\beta_t\1(\xi_t=(s,a))(\varepsilon_{r,t}(s,a)+\gamma\varepsilon_{J,t}(s,a))\notag\\
        &+\gamma\beta_t\1(\xi_t=(s,a))\left(J^{(s,a)}(\eta_t(s,a);V_t)-J^{(s,a)}(\eta_t^*;V_t)\right)\notag\\
        &+\gamma\beta_t\1(\xi_t=(s,a))\left(J^{(s,a)}(\eta_t^*;V_t)-J^{(s,a)}(\eta^*;V^*_{\robp})\right)\notag\\
        :=&(1-\beta_t\1(\xi_t=(s,a)))\Delta_t(s,a)+\beta_t Z_{t,1}(s,a)+\beta_t Z_{t,2}(s,a)\notag\\
        &+\gamma\beta_t\1(\xi_t=(s,a))\left(J^{(s,a)}(\eta_t^*;V_t)-J^{(s,a)}(\eta^*;V^*_{\robp})\right)
    \end{align*}
    where $\varepsilon_{J,t}(s,a)=J_{s_{t+1}}(\eta_t(s,a);V_t)-J^{(s,a)}(\eta_t(s,a);V_t)$ and $Z_t(s,a)$ covers the second and third terms. Noticing $\left|J^{(s,a)}(\eta_t^*;V_t)-J^{(s,a)}(\eta^*;V^*_{\robp})\right|\le\|\Delta_t\|_{\infty}$, we have:
    \begin{align*}
        \Delta_{t+1}(s,a)\le&(1-\beta_t\1(\xi_t=(s,a)))\Delta_t(s,a)+\beta_t (Z_{t,1}(s,a)+Z_{t,2}(s,a))+\gamma\beta_t\1(\xi_t=(s,a))\|\Delta_{t}\|_{\infty},\notag\\
        =&(1-\beta_t d_{\pi}(s,a))\Delta_t(s,a)+\beta_t (Z_{t,1}(s,a)+Z_{t,2}(s,a))\notag\\
        &+\beta_t(d_\pi(s,a)-\1(\xi_t=(s,a)))(\Delta_t(s,a)-\gamma\|\Delta_t\|_{\infty})+\gamma\beta_t d_\pi(s,a)\|\Delta_t\|_{\infty}
    \end{align*}
    By Lemma~\ref{lem: recursion}, we have:
    \begin{align*}
        \|\Delta_t\|_{\infty}\le\frac{\beta_t\|\Delta_0\|_{\infty}}{\beta_0}+\|B_t\|_{\infty}+\gamma\beta_t\sum_{k=0}^{t-1}\|B_k\|_{\infty},
    \end{align*}
    where $\{(B_t(s,a))_{(s,a)\in\SM\times\AM}\}_{t\ge0}$ satisfies $B_0 = \0$ and 
    \begin{align*}
        B_{t+1}(s,a)=&(1-\beta_ td_\pi(s,a))B_{t}(s,a)+\beta_t\left(Z_{t,1}(s,a)+Z_{t,2}(s,a)\right)\notag\\
        &+\beta_t(d_\pi(s,a)-\1(\xi_t=(s,a)))(\Delta_t(s,a)-\gamma\|\Delta_t\|_{\infty}),\notag\\
        =&\sum_{i=0}^t\beta_i Z_{i,1}(s,a)\prod_{j=i+1}^t(1-\beta_j d_\pi(s,a))+\sum_{i=0}^t\beta_i Z_{i,2}(s,a)\prod_{j=i+1}^t(1-\beta_j d_\pi(s,a))\notag\\
        &+\sum_{i=0}^t\beta_i (d_\pi(s,a)-\1(\xi_t=(s,a)))(\Delta_t(s,a)-\gamma\|\Delta_t\|_{\infty})\prod_{j=i+1}^t(1-\beta_j d_\pi(s,a))\notag\\
        :=&I_{t+1,1}(s,a)+I_{t+1,2}(s,a)+I_{t+1,3}(s,a).
    \end{align*}
    For $I_{t+1,1}(s,a)$, we notice that $Z_{t,1}(s,a)$ satisfies $\EB[Z_{t,1}(s,a)|\FM_t]=0$ and $|Z_{t,1}(s,a)|\le1+\gamma C_M$. By Lemma~\ref{lem: recursion_hoeffding}, we have:
    \begin{align*}
        \EB\|I_{t+1}\|_{\infty}\le3\sqrt{2d_{\min}^{-1}\beta_t (1+\gamma C_M)^2\ln 2|\SM||\AM|}.
    \end{align*}
    For $I_{t+1,2}(s,a)$, we notice $|Z_{t,2}(s,a)|\le\frac{\gamma\delta_t(s,a)^2}{2\sigma\lambda}$ by Assumption~\ref{asmp: smooth}. Thus, we have:
    \begin{align*}
        \EB\|I_{t+1,2}\|_{\infty}\le\frac{\gamma\beta_t}{2\sigma\lambda}\sum_{i=0}^t\|\delta_i\|_{\infty}^2.
    \end{align*}
    For $I_{t+1,3}(s,a)$, we notice:
    \begin{align*}
        \left\|\left(\Delta_{t+1}-\gamma\|\Delta_{t+1}\|_{\infty}\right)-\left(\Delta_{t}-\gamma\|\Delta_{t}\|_{\infty}\right)\right\|_{\infty}\le&(1+\gamma)\left\|\Delta_{t+1}-\Delta_t\right\|_{\infty}\notag\\
        =&(1+\gamma)\left\|Q_{t+1}-Q_t\right\|_{\infty}\notag\\
        \le&2(1+\gamma)\beta_t C_M.
    \end{align*}
    By Lemma~\ref{lem: markov_ineq}, we have:
    \begin{align*}
        \EB\left\|I_{t+1,3}\right\|_{\infty}\le&\frac{2(1+\gamma)\beta_t C_M M}{1-\rho}\left(3+\ln(t+1+p_{\beta})+2\sum_{k=0}^t\beta_k\right)\notag\\
        &+\frac{12(1+\gamma)\sqrt{2d_{\min}^{-1}\beta_t C_M^2M^2\ln 2|\SM||\AM|}}{1-\rho}.
    \end{align*}
    By setting $\beta_t = \frac{1}{(1-\gamma)d_{\min}(t+p_{\beta})}$, where $p_{\beta}=\lceil\frac{d_{\max}}{(1-\gamma)d_{\min}}\rceil$, we have:
    \begin{align*}
        \EB\|B_t\|_{\infty}\le&3\sqrt{\frac{2C_M^2\ln2|\SM||\AM|}{d_{\min}^2(1-\gamma)(t+1)}}+\frac{\gamma}{2\sigma\lambda d_{\min}(1-\gamma)(t+1)}\sum_{i=0}^t\EB\|\delta_i\|_{\infty}^2\notag\\
        &+\frac{18(1+\gamma)C_M M\ln(t+1+p_{\beta})}{(1-\rho)(1-\gamma)^2d^2_{\min}(t+1)}+12(1+\gamma)\sqrt{\frac{2C_M^2 M^2\ln2|\SM||\AM|}{(1-\gamma)(1-\rho)^2d_{\min}^2(t+1)}}.
    \end{align*}
    By Lemma~\ref{lem: markov_delta}, we have:
    \begin{align*}
        \sum_{i=0}^t\EB\|\delta_i\|_{\infty}^2&\le\diam^2(\Theta)+\sum_{i=1}^t\left(\frac{\Phi_1}{i^{\frac{1}{3}}}+\frac{\Phi_2\ln(i+p^\dagger)}{i^{\frac{2}{3}}}+\frac{\Phi_3}{i^{\frac{2}{3}}}\right)\notag\\
        &\le\diam^2(\Theta)+\frac{3\Phi_1 }{2}t^{\frac{2}{3}}+3\Phi_2t^{\frac{1}{3}}\ln(t+p^\dagger)+3\Phi_3 t^{\frac{1}{3}}\notag\\
        &\le\diam^2(\Theta)+\frac{3\Phi_1}{2}t^{\frac{2}{3}}+3(\Phi_2+2\Phi_3)t^{\frac{1}{3}}\ln(t+1+p^\dagger).
    \end{align*}
    Thus, we have
    \begin{align*}
        \EB\|B_t\|_{\infty}\le&3\sqrt{\frac{2C_M^2\ln2|\SM||\AM|}{d_{\min}^2(1-\gamma)(t+1)}}+\frac{\diam^2(\Theta)}{2\sigma\lambda d_{\min}(1-\gamma)(t+1)}+\frac{3\Phi_1}{4\sigma\lambda d_{\min}(1-\gamma)(t+1)^{\frac{1}{3}}}\notag\\
        &+\frac{3(\Phi_2+2\Phi_3)\ln(t+1+p^\dagger)}{2\sigma\lambda d_{\min}(1-\gamma)(t+1)^{\frac{2}{3}}}\notag\\
        &+\frac{36C_M M\ln(t+1+p_{\beta})}{(1-\rho)(1-\gamma)^2d^2_{\min}(t+1)}+24\sqrt{\frac{2C_M^2 M^2\ln2|\SM||\AM|}{(1-\gamma)(1-\rho)^2d_{\min}^2(t+1)}}.
    \end{align*}
    The dominating term in $\EB\|B_t\|_{\infty}$ is of order $(t+1)^{\frac{1}{3}}$, thus, the dominating term in $\EB\|\Delta_t\|_{\infty}$ satisfies:
    \begin{align*}
        \EB\|\Delta_t\|_{\infty}\le \widetilde{\OM}\left(\frac{\Phi_1}{\sigma\lambda d_{\min}(1-\gamma)^2(t+1)^{\frac{1}{3}}}\right).
    \end{align*}
\end{proof}

\section{Auxiliary Lemma}
\label{apd: aux}
\begin{lem}[Exercise 2.12 in \citet{wainwright2019high}]
    \label{lem: max_gaussian}
    Let $\{X_i\}_{i=1}^n$ be a sequence of mean-zero random variables, which satisfy (for any $\lambda\in\RB$):
    \begin{align*}
        \EB[\exp(\lambda X_i)]\le\exp\left(\frac{\lambda^2\sigma^2}{2}\right).
    \end{align*}
    Then, the following inequality holds:
    \begin{align*}
        \EB\left[\max_{i=1,\cdots,n}X_i\right]\le\sqrt{2\sigma^2\ln n}.
    \end{align*}
\end{lem}
\begin{proof}
    Indeed, for any $\lambda>0$ we have:
    \begin{align*}
        \exp\left(\lambda\EB\left[\max_{i=1,\cdots,n}X_i\right]\right)&\overset{(a)}{\le}\EB\left[\exp\left(\lambda\max_{i=1,\cdots,n}X_i\right)\right]\notag\\
        &\le\EB\left[\sum_{i=1}^n\exp(\lambda X_i)\right]\notag\\
        &\le n\exp\left(\frac{\lambda^2\sigma^2}{2}\right).
    \end{align*}
    Thus, we have:
    \begin{align*}
        \EB\left[\max_{i=1,\cdots,n}X_i\right]\le\inf_{\lambda>0}\left(\frac{\ln n}{\lambda}+\frac{\lambda\sigma^2}{2}\right)=\sqrt{2\sigma^2\ln n}.
    \end{align*}
\end{proof}

\begin{lem}[Lemma 3.1 in \citet{bubeck2015convex}]
    \label{lem: proj}
    Let $\|\cdot\|$ denote the Euclidean norm on set $\XM\subset\RB^d$. For any $y\in\RB^d$, we define the projection operator $\Pi_{\XM}$ on $\XM$ by:
    \begin{align*}
        \Pi_{\XM}(y)=\argmin_{x\in\XM}\|x-y\|.
    \end{align*}
    Then, for any $x\in\XM$ and $y\in\RB^d$, we have:
    \begin{align*}
        \left(\Pi_{\XM}(y)-x\right)^\top\left(\Pi_{\XM}(y)-y\right)\le0,
    \end{align*}
    which also implies $\|\Pi_{\XM}(y)-x\|^2+\|\Pi_{\XM}(y)-y\|^2\le\|y-x\|^2$.
\end{lem}

\begin{lem}
    \label{lem: alpha_ineq}
    Denote $\alpha_t=\frac{1}{(1+t)^\alpha}$, where $\alpha\in(0,1]$ and $t\in \ZB$. Then, for any $t\in\ZB_+$, we have:
    \begin{align*}
        (1-\alpha_{t})\alpha_{t-1}\le\alpha_{t}.
    \end{align*}
\end{lem}
\begin{proof}
    For $\alpha=1$, the result holds trivially. For $0<\alpha<1$, we denote $f(t)=(1+t)^{\alpha}-1-t^{\alpha}$, where $t\in[0,+\infty)$. The derivative of $f(x)$ satisfies:
    \begin{align*}
        f'(t)=\alpha\left(\frac{1}{(1+t)^{1-\alpha}}-\frac{1}{t^{1-\alpha}}\right)\le 0.
    \end{align*}
    Thus, $f(t)\le f(0)=0$ and our result is obtained.
\end{proof}

\begin{lem}
    \label{lem: sum_t}
    Denote $\alpha_t=\frac{1}{(1+t)^\alpha}$, where $\alpha\in(0,1]$ and $t\in \ZB$. Then, for any $i\le j$ we have:
    \begin{align*}
        \sum_{t=i}^{j} \alpha_t\le
        \begin{cases}
            \frac{(1+j)^{1-\alpha}-i^{1-\alpha}}{1-\alpha}& \hspace{4pt}\text{if $\alpha\in(0,1)$,}\\
            \1(i=0)+\ln(j+1)-\ln(\1(i=0)+i)&\hspace{4pt}\text{if $\alpha=1$.}
        \end{cases}
    \end{align*}
\end{lem}
\begin{proof}
    For $\alpha\in(0,1)$, we have:
    \begin{align*}
        \sum_{t=i}^j\alpha_t\le\int_{i-1}^j\frac{1}{(1+x)^{\alpha}}dx=\frac{(1+j)^{1-\alpha}-i^{1-\alpha}}{1-\alpha}.
    \end{align*}
    For $\alpha=1$ and $i\ge 1$, we have:
    \begin{align*}
        \sum_{t=i}^j\alpha_t\le\int_{i-1}^j\frac{1}{1+x}dx=\ln\left(1+j\right)-\ln i.
    \end{align*}
    For $\alpha=1$ and $i=0$, we have:
    \begin{align*}
        \sum_{t=0}^j\alpha_t\le1+\int_{0}^j\frac{1}{1+x}dx=1+\ln\left(1+j\right).
    \end{align*}
\end{proof}

\begin{lem}
    For any $\alpha\in(0,1)$ and $t>0$, we have:
    \begin{align}
        (t+1)^\alpha-t^\alpha\le\frac{1}{t^{1-\alpha}}.\notag
    \end{align}
\end{lem}
\begin{proof}
    We denote $f(x)=\left(1+x\right)^{\alpha}-x-1$, where $x\ge 0$. Its derivative satisfies:
    \begin{align}
        f'(x)=\frac{\alpha}{(1+x)^{1-\alpha}}-1<\alpha-1<0.\notag
    \end{align}
    Thus, $f(x)\le f(0)=0$. Taking $x=\frac{1}{t}$, we have:
    \begin{align}
        \left(1+\frac{1}{t}\right)^{\alpha}-\frac{1}{t}-1\le0.\notag
    \end{align}
    Arranging terms, the final result is obtained.
\end{proof}

\begin{lem}
    \label{lem: recursion}
    Suppose $\{(X_t(i))_{i\in[d]}\}_{t\ge0}$ and $\{(Y_t(i))_{i\in[d]}\}_{t\ge0}$ are two sequences that satisfy the following inequalities:
    \begin{align*}
        Y_{t+1}(i)&\le(1-\alpha_t c_i)Y_t(i)+\alpha_t X_t(i)+\gamma \alpha_t d(i)\|Y_t\|_{\infty},\\
        Y_{t+1}(i)&\ge(1-\alpha_t c_i)Y_t(i)+\alpha_t X_t(i)-\gamma \alpha_t d(i)\|Y_t\|_{\infty},
    \end{align*}
    where $\gamma\in[0,1)$, $c_i>0$, $(1-(1-\gamma)\alpha_t c_i)\alpha_{t-1}\le\alpha_t$, and $\alpha_0\le c_i^{-1}$ for all $t\ge1$ and $i\in[d]$. Then, we have:
    \begin{align*}
        \|Y_t\|_{\infty}\le\frac{\alpha_t\|Y_0\|_{\infty}}{\alpha_0}+\|B_t\|_{\infty}+\gamma\alpha_t\sum_{k=0}^{t-1}\|B_k\|_{\infty},
    \end{align*}
    where $\{(B_t(i))_{i\in[d]}\}_{t\ge0}$ satisfies $B_{t+1}(i)=(1-\alpha_t c_i)B_t(i)+\alpha_t X_t(i)$ and $B_0=\0$.
\end{lem}
\begin{proof}
    We construct auxiliary sequences:
    \begin{align*}
        A_{t+1}&=(1-(1-\gamma)\alpha_t c_{\min})A_t,\\
        B_{t+1}(i)&=(1-\alpha_t c_i)B_t(i)+\alpha_t X_t(i),\\
        C_{t+1}(i)&=(1-(1-\gamma)\alpha_t c_i)\|C_t\|_{\infty}+\gamma\alpha_t c_i\|B_t\|_{\infty},
    \end{align*}
    where $A_0=\|Y_0\|_{\infty}$, $B_0 = \0$ and $C_0=\0$. By induction, for $t$, we assume:
    \begin{align*}
        -(A_t+\|C_t\|_{\infty})+B_t(i)\le Y_t(i)\le (A_t+\|C_t\|_{\infty})+B_t(i).
    \end{align*}
    Then, for $t+1$, we have:
    \begin{align*}
        Y_{t+1}&\le(1-\alpha_t c_i)\left((A_t+\|C_t\|_{\infty})+B_t(i)\right)+\alpha_t X_t(i)+\gamma\alpha_t c_i\left((A_t+\|C_t\|_{\infty})+\|B_t\|_{\infty}\right)\notag\\
        &\le (A_{t+1}+\|C_{t+1}\|_{\infty})+B_{t+1}(i).
    \end{align*}
    Reversely, we have:
    \begin{align*}
        Y_{t+1}&\ge(1-\alpha_t c_i)\left(-(A_t+\|C_t\|_{\infty})+B_t(i)\right)+\alpha_t X_t(i)-\gamma\alpha_t c_i\left((A_t+\|C_t\|_{\infty})+\|B_t\|_{\infty}\right)\notag\\
        &\ge-(A_{t+1}+\|C_{t+1}\|_{\infty})+B_{t+1}(i).
    \end{align*}
    Thus, the following inequality holds for $t\ge0$:
    \begin{align*}
        -(A_t+\|C_t\|_{\infty})+B_t(i)\le Y_t(i)\le (A_t+\|C_t\|_{\infty})+B_t(i).
    \end{align*}
    By definition of $A_t$, $B_t$, $C_t$ and $(1-(1-\gamma)\alpha_t c_i)\alpha_{t-1}\le\alpha_t$, the following upper bouds lead to the final result:
    \begin{align*}
        A_t &= \prod_{k=0}^{t-1}\left(1-(1-\gamma)\alpha_k c_{\min}\right)\cdot \|Y_0\|_{\infty}\le\frac{\alpha_t\|Y_0\|_{\infty}}{\alpha_0},\\
        B_t(i)&=\sum_{k=0}^{t-1}\alpha_k X_k(i)\prod_{j=k+1}^{t-1}(1-(1-\gamma)\alpha_j c_i),\\
        \|C_t\|_{\infty}&\le\gamma\sum_{k=0}^{t-1}\alpha_{k}\|B_k\|_{\infty}\prod_{j=k+1}^{t-1}(1-(1-\gamma)\alpha_jc_{\min})\le\gamma\alpha_t\sum_{k=0}^{t-1}\|B_k\|_{\infty}.
    \end{align*}
    
\end{proof}

\begin{lem}
    \label{lem: recursion_hoeffding}
    Suppose $\{(X_t(i))_{i\in[d]}\}_{t\ge0}$ is a martingale difference w.r.t.\ filtration $\{\FM_{t}\}_{t\ge0}$, satisfying $\EB[X_t(i)|\FM_t]=0$ and $|X_t(i)|\le M$, a.s. for all $i\in[d]$. For recursion $Y_{t+1}(i)=(1-\alpha_t c_i)Y_t(i)+\alpha_t X_t(i)$, where $Y_0(i)=0$, $c_i>0$, $(1-\alpha_{t} c_i)\alpha_{t-1}\le\alpha_t$, and $\alpha_0\le c_{i}^{-1}$ for all $t\ge1$ and $i\in[d]$, for any $\lambda\in\RB$, we have:
    \begin{align}
        \EB\exp\left(\lambda \left\|Y_t\right\|_{\infty}\right)\le\exp\left(\frac{\lambda^2\alpha_{t-1}M^2}{2c_{\min}}\right),\label{eq: seq_hof}
    \end{align}
    where $c_{\min} = \min_{i\in[d]}c_i$. And also, $\EB\left\|Y_t\right\|\le3\sqrt{2c_{\min}^{-1}\alpha_{t-1}M^2\ln2d}$.
\end{lem}

\begin{proof}
    Firstly, we prove the following inequality by induction.
    \begin{align*}
        \EB\exp\left(\lambda Y_t(i)\right)\le\exp\left(\frac{\lambda^2\alpha_{t-1}M^2}{2c_{\min}}\right),
    \end{align*}
    which is true when $t=1$. For $Y_{t+1}(i)$, we have:
    \begin{align*}
        \EB\exp(\lambda Y_{t+1}(i))&=\EB\exp(\lambda((1-\alpha_t c_i)Y_t(i)+\alpha_t X_t(i)))\notag\\
        &\le\EB\exp\left(\frac{\lambda^2(1-\alpha_t c_i)^2\alpha_{t-1}M^2}{2c_{\min}}+\frac{\lambda^2\alpha_t^2M^2}{2}\right)\notag\\
        &\le\exp\left(\frac{\lambda^2\alpha_t M^2}{2c_{\min}}\right),
    \end{align*}
    where the last inequality is true due to $(1-\alpha_{t}c_i)\alpha_{t-1}\le\alpha_{t}$. Thus, for $\left\|Y_t\right\|_{\infty}$, we have:
    \begin{align*}
        \EB\exp(\lambda\left\|Y_t\right\|_{\infty})&\le\sum_{i\in[d]}\EB\exp(\lambda|Y_t(i)|)\notag\\
        &\le\sum_{i\in[d]}\EB\exp(\lambda Y_t(i))+\EB\exp(-\lambda Y_t(i))\notag\\
        &\le2d\exp\left(\frac{\lambda^2\alpha_{t-1} M^2}{2c_{\min}}\right).
    \end{align*}
    Then, the tail bound of $\left\|Y_t\right\|_{\infty}$ satisfies:
    \begin{align*}
        \PB\left(\left\|Y_t\right\|_\infty\ge\tau\right)\le2d\exp\left(-\frac{c_{\min}\tau^2}{2\alpha_{t-1}M^2}\right).
    \end{align*}
    By choosing $\tau_0=\sqrt{2 c^{-1}_{\min}\alpha_{t-1}M^2\ln2d}$, the expectation of $|Y_t|$ satisfies:
    \begin{align*}
        \EB\left\|Y_t\right\|_{\infty}=\int_{0}^{+\infty}\PB(\left\|Y_t\right\|_{\infty}\ge \tau)d\tau&=\int_{0}^{\tau_0}\PB(\left\|Y_t\right\|_{\infty}\ge\tau)d\tau +\int_{\tau_0}^{+\infty}\PB(\left\|Y_t\right\|_{\infty}\ge\tau)d\tau\\
        &\le\tau_0+\int_{\tau_0}^{+\infty}\PB(\left\|Y_t\right\|_{\infty}\ge\tau)d\tau\\
        &\overset{(a)}{\le}\tau_0+\frac{2\alpha_{t-1}M^2}{c_{\min}\tau_0}\\
        &\le3\sqrt{2c^{-1}_{\min}\alpha_{t-1}M^2\ln2d},
    \end{align*}
    where we use $\int_{c}^{+\infty}\exp(-x^2)dx\le\int_{c}^{+\infty}\exp(-cx)dx$ in (a).
\end{proof}

\begin{lem}
    \label{lem: markov_ineq}
    Denote $\{\xi_t\}_{t=0}^{T}$ is the random variable on a Markovian decision chain $(\SM\times\AM, P^\pi)$, which satisfies fast mixing property in Assumption~\ref{asmp: fastmixing}, and $\FM_t:=\sigma(\cup_{k<t}\sigma(\xi_k))$, for any recursion $X_{t+1}(s,a) = (1-\alpha_t d_{\pi}(s,a))X_t(s,a)+\alpha_t f_t(s,a)\left(\1(\xi_t=(s,a))-d_\pi(s,a)\right)$ satisfying $\alpha_t=\frac{1}{d_{\min}(t+p_\alpha)^\alpha}$, where $p_{\alpha}:=\left\lceil\left(\frac{d_{\max}}{d_{\min}}\right)^{1/\alpha}\right\rceil$, and $f_t(s,a)$ is measurable w.r.t. $\FM_t$, we have:
    \begin{align*}
        \EB\left\|X_{t+1}\right\|_{\infty}\le&\frac{\alpha_t M_f M}{1-\rho}\left(3+\ln(t+1+p_{\alpha})+\sum_{k=0}^t\alpha_k+\frac{\sum_{k=0}^n\EB\left\|f_{k+1}-f_k\right\|_{\infty}}{M_f}\right)\notag\\
        &+\frac{6\sqrt{2d_{\min}^{-1}\alpha_tM_f^2M^2\ln 2|\SM||\AM|}}{1-\rho},
    \end{align*}
    where:
    \begin{align*}
        \sum_{k=0}^t\alpha_k \le 
        \begin{cases}
            \frac{\ln(t+p_1)-\ln(p_1-1)}{d_{\min}} &\hspace{4pt}\text{if $\alpha=1$},\\
            \frac{(t+p_{\alpha})^{1-\alpha}-(p_{\alpha}-1)^{1-\alpha}}{1-\alpha}&\hspace{4pt}\text{if $\alpha\in(0,1)$}.
        \end{cases}
    \end{align*}
\end{lem}

\begin{proof}
    By Assumption~\ref{asmp: fastmixing}, we notice that:
    \begin{align*}
        \left|\sum_{t=0}^{+\infty}\left(P_t^\pi(s,a|\xi_t)-d_\pi(s,a)\right)\right|&\le\sum_{t=0}^{+\infty}\max_{(s,a)\in\SM\times\AM}d_{TV}(\rmP_t^\pi(\cdot|s,a), d_{\pi}(\cdot))\notag\\
        &\le\frac{M}{1-\rho}.
    \end{align*}
    Thus, we can decompose the $\1(\xi_t=(s,a))-d_\pi(s,a)$ into:
    \begin{align*}
        \1(\xi_t=(s,a))-d_\pi(s,a)&=\left(\sum_{k=0}^{+\infty}P_k^\pi(s,a|\xi_t)-d_\pi(s,a)\right)-\left(\sum_{k=1}^{+\infty}P_k^\pi(s,a|\xi_t)-d_\pi(s,a)\right)\notag\\
        &:=\psi(s,a;\xi_t)-\PM\psi(s,a;\xi_t),
    \end{align*}
    where $\PM\psi(s,a;\xi):=\sum_{(s',a')\in\SM\times\AM}\psi(s,a;s',a') P^\pi(s',a'|\xi)$. Thus, the update rule of $X_t(s,a)$ can be written by:
    \begin{align*}
        X_{t+1}(s,a) =& (1-\alpha_t d_{\pi}(s,a))X_t(s,a)+\alpha_t f_t(s,a)\left(\psi(s,a;\xi_t)-\PM\psi(s,a;\xi_t)\right)\notag\\
        =&(1-\alpha_t d_{\pi}(s,a))X_t(s,a)+\alpha_t f_t(s,a)\left(\psi(s,a;\xi_t)-\PM\psi(s,a;\xi_{t-1})\right)\notag\\
        &+\alpha_t f_t(s,a)\left(\PM\psi(s,a;\xi_{t-1})-\PM\psi(s,a;\xi_{t})\right).
    \end{align*}
    We denote $\widetilde{X}_{t}(s,a) = X_t(s,a)+\alpha_t f_t(s,a)\PM\psi(s,a;\xi_{t-1}))$, then we have:
    \begin{align*}
        \widetilde{X}_{t+1}(s,a) =& (1-\alpha_t d_{\pi}(s,a))\widetilde{X}_t(s,a)+\alpha_t f_t(s,a)\left(\psi(s,a;\xi_t)-\PM\psi(s,a;\xi_{t-1})\right)\notag\\
        &+\alpha_t^2 d_\pi(s,a)f_t(s,a)\PM\psi(s,a;\xi_{t-1})+\alpha_{t+1}f_{t+1}(s,a)\PM\psi(s,a;\xi_t)-\alpha_{t}f_{t}(s,a)\PM\psi(s,a;\xi_t)\notag\\
        =&(1-\alpha_t d_{\pi}(s,a))\widetilde{X}_t(s,a)+\alpha_t f_t(s,a)\left(\psi(s,a;\xi_t)-\PM\psi(s,a;\xi_{t-1})\right)\notag\\
        &+\alpha_t^2 d_\pi(s,a)f_t(s,a)\PM\psi(s,a;\xi_{t-1})+\alpha_{t+1}\left(f_{t+1}(s,a)-f_t(s,a)\right)\PM\psi(s,a;\xi_t)\notag\\
        &+\left(\alpha_{t+1}-\alpha_{t}\right)f_{t}(s,a)\PM\psi(s,a;\xi_t)\notag\\
        :=&(1-\alpha_t d_{\pi}(s,a))\widetilde{X}_t(s,a)+\alpha_t \Delta_{1,t}(s,a)+\alpha_t^2 d_\pi(s,a)\Delta_{2,t}(s,a)\notag\\
        &+\alpha_{t+1}\Delta_{3,t}(s,a)+(\alpha_{t+1}-\alpha_t)\Delta_{4,t}(s,a),
    \end{align*}
    where
    \begin{align*}
        \Delta_{1,t}(s,a)&:=f_t(s,a)\left(\psi(s,a;\xi_t)-\PM\psi(s,a;\xi_{t-1})\right)\notag\\
        \Delta_{2,t}(s,a)&:=f_t(s,a)\PM\psi(s,a;\xi_{t-1})\notag\\
        \Delta_{3,t}(s,a)&:=\left(f_{t+1}(s,a)-f_t(s,a)\right)\PM\psi(s,a;\xi_t)\notag\\
        \Delta_{4,t}(s,a)&:=f_{t}(s,a)\PM\psi(s,a;\xi_t).
    \end{align*}
    Recursively solving above equation, we have:
    \begin{align*}
        \widetilde{X}_{t+1}(s,a) =& \prod_{k=0}^{t}(1-\alpha_k d_\pi(s,a))\cdot\widetilde{X}_{0}(s,a)+\sum_{k=0}^t \alpha_k\Delta_{1,k}(s,a)\prod_{i=k+1}^t(1-\alpha_i d_\pi(s,a))\notag\\
        &+\sum_{k=0}^t\alpha_k^2 d_\pi(s,a)\Delta_{2,k}(s,a)\prod_{i=k+1}^t(1-\alpha_i d_\pi(s,a))\notag\\
        &+\sum_{k=0}^t\alpha_{k+1}\Delta_{3,k}(s,a)\prod_{i=k+1}^t(1-\alpha_i d_\pi(s,a))\notag\\
        &+\sum_{k=0}^t(\alpha_{k+1}-\alpha_k)\Delta_{4,k}(s,a)\prod_{i=k+1}^t(1-\alpha_i d_\pi(s,a)).
    \end{align*}
    For the second term, we notice $\Delta_{1,k}(s,a)$ is a martingale difference w.r.t. filtration $\{\FM_t\}_{t\ge0}$, which satisfies $\EB[\Delta_{1,k}(s,a)|\FM_{k}]=0$ and $|\Delta_{1,k}(s,a)|\le \frac{2M_f M}{1-\rho}$. Thus, by Lemma~\ref{lem: recursion_hoeffding}, we have:
    \begin{align*}
        \EB\max_{(s,a)\in\SM\times\AM}\left|\sum_{k=0}^t \alpha_k\Delta_{1,k}(s,a)\prod_{i=k+1}^t(1-\alpha_i d_\pi(s,a))\right|\le \frac{6\sqrt{2d_{\min}^{-1}\alpha_tM_f^2M^2\ln 2|\SM||\AM|}}{1-\rho}.
    \end{align*}
    For other terms, we apply inequality $(1-\alpha_t d_\pi(s,a))\alpha_{t-1}\le\alpha_t$. Thus, we can bound $\EB\left\|\widetilde{X}_{t+1}\right\|_{\infty}$ by:
    \begin{align*}
        \EB\left\|\widetilde{X}_{t+1}\right\|_{\infty}\le&\frac{(1-\alpha_0 d_{\min})\alpha_t}{\alpha_0}\EB\left\|\widetilde{X}_0\right\|_{\infty}+\frac{6\sqrt{2d_{\min}^{-1}\alpha_tM_f^2M^2\ln 2|\SM||\AM|}}{1-\rho}+\alpha_t\frac{M_f M\sum_{k=0}^t\alpha_k}{1-\rho}\notag\\
        &+\alpha_t\frac{M\sum_{k=0}^t\EB\|f_{k+1}-f_k\|_{\infty}}{1-\rho}+\alpha_t\frac{M_f M\sum_{k=0}^t\frac{\alpha_k-\alpha_{k+1}}{\alpha_k}}{1-\rho}.
    \end{align*}
    Then, by definition of $\widetilde{X}_{t}$ and $\EB\left\|\widetilde{X}_0\right\|_{\infty}\le\frac{M_f M}{1-\rho}$, we have:
    \begin{align*}
        \EB\left\|X_{t+1}\right\|_{\infty}\le&\alpha_t\frac{(1-\alpha_0 d_{\min})M_f M}{1-\rho}+\frac{6\sqrt{2d_{\min}^{-1}\alpha_tM_f^2M^2\ln 2|\SM||\AM|}}{1-\rho}+\alpha_t\frac{M_f M\sum_{k=0}^t\alpha_k}{1-\rho}\notag\\
        &+\alpha_t\frac{M\sum_{k=0}^t\EB\|f_{k+1}-f_k\|_{\infty}}{1-\rho}+\alpha_t\frac{M_f M\sum_{k=0}^t\frac{\alpha_k-\alpha_{k+1}}{\alpha_k}}{1-\rho}+\alpha_{t+1}\frac{M_f M}{1-\rho}.
    \end{align*}
    By Lemma~\ref{lem: sum_t}, when $\alpha=1$, we have:
    \begin{align*}
        &\sum_{k=0}^t\alpha_k\le\frac{\ln\left(t+p_{1}\right)-\ln\left(p_1-1\right)}{d_{\min}}\notag\\
        &\sum_{k=0}^t\frac{\alpha_k-\alpha_{k+1}}{\alpha_k}=\sum_{k=0}^{t}\frac{1}{k+1+p_1}\le\ln\left(t+p_1+1\right)-\ln p_1.
    \end{align*}
    When $\alpha\in(0,1)$, we have:
    \begin{align*}
        &\sum_{k=0}^t\alpha_k\le\frac{(t+p_{\alpha})^{1-\alpha}-(p_{\alpha}-1)^{1-\alpha}}{1-\alpha}\notag\\
        &\sum_{k=0}^t\frac{\alpha_k-\alpha_{k+1}}{\alpha_k}=\sum_{k=0}^{t}1-\frac{(k+p_{\alpha})^{\alpha}}{(k+p_{\alpha}+1)^{\alpha}}\le\sum_{k=0}^t\frac{1}{k+p_{\alpha}}\le\ln(t+p_{\alpha})+\frac{1}{p_{\alpha}}-\ln p_{\alpha}.
    \end{align*}
    Combining all above together, we have:
    \begin{align*}
        \EB\left\|X_{t+1}\right\|_{\infty}\le&\frac{\alpha_t M_f M}{1-\rho}\left(3+\ln(t+1+p_{\alpha})+\sum_{k=0}^t\alpha_k+\frac{\sum_{k=0}^n\EB\left\|f_{k+1}-f_k\right\|_{\infty}}{M_f}\right)\notag\\
        &+\frac{6\sqrt{2d_{\min}^{-1}\alpha_tM_f^2M^2\ln 2|\SM||\AM|}}{1-\rho}.
    \end{align*}
\end{proof}

\end{appendix}

\end{document}